\documentclass{article}
\usepackage{libertinus}
\usepackage[margin=0.85in]{geometry}

\usepackage[authoryear,round]{natbib}
\usepackage{amsmath, mathtools}
\usepackage{bbm, mathrsfs}
\usepackage{amssymb,amsthm,upgreek}
\usepackage[inline]{enumitem}
\usepackage{hyperref,xcolor, textcmds}
\hypersetup{colorlinks=true,linkcolor=black,citecolor=black,urlcolor=blue}
\usepackage{xargs}
\usepackage{booktabs}
\usepackage{tikz}
\usepackage{caption, subcaption, placeins}
\usepackage{todonotes}
\usepackage{algorithm}
\usepackage{algpseudocode}
\usepackage{multirow}
\usepackage{makecell}

\usepackage{thmtools} 
\usepackage{thm-restate}

\usepackage{tikz}
\usetikzlibrary{intersections} 
\usetikzlibrary{arrows.meta,calc,decorations.pathreplacing}

\usetikzlibrary{calc,arrows.meta,positioning, fit}

\usepackage{pgfplots}
\pgfplotsset{compat=1.18}

\usepackage[noabbrev,capitalize]{cleveref}

\makeatletter
\AddToHook{cmd/appendix/before}{%
  \def\cref@section@alias{appendix}%
  \def\cref@subsection@alias{appendix}%
}
\makeatother

\crefname{theorem}{theorem}{Theorems}
\Crefname{Theorem}{Theorem}{Theorems}

\newtheorem{lemma}{Lemma}
\crefname{lemma}{lemma}{lemmas}
\Crefname{Lemma}{Lemma}{Lemmas}

\newtheorem{corollary}{Corollary}
\crefname{corollary}{corollary}{corollaries}
\Crefname{Corollary}{Corollary}{Corollaries}

\newtheorem{proposition}{Proposition}
\crefname{proposition}{proposition}{propositions}
\Crefname{Proposition}{Proposition}{Propositions}

\newtheorem{remark}{Remark}
\crefname{remark}{remark}{remarks}
\Crefname{Remark}{Remark}{Remarks}

\newtheorem{example}{Example}
\crefname{example}{example}{examples}
\Crefname{Example}{Example}{Examples}

\crefname{figure}{figure}{figures}
\Crefname{Figure}{Figure}{Figures}

\crefname{algorithm}{algorithm}{algorithms}
\Crefname{Algorithm}{Algorithm}{Algorithms}

\Crefname{assumption}{\textbf{H}\hspace{-3pt}}{\textbf{H}\hspace{-3pt}}
\crefname{assumption}{\textbf{H}}{\textbf{H}}

\Crefname{assumptionAO}{\textbf{AO}\hspace{-3pt}}{\textbf{AO}\hspace{-3pt}}
\crefname{assumptionAO}{\textbf{AO}}{\textbf{AO}}

\Crefname{assumptionL}{\textbf{L}\hspace{-3pt}}{\textbf{L}\hspace{-3pt}}
\crefname{assumptionL}{\textbf{L}}{\textbf{L}}

\newtheorem{definition}{Definition}
\crefname{definition}{definition}{definitions}
\Crefname{Definition}{Definition}{Definitions}

\setlist[enumerate]{leftmargin=*}

\def\bfN{\mathbf{N}}

\def\msa{\mathsf{A}}

\def\msx{\mathsf{X}}

\def\mca{\mathcal{A}}

\def\mcb{\mathcal{B}}  

\def\rset{\mathbb{R}}

\def\nset{\mathbb{N}}

\def\rmd{\mathrm{d}}

\def\rme{\mathrm{e}}
\def\rmn{\mathrm{n}}

\newcommand{\R}{\mathbb R}

\newcommandx{\functionspace}[2][1=+]{\mathbb{F}_{#1}(#2)}

\newcommandx{\VarDeux}[3][3=]{\operatorname{Var}^{#3}_{#1}\left\{#2 \right\}}

\newcommand{\1}{\mathbbm{1}}

\newcommand{\LeftEqNo}{\let\veqno\@@leqno}

\newcommand{\N}{\ensuremath{\mathbb{N}}}

\newcommand{\PE}{\mathbb{E}}
\newcommand{\PP}{\mathbb{P}}
\newcommand{\Ft}{\mathcal{F}}

\newcommandx{\Vnorm}[2][1=V]{\| #2 \|_{#1}}
\newcommandx{\VnormEq}[2][1=V]{\left\| #2 \right\|_{#1}}
\newcommandx{\norm}[2][1=]{\ifthenelse{\equal{#1}{}}{\left\Vert #2 \right\Vert}{\left\Vert #2 \right\Vert^{#1}}}
\newcommandx{\normLigne}[2][1=]{\ifthenelse{\equal{#1}{}}{\Vert #2 \Vert}{\Vert #2\Vert^{#1}}}

\newcommandx\probaMarkovTilde[2][2=]
{\ifthenelse{\equal{#2}{}}{{\widetilde{\mathbb{P}}_{#1}}}{\widetilde{\mathbb{P}}_{#1}\left[ #2\right]}}

\def\ie{\textit{i.e.}}

\def\eqsp{\;}

\def\metric{\mathbf{g}}

\newcommandx{\weight}[2][2=n]{\omega_{#1,#2}^N}

\newcommandx\sequence[3][2=,3=]
{\ifthenelse{\equal{#3}{}}{\ensuremath{\{ #1_{#2}\}}}{\ensuremath{\{ #1_{#2}, \eqsp #2 \in #3 \}}}}
\newcommandx\process[3][2=,3=]
{\ifthenelse{\equal{#3}{}}{\ensuremath{( #1_{#2})}}{\ensuremath{( #1_{#2})_{ #2 \in #3}}}}
\newcommandx\processt[2][2=]
{\ifthenelse{\equal{#2}{}}{\ensuremath{( #1_{t})}}{\ensuremath{( #1_{t})_{ t \in #2}}}}
\newcommandx\sequenceD[3][2=,3=]
{\ifthenelse{\equal{#3}{}}{\ensuremath{\{ #1_{#2}\}}}{\ensuremath{( #1)_{ #2 \in #3} }}}

\newcommandx{\sequencen}[2][2=n\in\N]{\ensuremath{\{ #1_n, \eqsp #2 \}}}
\newcommandx\sequenceDouble[4][3=,4=]
{\ifthenelse{\equal{#3}{}}{\ensuremath{\{ (#1_{#3},#2_{#3}) \}}}{\ensuremath{\{  (#1_{#3},#2_{#3}), \eqsp #3 \in #4 \}}}}
\newcommandx{\sequencenDouble}[3][3=n\in\N]{\ensuremath{\{ (#1_{n},#2_{n}), \eqsp #3 \}}}

\def\eg{e.g.}

\newcommand{\opnorm}[1]{{\left\vert\kern-0.25ex\left\vert\kern-0.25ex\left\vert #1
    \right\vert\kern-0.25ex\right\vert\kern-0.25ex\right\vert}}

\def\Lip{\operatorname{Lip}}

\newcommandx{\CPE}[3][1=]{{\mathbb E}_{#1}\left[\left. #2 \, \middle \vert \, #3 \right. \right]} 

\newcommandx{\CPELigne}[3][1=]{{\mathbb E}_{#1}[\left. #2 \,  \vert \, #3 \right. ]} 

\newcommandx{\CPVar}[3][1=]{\mathrm{Var}^{#3}_{#1}\left\{ #2 \right\}}
\newcommand{\CPP}[3][]
{\ifthenelse{\equal{#1}{}}{{\mathbb P}\left(\left. #2 \, \right| #3 \right)}{{\mathbb P}_{#1}\left(\left. #2 \, \right | #3 \right)}}

\newcommandx{\osc}[2][1=]{\mathrm{osc}_{#1}(#2)}

\def\Lambdabf{\mathbf{\Lambda}}
\def\lambdabf{\boldsymbol{\lambda}}

\newcommand\coupling[2]{\Gamma(\mu,\nu)}

\def\bfZ{\mathbf{Z}}
\def\bfX{\mathbf{X}}

\newcommandx{\wasserstein}[3][1=\distance,3=]{\mathscr{W}_{#1}^{#3}\left(#2\right)}
\newcommandx{\wassersteinLigne}[3][1=\distance,3=]{\mathscr{W}_{#1}^{#3}(#2)}
\newcommandx{\wassersteinD}[1][1=\distance]{\mathscr{W}_{#1}}
\newcommandx{\wassersteinDLigne}[1][1=\distance]{\mathscr{W}_{#1}}

\def\bfY{\mathbf{Y}}

\def\gauss{\mathrm{N}}
\def\Idd{\mathrm{I}_d}

\def\bfX{\mathbf{X}}

\def\pre_Stif{\mathcal{V}}

\def\bfo{\mathbf{o}}
\def\Vlyap{\mathfrak{V}}

\newcommandx{\Voi}[1][1=i]{\Vlyap_{\bfo,#1}}
\newcommandx{\Vlyapc}[2][1=\bfo,2=i]{\Vlyap_{#1,#2}}
\newcommandx{\tVlyapc}[2][1=\bfo,2=i]{\tilde{\Vlyap}_{#1,#2}}

\def\Wlyap{\mathfrak{W}}

\newcommandx{\Woi}[1][1=i]{\Wlyap_{\bfo,#1}}
\newcommandx{\Wlyapc}[2][1=\bfo,2=i]{\Wlyap_{#1,#2}}

\newcommand{\Exp}{\mathrm{Exp}}

\def\Denois{K}

\def\bfZthin{\bar{\bfZ}}
\def\zthin{\bar{Z}}
\def\bart{\bar{T}}
\def\bfZtild{\tilde{\bfZ}}
\def\ztild{\tilde{Z}}

\def\bfW{\mathbf{W}}
\def\bfB{\mathbf{B}}
\def\MPPpi{\mathrm{MPP}(\pi)}

\def\pigdm{$\Pi$GDM}

\def\law{\mathrm{Law}}
\def\filt{\mathbf{F}}
\renewcommand{\iint}[2]{\{#1,\ldots,#2\}}
\newcommand{\unif}{\mathrm{Unif}}

\newcommand{\poisson}{\mathrm{Pois}}
\newcommand{\loigamma}{\boldsymbol{\Gamma}}

\newcommand{\breg}{\mathscr{D}}

\usepackage{mathtools}
\mathtoolsset{showonlyrefs}

\title{Mixed Poisson diffusion models and application to posterior sampling}

\author{Thibaut Pellerin$^{1,2}$\thanks{Corresponding author: thibaut.pellerin@cea.fr}, Alain Durmus$^{1}$, Elinor Berger$^{2}$ \& Eric Barat$^{2}$ \\
$^{1}$ CMAP, CNRS, École polytechnique, Institut Polytechnique de Paris, Palaiseau, France \\
$^{2}$ Université Paris-Saclay, CEA, List, F-91120 Palaiseau, France
}

\date{}

\begin{document}
\maketitle

\begin{abstract}
\normalsize
Diffusion models are a powerful framework for learning complex probability distributions. While widely used diffusion models for continuous data rely on Gaussian noising processes, we instead develop diffusion models based on Poisson noise. More precisely, we introduce and study a new mathematical framework, called Mixed Poisson Diffusion Models (MPDM), with noising dynamics driven by multivariate counting processes.

Our methodology is motivated by Poisson inverse problems. The success of Gaussian diffusion models has led to their use as prior models in a growing number of posterior sampling methods for inverse problems. However, most of these approaches focus on Gaussian inverse problems, in which observations are generated through a deterministic forward operator corrupted by additive Gaussian noise. By contrast, Poisson inverse problems, where observations are discrete and Poisson distributed, remain comparatively underexplored in this literature.
Building on our Poisson generative dynamics, we develop a natural framework for posterior sampling in Poisson inverse problems and derive several posterior sampling algorithms based on the proposed model. We evaluate our approach on realistic Poisson inverse problems and show that the resulting framework is well suited to this class of problems.
\end{abstract}

\section{Introduction}
Diffusion models have emerged as a powerful and versatile framework for learning complex probability distributions, with striking empirical success in high-dimensional generative modeling
\citep{sohl2015deep,ho2020denoising,song2021score}.
Their central principle is to define a tractable forward noising process that progressively transforms the data distribution into a simple reference distribution, and then to learn the corresponding reverse-time dynamics in order to generate new samples.
In their most widely used form, diffusion models rely on Gaussian perturbations, either through discrete-time Gaussian transitions or continuous-time stochastic differential equations (SDEs).
This Gaussian structure is particularly appealing from both theoretical and computational perspectives, as it leads to explicit transition kernels and efficient simulation schemes.

At the same time, recent works have demonstrated that the diffusion paradigm is not intrinsically tied to Gaussian noise.
In particular, diffusion-type generative models have been extended to discrete state spaces through appropriately designed Markov chains and continuous-time jump processes
\citep{austin2021structured,campbell2022continuous}.
These developments suggest a broader perspective in which the choice of the forward corruption mechanism can be adapted to the geometry or the observation model underlying the problem of interest.
Yet, despite this flexibility, comparatively little attention has been devoted to diffusion models whose noising mechanism is driven by counting processes.

In this work, we introduce and study a new mathematical framework, which we call
\emph{Mixed Poisson Diffusion Models} (MPDMs), in which the forward noising dynamics are given by multivariate counting processes, and in particular mixed Poisson processes \citep{grandell_mixed_1997}.
Rather than progressively corrupting the data by additive Gaussian perturbations, MPDMs encode the data through the random accumulation of discrete events, leading to a fundamentally different family of intermediate distributions driven by Poisson noising dynamics.
The resulting framework preserves the main conceptual structure of diffusion modeling: a simulation-free noising mechanism that defines a path of probability distributions, together with a Markov process sharing the same marginals and characterized by learnable dynamics, ultimately yielding a tractable generative model.

This construction is motivated in particular by inverse problems with Poisson observations. In a Bayesian inverse problem, the goal is to recover an unknown signal of interest $x$ from a corrupted observation $y$. The central object is the posterior distribution $\pi^y$ of the signal, defined through Bayes' rule, up to a normalizing constant, by $
\pi^y(\rmd x) \propto g(y\vert x)\,\pi(\rmd x)$,
where $\pi$ denotes the prior distribution on the unknown signal and $(y,x)\mapsto g(y\vert x)$ is the likelihood relating the observation to the underlying signal. Diffusion models have recently been used as expressive data-driven priors in Bayesian inverse problems and have enabled high-quality reconstructions in a broad range of imaging applications. However, sampling from $\pi^y$ remains challenging because the diffusion prior is typically available only implicitly through a learned score or denoiser function. To address this difficulty, a large body of work has proposed modifying the reverse diffusion dynamics through an additional guidance term so as to generate approximate samples from the posterior distribution. Prominent examples include Diffusion Posterior Sampling (DPS) \citep{chung_diffusion_2024}, Plug-and-Play Diffusion Models (PnP-DM) \citep{wu_principled_2024}, and Decoupled Annealing Posterior Sampling (DAPS) \citep{zhang_improving_2025}, among many others \citep{song_pseudoinverse-guided_2022,kawar_denoising_2022,janati_mixture-based_2025}; see, for instance, the review of \citet{janati_briding_2025}.

Most existing methodologies, however, are developed or theoretically justified primarily for linear inverse problems with Gaussian observation noise. In this setting, the Gaussian structure of the likelihood is naturally compatible with the Gaussian noising process underlying standard diffusion models. In particular, for an important class of such problems, this compatibility makes it possible to exploit explicit Gaussian identities and, in some cases, to construct exact posterior transitions; see \eqref{eq:denoiser_tilt_gaussian} and \Cref{sec:inverse_gaussian}. This favorable structure generally disappears even for simple Poisson inverse problems, where observations are discrete and the likelihood is Poisson rather than Gaussian. Developing a diffusion framework that is intrinsically compatible with Poisson observations is therefore the main applicative motivation for the methodology introduced in this work.
This class of problem is 
common in fields such as medical imaging, \eg, Positron Emission Tomography (PET) \citep{bailey_positron_2005}, or astronomical observation \citep{donath_joint_2024}. These Poisson inverse problems often arise from photon-counting measurements, where observed counts can be close to or equal to zero, and one cannot reasonably default to a Gaussian noise model. Indeed, in such settings, the observation noise is inherently discrete and signal dependent, making Gaussian-based generative dynamics less naturally matched to the underlying statistical model.
By building the generative process itself from Poissonian dynamics, MPDMs provide a framework that is naturally aligned with this class of observation models and, as we show, leads to new posterior-sampling methodologies for Poisson inverse problems.

Our contribution is threefold:
\vspace{-0.3cm}
\begin{enumerate}[wide,label=$\bullet$]
\item
In \Cref{sec:mpp}, we introduce \textit{Mixed Poisson Diffusion Models} (MPDMs). We construct a stochastic noising dynamics based on the theory of multivariate counting processes \citep{protter_stochastic_2005,daley_introduction_2008}, and notably mixed Poisson processes \citep{grandell_mixed_1997,zocher_multivariate_2005}. This yields a rigorous mathematical framework for Poissonian generative modeling \citep{chen_learning_2023,bhattacharya_itdpdm_2026}, closely paralleling the Gaussian diffusion framework described in \Cref{sec:diffusion_gaussian}. We further establish quantitative sampling-error bounds that explicitly account for early stopping, denoiser estimation, and time-discretization errors, thereby providing a Poissonian counterpart to the error analyses commonly developed for other types of diffusion samplers.
\item
In \Cref{sec:inv}, we develop a framework for \textit{MPDM-based posterior sampling}. To the best of our knowledge, this is the first posterior sampling methodology built upon Poissonian generative dynamics. We motivate this construction through Poisson inverse problems. For simple forward operators, \Cref{th:twist} establishes a structural connection between MPDMs and Poisson inverse problems: in contrast to Gaussian diffusion-based approaches, posterior sampling with MPDMs can be carried out without introducing substantial approximations. For more general forward operators, we adapt several established diffusion-based posterior sampling methods to the MPDM setting.
\item
Finally, we evaluate the proposed algorithms in \Cref{sec:expe}. In particular we consider both synthetic and realistic Poisson inverse problems. Our experiments demonstrate the practical relevance of the framework and highlight the advantages of matching the generative noising dynamics to the Poisson structure of the observation model.
\end{enumerate}

\vspace{-0.5cm}

\paragraph{Notation.} All random variable are supposed to be defined on a probability space $(\Omega, \Ft, \PP)$.
  Given a topological space $\msx$ equipped with its Borel $\sigma$-algebra $\mcb(\msx)$, a (continuous-time) \emph{stochastic process} $\bfX=\processt{X}_{t\geqslant 0}$ is a collection of random variables $X_t$ valued in $\msx$, indexed by $t\in\rset_+$. If $\msx=\rset^d$, we write $\bfX = (\bfX^1,\dots,\bfX^d)$, where for $i \in \iint{1}{d}$ and $t\in\rset_+$, $X^i_t$ is the $i$-th component of $X_t$. The \emph{canonical filtration} associated to $\bfX$ is the filtration $\filt^\bfX =\processt{\Ft^\bfX}_{t\geqslant 0}$, where for any $t\geqslant 0$, $\Ft^\bfX_t = \sigma(\bfX^i_s, 0\leqslant s\leqslant t, i=1,\dots,d)$. $\bfX$ is said \emph{càdlàg} if its trajectories are almost surely right-continuous and admit left limits, \ie, for almost all $\omega\in\Omega$ and any $t>0$, $X_{t+}(\omega) \coloneqq \lim_{s\downarrow t}X_s(\omega) = X_t(\omega)$ and $X_{t-}(\omega)\coloneqq \lim_{s\uparrow t}X_s(\omega)\in \rset^d$.

\section{Mixed Poisson diffusion models}
\label{sec:mpp}
In this section, we propose our framework for Poissonian generative modeling, leveraging the theory of counting processes and specifically mixed Poisson processes \citep{grandell_mixed_1997}. This construction provides a continuous-time mathematical framework for Poissonian generative models. A parallel with Gaussian diffusion models is provided in  \Cref{sec:diffusion_gaussian}.

\subsection{Mixed Poisson processes for generative modelling}
\label{sec:mpp_def}
\paragraph*{Counting processes.}
Unlike SDE-based diffusion models, whose trajectories are continuous, counting processes are \emph{càdlàg}, integer-valued stochastic processes. We refer to \Cref{app:more_counting} and \citep{laub_hawkes_2015, protter_stochastic_2005,zocher_multivariate_2005} for a more detailed introduction.
\begin{definition} 
\label{def:counting}A univariate \textbf{counting process} is a stochastic process $\bfZ=\processt{Z}_{t\geqslant 0}$  whose trajectories $Z(\omega) : t\in\rset_+ \mapsto Z_t(\omega)\in\N$, for almost all $\omega\in\Omega$, are càdlàg, non-decreasing and piecewise constant, with $Z_t(\omega)-Z_{t-}(\omega)\in \{0,1\}$ for any $t>0$ and $Z_0(\omega)=0$. A multivariate counting process is a $\N^d$-valued process whose coordinates $\bfZ^i, 1\leqslant i \leqslant d,$ and norm $\Vert\bfZ\Vert_1$
are univariate counting processes.
\end{definition}
Alternatively, a univariate counting process can be defined through its sequence of \emph{jump times}, \ie, $Z_t = \sum_{n\geqslant 1}\1_{T_n\leqslant t}$ where $(T_n)_{n\geqslant1}$ satisfies, almost surely, $0<T_1<T_2<\dots\leqslant\infty$. In the multivariate case, the condition on $\Vert\bfZ\Vert_1$ prevents simultaneous jumps: we unambiguously define the jump times $(T_n)_{n\geqslant1}$ of $\bfZ$ as those of $\Vert\bfZ\Vert_1$. For any $n\geqslant 1$, almost surely, there is a unique pair $1\leqslant i\leqslant d,k\geqslant 1$ such that $T_n = T_k^i$, where $(T_l^i)_{l\geqslant 1}$ are the jump times of $\bfZ^i$. 

In the rest of this section, 
 $\bfZ$ is a multivariate counting process adapted to a complete, right-continuous filtration $\filt=\processt{\Ft}_{t\geqslant 0}$, \ie,  for any $t\geqslant 0$, $\Ft^\bfZ_t \subseteq \Ft_t$ where $\Ft^\bfZ_t = \sigma(Z^i_s, 1\leqslant i\leqslant d,0\leqslant s \leqslant t)$. $\filt^\bfZ = \processt{\Ft^\bfZ}_{t\geqslant 0}$ is the canonical filtration associated to $\bfZ$.
We furthermore assume that for any $t\geqslant 0$, $\PE[\Vert Z_t \Vert_1]<\infty$. Being non-decreasing, each coordinate process $\bfZ^i, 1\leqslant i \leqslant d$ is a $\filt$-submartingale. As a consequence of the Doob-Meyer decomposition, we have the following result, that is a simple extension of the univariate case; see \citep[Chapter III]{protter_stochastic_2005}.  

\begin{proposition}
\label{prop:compensator}
There exists a unique $\filt$-predictable process $\Lambdabf$ whose coordinates $\Lambdabf^i, 1\leqslant i \leqslant d$ are
non-decreasing finite-variation processes, with $\Lambda^i_0=0$ a.s., such that $\bfZ-\Lambdabf$ is a $\filt$-local martingale (coordinate-wise). $\Lambdabf$ is called the $\filt$-\textbf{compensator} of  $\bfZ$. 

If there exists a non-negative, $\filt$-predictable, locally integrable process $\lambdabf$ satisfying, almost surely, $\Lambda_t = \int_0^t\lambda_s \rmd s$ for every $t\geqslant 0$, we call $\boldsymbol{\lambda}$ the $\filt$-\textbf{stochastic intensity} of $\bfZ$.
\end{proposition}

A well-known counting process is the \emph{Poisson process} with rate $\alpha>0$ \citep[Chapter 3]{protter_stochastic_2005}. Such a process $\bfN^\alpha=\processt{N^\alpha}_{t\geqslant 0}$ has independent increments $N^{\alpha}_t-N^{\alpha}_s \sim \poisson(\alpha(t-s)), t>s\geqslant 0$, with $N_0^\alpha=0$ a.s. A multivariate Poisson process $\bfN^\alpha$ with rate $\alpha=(\alpha_1,\dots,\alpha_d)\in(0,+\infty)^d$ has coordinate processes $\bfN^{\alpha_i}, 1\leqslant i\leqslant d$ that are independent univariate Poisson processes with respective rate $\alpha_i$. 
Its $\filt^{\bfN^\alpha}$-stochastic intensity is the deterministic function $t\mapsto\alpha$.

Poisson processes have the Markov property: this result can be extended to any counting process $\bfZ$ whose $\filt^\bfZ$-intensity has the form $t\mapsto \lambda_t = u_t(Z_{t-})\in\rset_+^d$, with $u$ a locally integrable function; see \Cref{prop:inhomogeneous_loc}. Simulation of $\bfZ$ is possible when its stochastic intensity has this form; see below.

\paragraph*{Mixed Poisson processes.} 
A natural extension of Poisson processes consists in replacing the constant rate by a random one. The resulting objects, called mixed Poisson processes, are extensively covered by \citet{grandell_mixed_1997}, as well as \citet{zocher_multivariate_2005} in the multivariate case. The statements below are proved in \Cref{app:proofs}.
In the rest of this section, the target distribution $\pi$ is a probability distribution supported on $\rset_+^d$, with finite first moment. 
\begin{definition}
\label{def:mixed}
Let $X$ be a $\rset_+^d$-valued random variable with distribution $\pi$. If, conditionally on $X$, $\bfZ$ is a Poisson process with rate $X$, we call $\bfZ$ the \textbf{mixed Poisson process (MPP)} with random rate $X$, and write $\bfZ\sim\MPPpi$. In particular, conditionally on $X$, $Z_t\sim \poisson(tX)$ for any $t\geqslant 0$.
\end{definition}
A mixed Poisson process $\bfZ\sim\MPPpi$ with random rate $X$ defines a \emph{stochastic localization process} \citep{montanari_sampling_2023}: indeed, \Cref{lemma:localization} shows that $X$ is $\Ft^\bfZ_{\infty}$-measurable, while the Markov property of $\bfZ$ follows from \Cref{prop:stochastic_intensity}.

\begin{restatable}
{theorem}{localization}
\label{lemma:localization}
    Let $\bfZ\sim\MPPpi$, with random rate $X$. Almost surely, $\lim\limits_{t\rightarrow\infty}Z_t/t=X$.
\end{restatable}
This result shows that if we were able to sample $Z_t$ for a large $t$, $Z_t/t$ would yield an approximate sample from $\pi$.
Given our first definition of MPPs, sampling from $\bfZ$ requires sampling from $\pi$.
Fortunately, similar to diffusion models (see \Cref{sec:diffusion_gaussian}),
sample-free simulation of $\bfZ$ is possible if we have access to the corresponding \emph{Poisson denoiser}. Indeed, \Cref{prop:stochastic_intensity}, which generalizes the univariate case  \citep[Proposition 4.1]{grandell_mixed_1997}, shows that the $\filt^\bfZ$-stochastic intensity of the mixed Poisson process coincides with its Poisson denoiser, defined as
\begin{equation}
    \label{eq:denoiser_def}
    (t,z)\in\rset_+\times \N^d \mapsto \Denois_t^\pi(z) \coloneqq\PE[X\vert Z_t=z] = \int_{\rset_+^d} x\, \pi_t(\rmd x \vert  z)\eqsp,
\end{equation}
where $\pi_t$ is the conditional distribution of $X$ given $Z_t$: for any $z\in\N^d$,
\begin{equation}
    \label{eq:denoising_dist}
    \pi_t(\rmd x\vert z) \propto\prod_{i=1}^d\rme^{-tx_i}\frac{(tx_i)^{z_i}}{z_i!}\pi(\rmd x) \propto \prod_{i=1}^d\rme^{-tx_i}x_i^{z_i}\pi(\rmd x)\eqsp.
\end{equation}
This denoiser depends solely on $\pi$ and not on any particular $\bfZ\sim\MPPpi$: in the univariate case, it can be expressed through the Laplace transform of $\pi$; see \citep[Chapter 4]{grandell_mixed_1997}. More generally, we have the following result \citep[Chapter 3]{zocher_multivariate_2005}.

\begin{restatable}{proposition}{denoiser}
\label{prop:denoiser}
    Given any $\bfZ\sim\MPPpi$, the Poisson denoiser $\Denois^\pi$ defined by \eqref{eq:denoiser_def} can be expressed as follows. For every $t>0$ and $z\in\N^d$, 
    \begin{equation}
        \label{eq:denoiser_ratio}
        \Denois_t^\pi(z) =  \left.\int_{\rset_+^d}x\prod_{i=1}^d\rme^{-tx_i}x_i^{z_i}\pi(\rmd x)\right/ \int_{\rset_+^d}\prod_{i=1}^d\rme^{-tx_i}x_i^{z_i}\pi(\rmd x)
    \end{equation}
   holds if $\PP(Z_t=z)>0$; otherwise, by convention, $\Denois^\pi_t(z)=0$.
 Furthermore, for any $z\in\N^d$,  $t>0 \mapsto \Denois_t^\pi(z)$ is smooth and $t>0\mapsto \Vert \Denois_t^\pi(z)\Vert_1$ is non-increasing.
\end{restatable}

Note that \eqref{eq:denoiser_ratio} allows us to extend the denoiser domain of definition to $\rset_+\times\rset_+^d$.

\begin{restatable}{theorem}{compensator}
\label{prop:stochastic_intensity} 
The $\filt^\bfZ$-compensator of a mixed Poisson process $\bfZ\sim\MPPpi$ is the process defined for every $t\geqslant 0$ by
    \begin{equation}
        \Lambda_t =  \int_0^t \Denois_s^\pi(Z_s) \rmd s\eqsp. 
    \end{equation}
    In other words, $\bfZ$ is a Markov counting process with $\filt^\bfZ$-stochastic intensity $t\in\R_+\mapsto \Denois^\pi_t(Z_{t-})$.
\end{restatable}

The native Markov property of mixed Poisson processes is a notable difference from several Gaussian diffusion models. In these frameworks, one often considers a non-Markov process, and then constructs a Markov process with identical time-marginal distributions, the \emph{Markovian projection}. 

\paragraph*{Learning the denoiser.}

Based on \Cref{prop:stochastic_intensity}, we will see that $\mathbf{Z}\sim\MPPpi$ can be simulated if the denoiser $\Denois^\pi$ is known, or realistically approximated by an intensity function $(t,z)\mapsto k^{\theta}_t(z)$, yielding approximate samples from $\pi$. While $\Denois^\pi$ is generally intractable, it can be learnt since, by definition, we have $\Denois_t^\pi(Z_t) = \PE[X|Z_t]$ for any $t\geqslant 0$. This conditional expectation is the unique minimizer of any functional of the form
    $f\mapsto\PE[\breg_{\phi}(X,f(t,Z_t))]$, with $\phi:\rset_+^d\rightarrow \rset$ a strictly convex, differentiable function 
and $(x,y)\in\rset_+^2\mapsto \breg_\phi(x,y)\coloneqq \phi(x)-\phi(y) -(x-y)^T\nabla\phi(y)$ is the \emph{Bregman divergence} associated with $\phi$. A common choice is $x\mapsto\phi(x) = \Vert x\Vert_2^2$, which results in the squared Euclidean distance. To learn a Poisson denoiser,
\citet{bhattacharya_itdpdm_2026, chen_learning_2023} suggest using $x\mapsto \phi(x) = \rme(x)= \sum_{i=1}^dx_i\log x_i$ instead, yielding the divergence 
\begin{equation}
    \label{eq:bregman}
    (x,y)\in (0,+\infty)^d\times(0,+\infty)^d\mapsto \breg_\rme(x,y) = \sum_{i=1}^dx_i\log(x_i/y_i)-x_i+y_i\eqsp.
\end{equation}

Given a choice of divergence $\breg_\phi$ and a neural networks $(t,z) \mapsto k^{\theta}_t(z)$ parameterized by $\theta$, we consider then the loss function $\theta \mapsto \mathscr{L}(\theta) = \int_0^{T} \PE[\breg_\phi(X,k^{\theta}_t(Z_t))] \rmd t$, with $T>0$.
The pseudo code associated with the minimization of this loss function is provided in \Cref{app:algo}, \Cref{alg:train}.

\paragraph*{Simulation algorithms.}
We explain here how to simulate a counting process $\bfZ$ associated with an intensity of the form $t\mapsto u_t(Z_{t-})$. Proofs and additional details can be found in \Cref{app:mppsimu}.

Exact simulation of the process with such intensity is possible by adapting Ogata's \emph{thinning} algorithm \citep[Algorithms 7.5.IV, 7.5.V]{daley_introduction_2008}. It requires the intensity function's $l^1$ norm to be non-increasing in $t$, for any fixed $z$, which, by \Cref{prop:denoiser}, holds notably for MPPs. This procedure and its validity are postponed to \Cref{app:mppsimu}, \Cref{prop:thinning}, \Cref{alg:thinning_1d}. 
This exact simulation method is a notable difference from Gaussian diffusion models, which necessarily induce a discretization error. However, the execution time of thinning is uncertain, as $t$ increases according to exponentially distributed random variables. Furthermore, the execution time is expected to scale roughly linearly with the dimension, as each step increments only one coordinate.

Alternatively, we may use an algorithm with fixed execution time, at the cost of a discretization error. \Cref{fig:thinning} provides a visual representation of both simulation methods.
For any $t>0$, consider a time-grid $0=t_0<t_1<\dots<t_{N}=t$, with $\Delta_k\coloneqq t_{k+1}-t_k$.
 Define the Euler discretization $\bfZtild$  of $\bfZ$ as the  counting process with the following recursion: $\ztild_0=0$ and for every $k\geqslant0, s\in[t_k,t_{k+1}]$,\\
\begin{minipage}[t]{0.55\textwidth}
 conditionally on $\ztild_{t_k}$, set
 \vspace{-4pt}
\begin{equation}
\ztild_{s} = \ztild_{t_k} + \vphantom{}^k N_{s-t_k}^{u_{t_k}(\ztild_{t_k})}\eqsp,
\end{equation}
where the $\vphantom{}^k\bfN^{u_{t_k}(\ztild_{t_k})}$ is a standard Poisson process with constant intensity $u_{t_k}(\ztild_{t_k})$ and independent of the past. Said differently, $\bfZtild$ has $\filt^{\bfZtild}$-stochastic intensity $s\mapsto \lambda_s \coloneqq \sum_{k=0}^{N-1}\1_{[t_k,t_{k+1})}(s)u_{t_k}(\ztild_{t_k})$. In particular, $\tilde{Z}_{t_{k+1}}-\tilde{Z}_{t_k} \sim \poisson(\Delta_ku_{t_k}(\tilde{Z}_{t_k}))$.
We show in \Cref{prop:discretization}, \Cref{app:proofs} that under mild regularity assumptions on $u$, $\PE[\Vert\ztild_t-Z_t\Vert_1] = \mathcal{O}(\max_k \Delta_k)$ as $\max_k\Delta_k \rightarrow 0$. We also estimate the empirical discretization error on a simple example; see \Cref{fig:discretized_bound}.
\end{minipage}
\hfill
\begin{minipage}[t]{0.42\textwidth}
\begin{algorithm}[H]
\caption{Time-dicretization}
\label{alg:discretization_1d}
\begin{algorithmic}[1]
\Require $(t,z)\mapsto u_t(z)\in\rset_+^d$; time-grid $0=t_0<t_1<...<t_N$
\State  Set $\ztild_0 = 0$ 
\For{$k=0,\dots,N-1$}
\State Compute $u_{t_k}(\ztild_{t_k})$ 
\State Draw $n_k\sim \poisson(\Delta_ku_{t_k}(\ztild_{t_k}))$
\State Set $\ztild_{t_{k+1}} = \ztild_{t_k} + n_k$
\EndFor
\State \Return $\ztild_{t_0},\ztild_{t_1},\cdots,\ztild_{t_N}$
\end{algorithmic}
\end{algorithm}
\end{minipage}

\paragraph*{Error bounds.} In order to evaluate the quality of approximate $\pi$ samples upon simulation of $\bfZ$, we first derive an \emph{early stopping error}, due to the restriction of $\bfZ$ to a finite time-window. 

\begin{restatable}{proposition}{early}
\label{lemma:early}
    Let $\bfZ\sim\MPPpi$ with random rate $X$. 
    Assume $\PE[\Vert X\Vert^2_2]<\infty$; then for every $t>0$,
    \begin{equation}
\textstyle    \mathscr{W}_2(\law(Z_t/t),\pi)^2 \leqslant \PE[\Vert Z_t/t-X\Vert^2_2] = \PE[\Vert X\Vert_1]/t\eqsp,
    \end{equation} 
    where $ \mathscr{W}_2$ is the $2$-Wasserstein distance defined in \Cref{app:further_notation}, \eqref{eq:wass2}.
\end{restatable}
Compared to the analogous bound for Gaussian diffusion (\Cref{sec:diffusion_gaussian}), there is an additional dependence on $\PE[\Vert X\Vert_1]$ due to the equidispersion of the Poisson distribution. This result is reminiscent of \citep[Proposition 2]{saremi_chain_2024}, where $X$ is estimated from a discrete set of measurements.
\begin{remark}
\label{rmk:early}
    Instead of approximating $X$ by $Z_t/t$ for large $t$, using the denoiser $\PE[X|Z_t]$ instead is preferable; see also \Cref{fig:bounds_early_denois}. Indeed, a short computation shows that
    \begin{equation}
   \textstyle     \PE\left[\Vert X-\PE[X|Z_t]\Vert^2_2\right] = \PE\left[\left\Vert X-Z_t/t\right\Vert^2_2\right] - \PE\left[\left\Vert \PE[X\vert Z_t]-Z_t/t\right\Vert^2_2\right]\eqsp.
    \end{equation}
    Thus $\PE[X|Z_t] \rightarrow X$ in $L^2$, with an early stopping error no larger than the one obtained using $Z_t/t$.
\end{remark}

In our next result, we show that the distributions of two counting processes can be quantitatively controlled in terms of the integrated discrepancy between their respective intensity functions, provided one is Lipschitz. As a direct consequence, we obtain a bound on the distance between the law of an MPP targeting $\pi$ and that of an approximate MPP driven by an intensity $k^\theta\approx \Denois^\pi$.

\begin{restatable}{lemma}{estimationerror}
\label{lemma:estimation_error}
Let $\bfZ,\bfY$ be counting processes with respective intensities $t\mapsto u_t(Z_{t-}), t\mapsto v_t(Y_{t-})$, where $u,v$ are measurable, locally integrable functions.
    Assume that for $t>0$, the integrated discrepancy $\delta(t)\coloneqq\int_0^t \PE[\Vert u_s(Y_s)-v_s(Y_s)\Vert_1]\rmd s$ is finite, and finally that $t\mapsto \mathrm{Lip}(u_t) = \sup_{y\neq z}\Vert u_t(z)-u_t(y)\Vert_1/\Vert z-y\Vert_1$ is locally integrable. Then, for $t>0$,
    \begin{equation}
  \textstyle  \mathscr{W}_{1,1}\left(\law\left(Z_t\right),\law\left(Y_t\right)\right) \leqslant \delta(t)\exp\left(\int_0^t\mathrm{Lip}(u_s) \rmd s\right)\eqsp,
    \end{equation}
    where $\mathscr{W}_{1,1}$ denotes the $1$-Wasserstein distance associated to the $l^1$ distance: see \Cref{app:further_notation}, \eqref{eq:wass1}.
\end{restatable}
The bound derived in \Cref{lemma:estimation_error} is relatively sharp: in \Cref{ex:gamma_estimation}, we consider a tractable case where $\pi$ is Gamma distributed on $\rset_+$, and show that the upper bound differs from the actual Wasserstein distance by only a logarithmic factor in $t$. Finally,
combining the results above, we can establish, under appropriate conditions, theoretical guarantees on the sampling performance of MPDM; see \Cref{app:mppsimu}, \Cref{col:bound}.

\subsection{Related work and comparison with diffusion models}
\label{sec:diffusion_gaussian}

\paragraph{Comparison with classical Gaussian diffusion models.}
We choose here the same diffusion noising process as, \eg, \citet{grenioux_stochastic_2024}. While many diffusion models are parametrized through their score, Tweedie's formula shows that one can equivalently work with the denoiser; with this construction, the denoiser appears naturally, which is reminiscent of the MPDM construction of the previous section.
Consider a target distribution $\pi$ on $\rset^d$, $X\sim\pi$, and the process $\bfX^\rmn=\processt{X^\rmn}_{t\geqslant 0}$ defined as $  X^\rmn_t = tX + W_t$,
where $\bfW$ is a $d$-dimensional Brownian motion independent of $X$. Similar to mixed Poisson processes, we have $X^\rmn_t/t \rightarrow X$ when $t\rightarrow \infty$ almost surely. Given $t>0$, we denote the conditional distribution of $X$ given $X^\rmn_t$ as $ \pi_t(\rmd x\vert x_t) \propto \rme^{-\Vert x_t-tx\Vert^2/(2t)}\pi(\rmd x)$, and define the \emph{Gaussian denoiser} as the function $(t,x_t)\in\rset_+\times \rset^d\mapsto D_t^\pi(x_t)=\int x\times  \pi_t(\rmd x\vert x_t)$.

A priori, it is not possible to simulate this process without first sampling from $\pi$. However, assuming $\pi$ has a finite first moment, the solution of the stochastic differential equation 
\begin{equation}
\label{eq:diffusion_sde}
    \rmd X_t^{\rmd} = D^\pi_t(X_t^{\rmd})\rmd t + \rmd B_t\eqsp, 
\end{equation}
where $\bfB$ is another Brownian motion, is a Markov process with the same time-marginal distributions as $\bfX^\rmn$.
The Gaussian denoiser can be learned by minimization of a functional, like the Poisson denoiser. Samples from the joint distribution $(X,X^\rmn_t)$ are obtained by 1) drawing $x\sim\pi$ from a dataset and 2) computing the noisy observation $x_t\sim X^\rmn_t$ as $tx + \sqrt{t}G$, with $G\sim \gauss(0,\Idd)$.
Once the denoiser is learned, we obtain approximate samples from $\pi$ by simulating \eqref{eq:diffusion_sde} on a discretized time grid.
However, diffusion models are not limited to generative modeling tasks: a pre-trained model may also be used to solve inverse problems, which is the object of \cref{sec:inv}.

\paragraph{Related works.} 
Poissonian generative models were previously considered by \citet{chen_learning_2023,bhattacharya_itdpdm_2026}, but their constructions differ substantially from ours. 

\citet{chen_learning_2023} introduced JUMP, which encodes $X$ into a Poisson-distributed latent variable $Z_{t_0}\sim\poisson(\lambda X)$, constructs a discrete-time Markov chain $
0\approx Z_{t_N}, Z_{t_{N-1}},\dots,Z_{t_1},Z_{t_0}$
through a binomial thinning argument, and derives the corresponding training objective from a variational lower bound. Their resulting training and sampling procedures (\ie, inference algorithm) are closely related to our \Cref{alg:train,alg:discretization_1d} when the Bregman divergence \eqref{eq:bregman} is used to train the denoiser; see \Cref{app:training} for further details. The main algorithmic differences between JUMP and MPDM are analogous to those between DDPMs \citep{ho2020denoising} and SDE-based diffusion models \citep{song2021score}. Beyond these algorithmic similarities, the continuous-time framework developed here also enables us to derive theoretical guarantees for MPDMs.

\citet{bhattacharya_itdpdm_2026} adopt an information-theoretic perspective, in the spirit of the diffusion framework developed by \citet{kong_information-theoretic_2023}. Conditionally on $X$, they consider, for any \emph{fixed} $t\geqslant 0$, an observation channel of the form $
Z_t\sim\poisson(tX)$.
They then propose learning the Poisson denoiser $\PE[X\vert Z_t]$ through a reconstruction loss related to the ELBO derived in \citep{chen_learning_2023}. They also introduce a sampling procedure that differs substantially from both JUMP and our method. Through a controlled synthetic experiment, we show in \Cref{app:mppsimu}, \Cref{rmk:inference} that this inference procedure appears less effective than the one considered in \Cref{alg:discretization_1d}.
\section{Poisson Bayesian inverse problems}
\label{sec:inv}
We next apply the MPDM framework to inverse problems. We review recent diffusion-based methods that leverage a pre-trained diffusion model in \Cref{sec:inverse_gaussian}.
These approaches have been successfully applied to diverse inverse problems, but mostly ones driven by Gaussian noise. In this section, we consider Poisson inverse problems of the form $Y \sim \poisson(\mca(X))$, with forward operator $\mca:\R_+^d\rightarrow \rset_+^{d_y}$. The objective is to sample from $\pi^y\propto g(y\vert x)\pi(\rmd x)$, the posterior distribution of $X$ given $Y=y$, with prior distribution $\pi$ supported on $\rset_+^d$ and likelihood function $g$ defined in \eqref{eq:posterior_poisson}.

We show how a pretrained denoiser $k^\theta\approx \Denois^\pi$ associated with $\MPPpi$ can be used to approximately sample from $\mathrm{MPP}(\pi^y)$, hence from the posterior distribution $\pi^y$, in close analogy with diffusion-based posterior sampling methods. In particular, we adapt DPS \citep{chung_diffusion_2024}, PnP-DM \citep{wu_principled_2024} and DAPS \citep{zhang_improving_2025} to the MPDM framework.

\paragraph*{Base case: Poisson denoising.}
We first consider the Poisson inverse problem
$Y\sim\poisson(aX)$, $a>0$, \ie, $\mca=a\Idd$.
Given an observation $Y=y$, \Cref{th:twist} shows that the \emph{conditional} denoiser $\Denois^{\pi^y}$ of $\bfZ^y\sim \mathrm{MPP}(\pi^y)$ can be obtained through a shift of  $\Denois^\pi$ inputs.
Thus, simulating $\bfZ^y$ requires neither additional approximation nor retraining.
We refer to the associated posterior sampling method as  \emph{Poisson Denoising Posterior Sampling} (PDPS); see \Cref{alg:poisson_denoising}. This result is the Poisson analogue of the ``Gaussian likelihood, Gaussian diffusion" case; see \Cref{sec:inverse_gaussian}, \eqref{eq:denoiser_tilt_gaussian}.

\begin{restatable}{theorem}{thpoisson}
\label{th:twist}
    Let $\pi^y$ denote the posterior distribution of $X$ given $Y=y$, where $Y\sim\poisson(aX)$ and $X\sim\pi$. The denoiser $\Denois^{\pi^y}$ associated with $\mathrm{MPP}(\pi^y)$ is given by $     (t,z)
        \mapsto
        \Denois^{\pi^y}_t(z)
        =
        \Denois^\pi_{t+a}(z+y)$.
\end{restatable}

We next turn to general forward operators and show how several diffusion-based posterior sampling algorithms can be adapted to the MPDM framework.
In this section, we state only the main results underlying these adaptations; implementation details and proofs are deferred to \Cref{app:algo}.
Although our main motivation is Poisson inverse problems, the remainder of this section does not require any specific assumption on the likelihood $g$ and therefore extends beyond Poisson inverse problems; we only require that the prior $\pi$ be supported on $\rset_+^d$.

\paragraph*{A DPS-like algorithm.}
Consider a general forward operator $\mca$ and the associated likelihood $g$.
Recall that $\pi_t(\rmd x\vert z_t)$ denotes the conditional distribution of $X$ given $Z_t=z_t$, for $\bfZ\sim\MPPpi$, as defined in \eqref{eq:denoising_dist}.
Define
$
g_t(y\vert z_t)
\coloneqq
\int_{\rset_+^d} g(y\vert x)\pi_t(\rmd x\vert z_t)$.
The following result gives an exact expression for the posterior denoiser in terms of a discrete likelihood-ratio correction.

\begin{restatable}{theorem}{dpspoisson}
\label{th:approx_dps}
    Consider the posterior distribution $\pi^y$ of a general Poisson inverse problem with likelihood function $g$ defined in \eqref{eq:posterior_poisson}. For any $z_t\in\N^d$ and $t>0$, writing $e_i = (\delta_{i,j})_{1\leqslant j\leqslant d}$, we have
    \begin{equation}
    \label{eq:denoiser_y}
        [\Denois_t^{\pi^y}(z_t)]_i
        =
        [\Denois_t^{\pi}(z_t)]_i
        \frac{g_t(y\vert z_t+e_i)}
        {g_t(y\vert z_t)}
         =
        [\Denois_t^{\pi}(z_t)]_i
        \left(
        1+
        \frac{g_t(y\vert z_t+e_i)-g_t(y\vert z_t)}
        {g_t(y\vert z_t)}
        \right)\eqsp,
        1\leqslant i\leqslant d\eqsp.
    \end{equation}
\end{restatable}

This expression is reminiscent of its diffusion counterpart (see \Cref{sec:inverse_gaussian}, \eqref{eq:approx_dps}), with a \emph{discrete derivative} replacing the usual continuous gradient.
Leveraging the usual DPS Dirac approximation
$\pi_t(\rmd x\vert z_t)
\approx
\delta_{\Denois_t^\pi(z_t)}(\rmd x)$, then approximating the discrete derivative by a continuous one, we get
\begin{align}
    [\Denois_t^{\pi^y}(z_t)]_i
    &\approx
    [\Denois_t^{\pi}(z_t)]_i
    \left(
    1+
    \frac{
    g(y\vert \Denois_t^\pi(z_t+e_i))
    -
    g(y\vert \Denois_t^\pi(z_t))
    }{
    g(y\vert \Denois_t^\pi(z_t))
    }
    \right)
    \label{eq:dps_poisson_approx}
    \\
    &\approx
    [\Denois_t^{\pi}(z_t)]_i
    \left(
    1+\partial_i \log g(y\vert \Denois_t^\pi(z_t))
    \right)\eqsp,
\end{align}
where the domain of $\Denois_t^\pi(\cdot)$ is extended to $\rset_+^d$ using the expression of \Cref{prop:denoiser}, and also implicitly assuming its differentiability. Note that non-negativity of the denoiser is not necessarily guaranteed with this approximation: in practice, we project it to the positive orthant.
Alternatively, one may retain the discrete derivative and approximate \eqref{eq:dps_poisson_approx} by evaluating
$\Denois_t^\pi(z_t+\mathbf{1})-\Denois_t^\pi(z_t)$,
thereby implicitly assuming that $[\Denois_t^\pi(z_t)]_i$ depends only on $z_t^i$.
Our empirical results suggest that the continuous-derivative approximation performs better at large $t$, whereas the ``independent coordinates'' approximation can be slightly more accurate at small $t$; see \Cref{app:dps_poisson} for additional details.
Combining either approximation of $\Denois^{\pi^y}$ with the simulation scheme of \Cref{alg:discretization_1d} yields our MPDM analogue of DPS.
Pseudocode is provided in \Cref{alg:poisson_dps}.

\begin{example}
\label{ex:exact_derivative}
From \Cref{app:gamma_case}, when
$\pi=\gamma=\loigamma(\alpha,\beta)$, $\alpha,\beta\in(0,+\infty)^d$, we have $[\Denois_t^\gamma(z_t)]_i
=
(z_t^i+\alpha_i)/(\beta_i+t), 1\leqslant i\leqslant d$.
In this case, both derivative approximations are exact.
\end{example}

The \pigdm{} algorithm \citep{song_pseudoinverse-guided_2022} relies on an approximation different from that used in DPS. The denoising distribution is approximated by a Gaussian distribution. For Gaussian inverse problems with linear $\mca$, this approximation makes the conditional score, or equivalently the conditional denoiser, tractable. However, \pigdm{} does not directly apply to Poisson inverse problems.

\paragraph*{Split Gibbs sampling methods.}
We adapt PnP-DM by introducing the augmented distributions
\begin{equation}
    \label{eq:pnpdm_poisson}
    \pi^y_k(\rmd x, \rmd \tilde{x})
    \propto
    g(y\vert \tilde{x})
    \exp\left(-\mathscr{D}_{\mathrm{e}}(\tilde{x},x)/\rho_k^2\right)
    \pi(\rmd x)\eqsp,
\end{equation}
which mirror the structure of \eqref{eq:pnpdm_gaussian}, where $\mathscr{D}_{\mathrm{e}}$ denotes the Bregman divergence defined in \eqref{eq:bregman}.
The prior step, namely sampling from $\pi^y_k(\rmd x\vert \tilde{x})$, can be performed using $\MPPpi$.
More precisely, $\pi^y_k(\rmd x \vert \tilde{x})
\propto
\pi_{t_k}\left(
\rmd x
\left\vert
t_k\tilde{x}
\right.
\right)$ with $
t_k\coloneqq 1/\rho_k^2$;
see \Cref{prop:split_gibbs}.
The likelihood step, \ie, sampling from $\pi^y_k(\rmd \tilde{x}\vert x)$, requires an MCMC procedure.
To handle the positivity constraint, we use a Riemannian Langevin method; see \Cref{alg:rla}, \Cref{app:pnpdm}, and \citep{girolami_riemann_2011}.

\begin{restatable}{proposition}{pnpdm}
\label{prop:split_gibbs}
    For every $t>0$ and $z_t\in\N^d$, $\pi_t(\rmd x\vert z_t)
        \propto
        \exp\left(-t\mathscr{D}_{\mathrm{e}}(z_t/t,x)\right)\pi(\rmd x)$.
\end{restatable}

To adapt DAPS to our setting, we approximate the denoising distribution as $\pi_t(\rmd x\vert z_t)
\approx
\loigamma(r_t^{-2}\hat{x}(z_t),r_t^{-2}\mathsf{1})$,
where typically $\hat{x}(z_t)=K_t^\pi(z_t)$.
When $\pi$ is Gamma distributed, the denoising distribution is itself Gamma distributed and is exactly of this form, providing a natural justification for the approximation; see \Cref{app:daps}.
As in the original DAPS method, $\hat{x}(z_t)$ may alternatively be estimated by simulating the MPP forward from $z_t$.
The remaining steps of the algorithm can then be adapted directly to the MPDM setting; see \Cref{app:daps}, \Cref{alg:poisson_daps}.

\section{Poisson inverse problems: numerical experiments}
\label{sec:expe}
Despite our focus on inverse problems, we emphasise that MPDMs are generative models, and that prior works on Poissonian generative modelling suggest potential advantages over classical diffusion models for certain data distributions; see \citep{chen_learning_2023, bhattacharya_itdpdm_2026}. 
For the experiments below, we trained MPDMs on the \emph{BrainWeb} dataset \citep{aubert-broche_twenty_2006,schramm_simulated_2021}, and mock galaxy images from the \emph{EAGLE} simulations \citep{the_eagle_team_eagle_2017}; see \Cref{app:gallery} for a gallery of MPDM-generated images, and \Cref{app:datasets} for datasets details.

\paragraph{Benchmark.} Several diffusion-based posterior sampling algorithms are formulated for additive, homogeneous Gaussian noise, and do not directly apply to Poisson inverse problems, even with Gaussian approximation $\poisson(\mca(X))\approx \gauss(\mca(X), \mathrm{diag}(\mca(X)))$. DPS can accommodate both the true Poisson likelihood and its Gaussian approximation; the latter is the canonical version described by \citet{chung_diffusion_2024}. PnP-DM and DAPS likelihood steps can theoretically handle a Poisson likelihood.
Our benchmark can be separated into three blocks. 1) Our MPDM-based algorithms adapting DPS (\Cref{alg:poisson_dps}), PnP-DM (\Cref{alg:pnp_dm}) and DAPS (\Cref{alg:poisson_daps}) for general inverse problems; or PDPS (\Cref{alg:poisson_denoising}) for the Poisson denoising problems.
2) Diffusion-based models with the exact Poisson likelihood: DPS, PnP-DM and DAPS; or
3) with Gaussian-approximated likelihood. Each method uses the same number of NFEs (\ie, pre-trained denoiser calls). Computational costs nevertheless differ, notably because PnP-DM and DAPS require MCMC steps; see \Cref{app:expe} for details.
We train MPDMs and Gaussian diffusion-based models from scratch for each dataset, with identical architectures and training budgets; see \Cref{app:settings}. 

The performance on realistic Poisson inverse problems is evaluated through several metrics: Peak Signal-to-Noise Ratio (PSNR), Structural Similarity Index Measure
(SSIM), as well as Learned Perceptual Image Patch Similarity (LPIPS). These metrics evaluate the quality of a single reconstructed image, and \Cref{tab:inv_result_small} shows averages over the validation datasets. We provide a complementary evaluation of posterior sampling quality in \Cref{app:more_pet}.

\begin{minipage}[t]{0.41\columnwidth}
\vspace{2pt}
    \paragraph{Toy example.} To demonstrate the capacity of PDPS to accurately sample from the Bayesian posterior $\pi^y$ associated with the Poisson denoising problem ($\mca = a\Idd$), we consider a fully tractable example, where $\pi$ is a mixture of Gamma distributions; see \Cref{app:gamma_case} for details. We train a Poisson and a Gaussian generative model; although the Poisson model is tractable in this setting, we learn it to ensure fairness. We draw 10 observations $y$ in both a low-count (LC) and a high-count (HC) regime ($a=0.6$ and $a=3$, respectively), consider the associated $\pi^y$, and evaluate the sampling performance of different algorithms. 
    We report on \Cref{tab:toy_inv_small,tab:toy_inv_large} the 2-Wasserstein distance between exact $\pi^y$ samples and approximate ones, averaged over the observations $y$. PDPS significantly outperforms its competitors: indeed, except for the ``intrinsic" errors described in \Cref{sec:mpp_def}, this method introduces no additional sampling bias.
    
\end{minipage}
\hfill
\begin{minipage}[t]{0.57\columnwidth}
\vspace{0pt}
\vspace{-2ex}
\begin{figure}[H]
    \centering
    \includegraphics[width=0.8\linewidth]{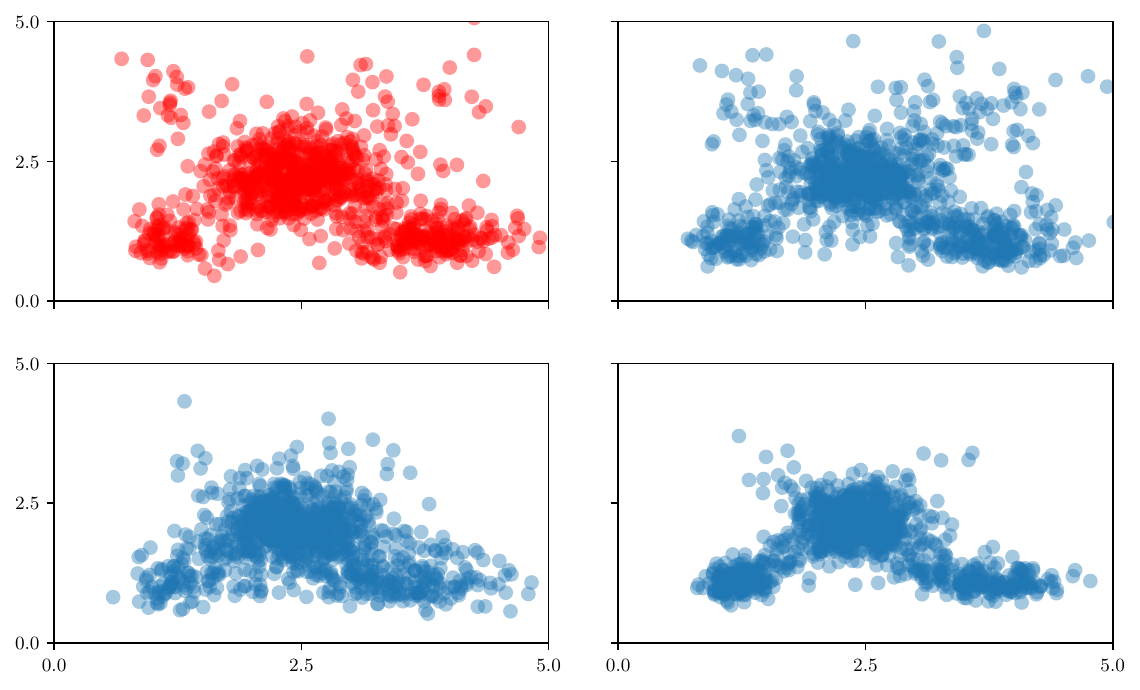}
    \caption{\textcolor{red}{True} and \textcolor{blue}{approximate} $\pi^y$ samples, obtained with PDPS (top-right), diffusion DPS (bottom-left) and PnP-DM (bottom-right).}
    \label{fig:comparison}
\end{figure}
\vspace{-5ex}
\begin{table}[H]
    \scriptsize   
        \centering
        \begin{tabular}{c|c|c}
        \toprule
        \textbf{Methods} & \textbf{Toy (LC)} & \textbf{Toy (HC)}  \\
        \midrule
            \textbf{PDPS (MPDM)} & \colorbox{red!25}{$0.0945\,{\scriptstyle\pm\,0.0053}$} & \colorbox{red!25}{$0.0732\,{\scriptstyle\pm\,0.0023}$} \\
            DPS-p (diffusion) & $0.2134\,{\scriptstyle\pm\,0.0046}$ & $0.2999\,{\scriptstyle\pm\,0.0038}$ \\
            
            \bottomrule
        \end{tabular}
        \caption{Average 2-Wasserstein distance between $\pi^y$ exact samples and approximate ones.}
        \label{tab:toy_inv_small}
    \end{table}
\end{minipage}

\paragraph{Poisson inverse problems.}
We consider two classical Poisson inverse problems. PET \citep{bailey_positron_2005}, where $\mca$ is a Radon operator, and astronomical imaging, where $\mca$ is a Gaussian blur operator \citep{donath_joint_2024}, as well as the denoising problem ($\mca=a\Idd$).
 Each problem is studied in a both low-count regime and a high-count regime.
 For Poisson denoising problems, despite its low computational budget, PDPS consistently ranks first or second. MPDM-based DPS yields the best LPIPS for the blur problems, and MPDM-based PnP-DM the best overall performance for PET. Diffusion-based DPS and DAPS remain competitive.
\begin{table}[t]
\scriptsize
\centering
\setlength{\tabcolsep}{1pt}
\begin{tabular}{c l | c c c | c c c| c c c| c c c}
\toprule
\multicolumn{2}{c}{\multirow{2}{*}{\textbf{Methods}}} & \multicolumn{3}{c}{\textbf{EAGLE denoising (LC)}} & \multicolumn{3}{c}{\textbf{EAGLE denoising (HC)}} 

 & \multicolumn{3}{c}{\textbf{BrainWeb denoising (LC)}} & \multicolumn{3}{c}{\textbf{BrainWeb denoising (HC)}} \\
 \cmidrule(lr){3-5}
\cmidrule(lr){6-8}
\cmidrule(lr){9-11}
\cmidrule(lr){12-14}

 & & PSNR {\scriptsize$\uparrow$} & SSIM {\scriptsize$\uparrow$} & LPIPS {\scriptsize$\downarrow$} & PSNR {\scriptsize$\uparrow$} & SSIM {\scriptsize$\uparrow$} & LPIPS {\scriptsize$\downarrow$}
 
 & PSNR {\scriptsize$\uparrow$} & SSIM {\scriptsize$\uparrow$} & LPIPS {\scriptsize$\downarrow$} & PSNR {\scriptsize$\uparrow$} & SSIM {\scriptsize$\uparrow$} & LPIPS {\scriptsize$\downarrow$}\\
\midrule
\midrule
& \textbf{PDPS (MPDM)} & \colorbox{red!25}{$32.34$} & \colorbox{red!10}{$0.7916$} & \colorbox{red!25}{$0.1543$} & \colorbox{red!25}{$34.94$} & \colorbox{red!25}{$0.8435$} & \colorbox{red!25}{$0.0863$} & \colorbox{red!10}{$16.90$} & \colorbox{red!10}{$0.7812$} & \colorbox{red!10}{$0.1124$} & \colorbox{red!10}{$21.31$} & \colorbox{red!25}{$0.9019$} & \colorbox{red!25}{$0.0585$}\\

& \textbf{DPS (MPDM)} & \colorbox{red!10}{$31.55$} & $0.74729$ & \colorbox{red!10}{$0.1783$} & $33.16$ & $0.7748$ & $0.1369$ &\colorbox{red!10}{$16.91$} & $0.7574$ & $0.1429$ & $20.49$ & $0.8477$ & $0.1207$\\

& \textbf{PnP-DM (MPDM)} & $26.25$ & $0.5933$ & $0.4149$ & \colorbox{red!10}{$34.80$} & \colorbox{red!10}{$0.8401$} & \colorbox{red!10}{$0.0991$} & 
$16.65$ & $0.6278$ & $0.2258$ & \colorbox{red!25}{$21.72$} & $0.8171$ & $0.1408$\\

& DPS-p (diffusion) & \colorbox{red!25}{$32.46$} & \colorbox{red!25}{$0.8060$} & $0.2213$ & \colorbox{red!10}{$34.69$} & \colorbox{red!25}{$0.8409$} & $0.1612$  & \colorbox{red!25}{$17.22$} & \colorbox{red!25}{$0.7915$} & \colorbox{red!25}{$0.0890$} & $20.62$ & \colorbox{red!10}{$0.8814$} & \colorbox{red!25}{$0.0591$}\\

 & DAPS-p (diffusion) & $29.92$ & $0.7268$ & $0.2547$ & $34.60$ & \colorbox{red!10}{$0.8388$} & $0.1546$ & \colorbox{red!25}{$17.12$} & \colorbox{red!10}{$0.7794$} & $0.1141$ & \colorbox{red!10}{$21.41$} & $0.8796$ & \colorbox{red!10}{$0.0668$}\\

\bottomrule
\end{tabular}

\vspace{.1ex}

\begin{tabular}{c l | c c c | c c c| c c c| c c c}
\toprule
\multicolumn{2}{c}{\multirow{2}{*}{\textbf{Methods}}} & \multicolumn{3}{c}{\textbf{EAGLE blur (LC)}} & \multicolumn{3}{c}{\textbf{EAGLE blur (HC)}} 

 & \multicolumn{3}{c}{\textbf{BrainWeb PET (LC)}} & \multicolumn{3}{c}{\textbf{BrainWeb  PET (HC)}} \\
 \cmidrule(lr){3-5}
\cmidrule(lr){6-8}
\cmidrule(lr){9-11}
\cmidrule(lr){12-14}

 & & PSNR {\scriptsize$\uparrow$} & SSIM {\scriptsize$\uparrow$} & LPIPS {\scriptsize$\downarrow$} & PSNR {\scriptsize$\uparrow$} & SSIM {\scriptsize$\uparrow$} & LPIPS {\scriptsize$\downarrow$}
 
 & PSNR {\scriptsize$\uparrow$} & SSIM {\scriptsize$\uparrow$} & LPIPS {\scriptsize$\downarrow$} & PSNR {\scriptsize$\uparrow$} & SSIM {\scriptsize$\uparrow$} & LPIPS {\scriptsize$\downarrow$}\\
\midrule
\midrule

& \textbf{DPS (MPDM)} & \colorbox{red!10}{$31.13$} & $0.7638$ & \colorbox{red!25}{$0.2049$} & $31.64$ & $0.7764$ & \colorbox{red!25}{$0.1858$} & $15.56$ & $0.7084$ & $0.1395$ & $16.84$ & $0.7503$ & $0.1318$ \\

& \textbf{PnP-DM (MPDM)} & $29.83$ & $0.7341$ & $0.3039$ & \colorbox{red!25}{$33.22$} & \colorbox{red!10}{$0.8000$} & $0.2511$ & \colorbox{red!25}{$16.41$} & $0.7329$ & \colorbox{red!25}{$0.1051$} & \colorbox{red!25}{$18.11$} & \colorbox{red!25}{$0.7940$} & \colorbox{red!25}{$0.0828$}
\\

& DPS-p (diffusion) & \colorbox{red!25}{$31.72$} & \colorbox{red!25}{$0.7935$} & \colorbox{red!10}{$0.2432$} & $32.56$ & \colorbox{red!25}{$0.8017$} & \colorbox{red!10}{$0.2129$} & $15.82$ & \colorbox{red!10}{$0.7401$} & \colorbox{red!25}{$0.0968$} & $16.66$ & $0.7758$ & \colorbox{red!10}{$0.0947$}\\

 & DAPS-g (diffusion) & \colorbox{red!25}{$31.99$} & \colorbox{red!10}{$0.7881$} & \colorbox{red!10}{$0.2652$}  & \colorbox{red!25}{$33.14$} & \colorbox{red!25}{$0.8041$} & $0.2365$ & \colorbox{red!10}{$16.13$} & \colorbox{red!25}{$0.7478$} & \colorbox{red!10}{$0.1111$} & \colorbox{red!10}{$17.60$} & \colorbox{red!25}{$0.7930$} & \colorbox{red!10}{$0.1115$}\\

\bottomrule
\end{tabular}
\caption{Quality of reconstructed images for different posterior sampling algorithms and Poisson inverse problems. Colors show \colorbox{red!25}{best} and \colorbox{red!10}{second best}, with ties due to overlapping confidence intervals. We only show best-performing methods: complete results are deferred to \Cref{app:settings}. DPS-p and DAPS-p use the Poisson likelihood, and DAPS-g the Gaussian-approximated one.}
\label{tab:inv_result_small}
\end{table}
\section{Conclusion}
We construct and study ``Poissonian" diffusion models, named MPDMs, using the multivariate counting processes framework. We motivate the use of MPDM beyond generative modeling by adapting several diffusion-based posterior sampling algorithms; in particular, we demonstrate mathematically and empirically that MPDMs are well adapted to Poisson inverse problems, which remain relatively under-explored in this literature. 

\subsection*{Acknowledgments}
Part of this work has been
supported by the PRIMaL project (Carnot Institute). AD a obtenu un financement de la Fondation de l'École polytechnique dans le cadre de sa campagne Servir la science.
AD is supported by the France 2030 program with the reference ANR-25-PEIA-0001 (THEOREM project)
AD is funded by the European Union (ERC-2022-SYG-OCEAN-101071601). Views and opinions expressed are however
those of the author(s) only and do not necessarily reflect those of the European Union or the European Research Council
Executive Agency. Neither the European Union nor the granting authority can be held responsible for them.
AD is supported by Hi! Paris and Agence Nationale de la Recherche (Grant 11-LABX-0047)
This work received government funding administered by the National Research Agency (ANR) under the France 2030 program "Hi! PARIS", grant number ANR-23-IACL-0005.

\subsection*{AI use statement}
Generative AI tools were used to assist with minor English-language editing and polishing of parts of the main text. They were also used to assist with coding during the implementation and experimentation, with the resulting code systematically verified by the authors. All mathematical content, including the derivation of the theoretical results, was developed by the authors without the use of generative AI tools.

\subsection*{Reproducibility statement} \Cref{app:more_counting,app:algo} contain detailed proofs of the mathematical statements as well as the pseudo-code associated with every algorithm mentioned in the main text. Implementation details are reported in \Cref{app:expe}. Both datasets used in \Cref{sec:expe} are public; see also \Cref{app:datasets}.

\bibliography{poisson_clean}
\bibliographystyle{plainnat}

\appendix

\section{Further notation}
\label{app:further_notation}
Throughout this paper, $\N=\{0,1,\dots\}$ denotes the set of natural numbers, $\rset$ the set of real numbers and $\rset_+\coloneqq [0,+\infty($ the set of non-negative real numbers. Given $d\geqslant 1$, we write a $d$-dimensional vector as $x=(x_1,\dots,x_d)\in\rset^d$. $\Vert x\Vert_2$ denotes the usual Euclidean norm, and $\Vert x\Vert_1 = \sum_{i=1}^d \vert x_i \vert$ the $l^1$ norm. 

Given two probability distributions $\mu,\nu$ on $\rset^d$, the 2-Wasserstein distance is defined as
\begin{equation}
    \label{eq:wass2}
    \mathscr{W}_2(\mu,\nu) = \inf\limits_{\rho\in \mathscr{C(\mu,\nu)}}\left(\int_{\rset^d\times \rset^d} \Vert x-y\Vert_2^2\,\rho(\rmd x,\rmd y)\right)^{1/2}\eqsp,
\end{equation}
where $\mathscr{C}(\mu,\nu)$ denotes the set of probability measures on $\rset^d\times\rset^d$ with first (respectively second) marginal equals to $\mu$ (respectively $\nu$). In particular, if $X\sim \mu, Y\sim\nu$, then $\mathscr{W}_2(\mu,\nu)^2\leqslant \PE\left[\Vert X-Y\Vert_2^2\right]$. In this work, we also use the 1,1-Wasserstein distance, which we define as 
\begin{equation}
    \label{eq:wass1}
    \mathscr{W}_{1,1}(\mu,\nu) = \inf\limits_{\rho\in \mathscr{C(\mu,\nu)}}\int_{\rset^d\times \rset^d} \Vert x-y\Vert_1\,\rho(\rmd x,\rmd y)\eqsp.
\end{equation}

The exponential distribution with rate $\lambda>0$ is denoted by $\Exp(\lambda)$, and $\unif_{[a,b]}$ is the uniform distribution on $[a,b]\subseteq \rset$.
The $d$-dimensional normal distribution with mean $\mu\in\rset^d$ and isotropic variance $\sigma^2>0$ is denoted by $\gauss(\mu,\sigma^2\Idd)$, where $\Idd$ is the identity matrix in dimension $d$. If $\alpha\in\rset_+^d$, the multivariate Poisson distribution with mean $\alpha$ is denoted by $\poisson(\alpha)$: a $\nset^d$-valued random variable $Z\sim \poisson(\alpha)$ has independent coordinates $Z_i\sim\poisson(\alpha_i)$, with $\poisson(\alpha_i)$ the usual univariate Poisson distribution with mean $\alpha_i\in\rset_+$. Similarly, given $\alpha,\beta\in(0,+\infty)^d$, we write the multivariate Gamma distribution (with independent coordinates) as $\loigamma(\alpha,\beta)$. 
Finally, if $x\in\rset_+^d\setminus\{0\}$, the categorical distribution $\mathrm{Cat}(x)$ is the probability distribution on $\{1,\dots,d\}$ defined by $\mathrm{Cat}(x)(\{i\}) = x_i/\Vert x\Vert_1$.

\section{Mixed Poisson processes: additional content}
\label{app:more_counting}
This section is an extended version of \Cref{sec:mpp}, with additional statements, detailed proofs, and illustrations. 

\subsection{Background on counting processes}
\label{app:point_proc}
\paragraph{Point processes.}
In \Cref{def:counting}, we define counting processes as time-indexed stochastic processes with unit-jump trajectories, and adopt a martingale approach to characterize them through their compensator (or stochastic intensity). While this presentation suffices to develop our generative framework and is furthermore analogous to SDE-based diffusion models, counting processes are usually part of the \emph{point processes} literature. We provide a brief overview of this topic and refer to \citep{daley_introduction_2008} for a detailed introduction.

A \emph{point process} \citep[Chapter 9]{daley_introduction_2008} $\xi$ is an integer-valued random measure on an adequate measurable metric space $\msx$ equipped with its Borel $\sigma$-algebra $\mcb(\msx)$. To be precise, for any $\omega\in\Omega$, $\msa\in \mcb(\msx) \mapsto \xi(\omega)(\msa) \in\N$ is an integer-valued, locally finite measure on $\msx$; for a fixed $\msa\in\mcb(\msx)$, $\omega \in\Omega\mapsto \xi(\omega)(\msa) \in \N$ is an integer-valued random variable.

A \emph{simple} point process satisfies, for almost all $\omega\in\Omega$, $\xi(\omega)(\{t\}) \in\{0,1\}$ for any $t\in\msx$, and can be written as a purely atomic measure: almost surely, there exists a countable set $\{t_i\}_i$, such that $\{t_i\}\cap \msa$ is finite for any compact $\msa\in\mcb(\msx)$, and
\begin{equation}
    \xi = \sum_{i}\delta_{t_i}\eqsp.
\end{equation}

Univariate counting processes can be viewed as simple point processes on $\rset_+$: given a counting process $\bfZ$ with jump times $(T_i)_{i\geqslant 1}$, $\xi = \sum_{i\geqslant 1}\delta_{T_i}$ is a simple point process on $\rset_+$. Reciprocally, given a simple point process on $\rset_+$ $\xi$, $t\mapsto \xi([0,t])$ is a counting process. 

A simple \emph{marked point process} is obtained by replacing the ground space $\msx$ with, typically, $\msx\times\{1,\dots,d\}$: 
\begin{equation}
    \xi = \sum_{i}\delta_{(t_i,k_i)}\eqsp,
\end{equation}
with $k_i\in\{1,\dots,d\}$. A multivariate counting process $\bfZ$ can be viewed as a marked point process on $\rset_+\times\{1,\dots,d\}$, writing, for $1\leqslant k\leqslant d, t\geqslant 0$, $Z_t^k = \xi([0,t]\times\{k\})$.

A useful special case is the following: $\Pi$ is said to be a \emph{Poisson point process with intensity measure $\mu$}, with $\mu$ a measure on $\msx$, if it satisfies $\law(\Pi(\msa)) = \poisson(\mu(\msa))$ for every measurable $\msa\in\mcb(\msx)$; and if furthermore $\Pi(\msa_1),\dots,\Pi(\msa_n)$ are independents when $\msa_1,\dots,\msa_n$ are disjoint measurable sets. 

We may use a Poisson point process on $\rset_+\times\rset_+$ to build a univariate counting process $\bfZ$ from its $\filt$-stochastic intensity $\lambdabf$. 
Consider a $\filt$-Poisson point process $\Pi$ with intensity measure $\rmd t\rmd x$ on $\rset_+\times\rset_+$, \ie,  for every $\msa\in\mcb(\rset_+)$, $\Pi([0,t]\times\msa)$ is a $\filt$-adapted Poisson process. As a side note, the $\filt$-Poisson point process has the strong Markov property: conditionally on $\Ft_T$, with $T$ a $\filt$-stopping time, the restriction of $\Pi$ to $[T,+\infty)\times\rset_+$ is still a Poisson point process, independent of $\Ft_T$.

Following \citep[Theorem 2.8.11]{chevallier_modelling_2016} or \citep[Proposition 7.5.I]{daley_introduction_2008}, if $\filt^\bfZ\subseteq\filt$, we have for every $t\geqslant 0$,
\begin{equation}
\label{eq:integral}
    Z_t = \int_0^t \int_{\rset_+} \1_{[0,\lambda_{s-}]}(x) \Pi(\rmd s,\rmd x)\eqsp.
\end{equation}
Writting $\Pi = \sum_i \delta_{(s_i,x_i)}$, we have $Z_t = \sum_i \1_{s_i\leqslant t}\1_{[0,\lambda_{s_i-}]}(x_i)$: the Poisson point process generates points $(s_i,x_i)\in\rset_+\times\rset_+$, and $Z_{s_i} = Z_{s_i-}+1$ when $\lambda_{s_i-}\geqslant x_i$.

\paragraph{Inhomogeneous Poisson processes.}
A univariate counting process whose stochastic intensity is a deterministic function is called an \emph{univariate inhomogeneous Poisson process}. A \emph{multivariate inhomogeneous Poisson process} $\bfN$ has independent coordinates which are inhomogeneous Poisson processes; in particular, its $\filt^\bfN$-stochastic intensity is a locally integrable function $t\mapsto \alpha_t=(\alpha_t^1,\dots,\alpha_t^d)\in\rset_+^d$. Alternatively, we may define it through the following transition mechanism. 
\begin{lemma}
    \label{lemma:inhomogeneous}
    The jump times $(T_i)_{i\geqslant 1}$ of an inhomogeneous Poisson process $\bfN$ satisfy, for every $i\geqslant 0$,
    \begin{equation}
    \label{eq:jump_exp}
        \PP(T_{i+1}>t\vert \Ft^\bfN_{T_i}) = \exp\left(-\int_{T_i}^t\Vert \alpha_s\Vert_1\rmd s\right)\eqsp,
    \end{equation}
    with $T_0=0$.
    Given the $(i+1)$-th jump time, the coordinate that jumps is chosen proportionally to the intensity, namely, for every $1\leqslant k\leqslant d$,
    \begin{equation}
        \PP(N^k_{T_{i+1}} = N^k_{T_i}+1 \vert \Ft^\bfN_{T_i},T_{i+1}) = \alpha^k_{T_{i+1}}/\Vert \alpha_{T_{i+1}}\Vert_1\eqsp.
    \end{equation}
\end{lemma}
\begin{proof}
This characterization of an inhomogeneous Poisson process is relatively standard: for completeness, we justify it below.

The sum of $\bfN$ coordinates, $\Vert \bfN\Vert_1$, is a univariate, inhomogeneous Poisson process with $\filt^\bfN$-intensity $t\mapsto \Vert\alpha_t\Vert_1$. To see it, we may compute for $t>s$,
\begin{equation}
    \PE\left[\Vert N_t\Vert_1-\Vert N_s\Vert_1\vert \Ft_s^\bfN\right] = \sum_{k=1}^d \PE\left[N_t^k-N_s^k\vert \Ft_s^\bfN\right] = \int_s^t\Vert\alpha_u\Vert_1\rmd u\eqsp.
\end{equation}
Now, using \eqref{eq:integral}, $\Vert \bfN\Vert_1$ can we written as $\Vert N_t\Vert_1 = \int_0^t \int_{\rset_+} \1_{[0,\Vert\alpha_s\Vert_1]}(x)\Pi(\rmd s,\rmd x)$, where $\Pi$ is a $\filt^{\bfN}$-Poisson point process. Then, for $t\geqslant T_i$,
    \begin{align}
    \label{eq:exp_jump_sum}
        \PP(T_{i+1}>t \vert \Ft^\bfN_{T_i}) &= \left.\PP\left(\int_{T_i}^t\int_{\rset_+} \1_{[0,\Vert \alpha_s\Vert_1]}(x)\Pi(\rmd s,\rmd x) = 0 \right\vert  \Ft^\bfN_{T_i}\right) \\
        &= \exp\left(-\int_{T_i}^t \Vert \alpha_s\Vert_1\rmd s\right)\eqsp,
    \end{align}
    as, by definition of the Poisson point process, conditionally on $\Ft_{T_i}$,
    \begin{equation}
        \int_{T_i}^t\int_{\rset_+} \1_{[0,\Vert \alpha_s\Vert_1]}(x)\Pi(\rmd s,\rmd x) \sim \poisson\left(\int_{T_i}^t\int_{\rset_+} \1_{[0,\Vert \alpha_s\Vert_1]}(x)\rmd s\rmd x\right)\eqsp,
    \end{equation}
    Here, we used the strong Markov property of the $\filt^\bfN$-Poisson point process $\Pi$: conditionally on $\Ft^\bfN_{T_i}$, with $T_i$ a $\filt^\bfN$-stopping time, the restriction of $\Pi$ to $[T_i,+\infty)\times\rset_+$ is still a Poisson point process. Similarly, we can obtain the jump distribution of each coordinate process:  
    \begin{align}
    \label{eq:exp_jump_coor}
        \PP(N_t^j = N_{T_i}^j \vert \Ft^\bfN_{T_i}) &= \left.\PP\left(\int_{T_i}^t\int_{\rset_+} \1_{[0,\alpha_s^j]}(x)\Pi^j(\rmd s,\rmd x) = 0 \right\vert  \Ft^\bfN_{T_i}\right) \\
        &= \exp\left(-\int_{T_i}^t \alpha_s^j\rmd s\right)\eqsp.
    \end{align}
    We now identify the mark distribution, \ie, conditionally on $\Ft^\bfN_{T_i},T_{i+1}$, the probability that $N_{T_{i+1}}^k=N_{T_{i}}^k+1$ for a fixed $1\leqslant k\leqslant d$. Conditionally on $\Ft^\bfN_{T_i}$, define 
    \begin{equation}
    t\geqslant T_i\mapsto p_j(t) = \alpha_t^j \exp\left(-\int_{T_i}^t \alpha_s^j\rmd s\right)\eqsp, 1\leqslant j \leqslant d \eqsp,
    \end{equation}
    and $t\geqslant T_i \mapsto p(t)= \Vert \alpha_t\Vert_1 \exp\left(-\int_{T_i}^t \Vert \alpha_s\Vert_1\rmd s \right)$. From \eqref{eq:exp_jump_coor},\eqref{eq:exp_jump_sum}, we recognize the conditional density of the first jump time of $\bfN^j$ (respectively $\Vert \bfN\Vert_1$) on $[T_i,+\infty)$. Then, we compute the probability that $\bfN^k$ is the first coordinate to jump after $T_i$, and do so before a given $t\geqslant T_i$:
    \begin{align}
        \PP(T_{i+1}\leqslant t, N_{T_{i+1}}^k = N_{T_{i}}^k +1 \vert \Ft_{T_i}^\bfN) &= \int_{T_i}^t p_k(s_k) \prod_{j\neq k} \int_{s_k}^\infty p_j(s_j) \rmd s_j \rmd s_k \\
        &= \int_{T_i}^t p_k(s_k) \prod_{j\neq k}\exp\left(-\int_{T_i}^{s_k} \alpha_u^j \rmd u\right) \rmd s_k \\
        &= \int_{T_i}^t \alpha_{s_k}^k \exp\left(-\int_{T_i}^{s_k}\Vert \alpha_u\Vert_1 \rmd u\right)\rmd s_k \eqsp.
    \end{align}
Here, we used the fact that the coordinate processes are independent. We may also compute this probability by considering that the next jump of $\Vert \bfN\Vert_1$ occurs at time $s$, and introduce the probability $q_k(s)$ that this jump actually occurs on coordinate $k$:
\begin{align}
    \PP(T_{i+1}\leqslant t, N_{T_{i+1}}^k = N_{T_{i}}^k +1 \vert \Ft_{T_i}^\bfN) &= \int_{T_i}^t p(s) q_k(s)\rmd s \\
    &= \int_{T_i}^t \Vert\alpha_{s}\Vert_1 \exp\left(-\int_{T_i}^{s}\Vert \alpha_u\Vert_1 \rmd u\right) q_k(s) \rmd s\eqsp.
\end{align}
These two expressions coincide by setting $q_k(s) = \alpha^k_s/\Vert\alpha_s\Vert_1$, yielding the desired mark distribution.
\end{proof}

Consider now a multivariate counting process $\bfZ$ whose $\filt^\bfZ$-stochastic intensity is of the form $t\mapsto \lambda_t=u_t(Z_{t-})$, where, for any $z\in\N^d$, $t\mapsto u_t(z)$ is locally integrable. Observe that between two consecutive jumps, \ie, on $[T_i,T_{i+1})$, conditionally on $\Ft_{T_i}$, $\bfZ$ coincides with an inhomogeneous Poisson process with intensity deterministic $t\mapsto u_t(Z_{T_i})$. As a result, we can extend \Cref{lemma:inhomogeneous} as follow.
\begin{proposition}
    \label{prop:inhomogeneous_loc}
    The jump times $(T_i)_{i\geqslant 1}$ of a counting process $\bfZ$ with intensity of the form $t\mapsto u_t(Z_{t-})$ satisfy, for every $i\geqslant 0$,
    \begin{equation}
        \PP(T_{i+1}>t\vert \Ft^\bfZ_{T_i}) = \exp\left(-\int_{T_i}^t\Vert u_s(Z_{T_i})\Vert_1\rmd s\right)\eqsp,
    \end{equation}
    with $T_0=0$.
    Given the $(i+1)$-th jump time, the coordinate that jumps is chosen proportionally to the intensity, namely, for every $1\leqslant k\leqslant d$,
    \begin{equation}
        \PP(Z^k_{T_{i+1}} = Z^k_{T_i}+1 \vert \Ft^\bfZ_{T_i},T_{i+1}) = u^k_{T_{i+1}}(Z_{T_{i}})/\Vert u_{T_{i+1}}(Z_{T_i})\Vert_1\eqsp.
    \end{equation}
    In particular, $\bfZ$ is a Markov process.
\end{proposition}

\subsection{MPP properties: proofs of \Cref{sec:mpp_def}}
\label{app:proofs}

{\denoiser*}
\begin{proof}
This expression of the denoiser is derived in \citep[Theorem 3.4.2, 4.4.2]{zocher_multivariate_2005}; we include it here for completeness.
Let $\bfZ \sim \MPPpi$, with random rate $X$.
Fix $t>0, z\in\N^d$. 
    Since $Z_t$, conditionally on $X$, has independent coordinates distributed as $\mathcal{P}(tX_i), 1\leqslant i\leqslant d$, we have
    \begin{equation}
    \label{eq:PZt}
        \PP(Z_t=z) = \int_{\rset_+^d}p_t(z\vert x)\pi(\rmd x)\eqsp,\quad p_t(z\vert x)\coloneqq \prod_{j=1}^d\rme^{-tx_j}\frac{(tx_j)^{z_j}}{z_j!}\eqsp.
    \end{equation}
    This probability may vanish if, \eg, $\pi(\{0\}\times \rset_+^{d-1}) = 1$ and $z_1\geqslant 1$; it is non-zero if and only if $\pi((\{z_1\}\cup(0,+\infty))\times\dots\times(\{z_d\}\cup(0,+\infty)))>0$, which we assume from now on.
    
     Using Bayes' rule leads to the conditional distribution of $X$ given $Z_t=z$, $\pi_t(\rmd x\vert z)=p_t(z\vert x)\pi(\rmd x)/\PP(Z_t=z)$;
     hence for any $1\leqslant i\leqslant d$, the denoiser $i$-th coordinate satisfies
    \begin{align}
        [\Denois_t^\pi(z)]_i &=\PE[X_i\vert Z_t=z]\\
        &= \frac{1}{\PP(Z_t=z)}\int_{\rset_+^d} x_i\times p_t(z\vert x) \pi(\rmd x) \\
        &= \frac{1}{\PP(Z_t=z)}\int_{\rset_+^d}\frac{z_i+1}{t} \frac{(tx_i)^{z_i+1}}{(z_i+1)!} \times \prod_{j\neq i}\rme^{-tx_j}\frac{(tx_j)^{z_j}}{z_j!}\pi(\rmd x) \\
        &= \frac{z_i+1}{t}\frac{\PP(Z_t=z+e_i)}{\PP(Z_t=z)}\eqsp,
    \end{align}
    recognizing \eqref{eq:PZt} with $z+e_i$. Using again \eqref{eq:PZt}, we get 
    \begin{align}
        [\Denois_t^\pi(z)]_i &= \left.\int_{\rset_+^d}x_i\prod_{j=1}^d\rme^{-tx_j}\frac{(tx_j)^{z_j}}{z_j!}\pi(\rmd x)\right/ \int_{\rset_+^d}\prod_{j=1}^d\rme^{-tx_j}\frac{(tx_j)^{z_j}}{z_j!}\pi(\rmd x)\\
        &= \left.\int_{\rset_+^d}x_i\prod_{j=1}^d\rme^{-tx_j}x_j^{z_j}\pi(\rmd x)\right/ \int_{\rset_+^d}\prod_{j=1}^d\rme^{-tx_j}x_j^{z_j}\pi(\rmd x)\eqsp,
    \end{align}
    after simplifying the common factors, which proves the first part of the proposition.

    Smoothness of $t>0\mapsto \Denois^\pi_t(z)$ follows from the smoothness of the Laplace transform of $\pi$; we refer to \citep[Theorem 3.4.5]{zocher_multivariate_2005} for details.

    Regarding the last point of the proposition let $S\coloneqq \sum_{j=1}^dX_j,\,P\coloneqq \prod_{j=1}^dX_j^{z_j}$; then for $t>0$, 
    \begin{equation}
    \Vert\Denois^\pi_t(z)\Vert_1 = \sum_{i=1}^d\PE[X_i \rme^{-tS}P]/\PE[\rme^{-tS}P] = \PE[S \rme^{-tS}P]/\PE[\rme^{-tS}P]\eqsp, 
    \end{equation}
    \begin{equation}
    \label{eq:derivative_denoise}
        \PE[\rme^{-tS}P]^2\times\frac{\rmd}{\rmd t}\Vert\Denois_t^\pi(z)\Vert_1 = \PE[S\rme^{-tS}P]\PE[S\rme^{-tS}P] - \PE[S^2\rme^{-tS}P]\PE[\rme^{-tS}P]\eqsp.
    \end{equation}
    From the Cauchy-Schwarz inequality, we have  
\begin{align}
\PE[S\rme^{-tS}P]^2&=\PE[S\rme^{-tS/2}P^{1/2}\times \rme^{-tS/2}P^{1/2}]^2 \\
&\leqslant \PE[S^2\rme^{-tS}P]\PE[\rme^{-tS}P]\eqsp,
\end{align}
thus $\eqref{eq:derivative_denoise}\leqslant 0$, 
which concludes the proof.
\end{proof}

{\localization*}
\begin{proof}
The proof is a straightforward generalization of the univariate case \citep[Proposition 4.2]{grandell_mixed_1997}. Conditioning on $X$, we get
\begin{equation}
\PP(Z_t/t\rightarrow X) = \PE\left[\left.\PP\left(\bigcap_{i=1}^d \{Z_t^i/t \rightarrow X_i\}\right) \right\vert X\right]\eqsp.
\end{equation}
By definition, conditionally on $X$, the coordinates $\bfZ^i, 1\leqslant i \leqslant d\eqsp,$ are Poisson processes with respective intensities $X_i, 1\leqslant i\leqslant d$.
As a consequence, it suffices to check that for every $x>0, N_t^x/t \rightarrow x$, where $\bfN^x$ is a univariate Poisson process with rate $x$, whose jump times are denoted by $(T^x_n)_{n\geqslant 1}$. The case $x=0$ is obvious. If $x>0$,
 the strong law of large numbers gives $T^x_n/n \rightarrow 1/x$ a.s., since the increments $T_n^x-T_{n-1}^x$ are independent and distributed as $\Exp(x)$; in particular, $T^x_n\rightarrow +\infty$ a.s. Now fix $\varepsilon >0$: almost surely, there exist $p\in\N$ such that $x-\varepsilon \leqslant n/T_{n+1}$ and $n/T_n\leqslant x+\varepsilon$ for any $n\geqslant p$. If $t\geqslant T^x_p$, there exists $n\geqslant p$ satisfying $T^x_n\leqslant t < T_{n+1}^x$, and
 \begin{equation}
     x-\varepsilon\leqslant n/T_{n+1}\leqslant N_t^x/t \leqslant n/T_n \leqslant x+\varepsilon\eqsp,
 \end{equation}
 concluding the proof.
\end{proof}

{\compensator*}
\begin{proof}
Let $X\sim\pi$ denote the random rate of $\bfZ$.
    First, we show that the result holds by replacing $\Denois_t^\pi(Z_s)=\PE[X|Z_s]$ with $\PE[X|\Ft^\bfZ_s]$. Consider the augmented $\sigma$-algebra $\mathcal{G}_t=\sigma(X)\vee \Ft^\bfZ_t$, satisfying $\Ft_t^\bfZ\subseteq\mathcal{G}_t$. $\bfZ$ is, conditionally on $X$, a multivariate Poisson process with intensity $X$, thus have Poisson-distributed independent increments. Then, for every $t>s\geqslant 0$,
    \begin{align}
    \PE\left[Z_t-\left.\int_0^t \PE[X|\Ft^\bfZ_u] \rmd u\right|\Ft^\bfZ_s\right] &= Z_s+ \PE[\PE[Z_t-Z_s |\mathcal{G}_s]|\Ft^\bfZ_s] - \int_0^s\PE[X|\Ft^\bfZ_u]\rmd u - \int_s^t\PE[X|\Ft^\bfZ_s]du\\
    &= Z_s +\PE[(t-s)X|\Ft^\bfZ_s] - \int_0^s\PE[X|\Ft^\bfZ_u]\rmd u - (t-s)\PE[X|\Ft^\bfZ_s] \\
    &= Z_s - \int_0^s\PE[X|\Ft^\bfZ_u]\rmd u\eqsp,
\end{align}
where we used $\PE[\PE[X\vert\Ft^\bfZ_u]\vert \Ft^\bfZ_s] = \PE[X\vert\Ft^\bfZ_s]$ for $u\geqslant s$. 

Finally, the structure of the mixed Poisson process ensures that $\PE[X|\Ft^\bfZ_t] = \PE[X|Z_t]$. Take $\msa = \{Z_{t_1}=z_1,...,Z_{t_n}=z_n\}\in\Ft^\bfZ_t$, with $0=t_0\leqslant t_1 \leqslant...\leqslant t_n=t$ and for any $0\leqslant i\leqslant n, z_i = (z_i^1,\dots,z_i^d)\in\N^d$ with
for every $1\leqslant j \leqslant d$, $0=z_0^j\leqslant z_1^j\leqslant...\leqslant z_n^j$. As, conditionally on $X$, $\bfZ$ is a multivariate Poisson process with intensity $X$, its coordinates are independent univariate Poisson processes, whose increments are independent and Poisson distributed. Consequently, 
\begin{align}
    \PP(\msa\vert X) &= \left.\PP\left(\bigcap_{j=1}^d \bigcap_{i=1}^n \left\{Z^j_{t_i}-Z^j_{t_{i-1}} = z^j_{i}-z^j_{i-1} \right\}\right\vert X\right) \\
    &= \prod_{j=1}^d\prod_{i=1}^n \rme^{-(t_i-t_{i-1})X_j} \frac{(t_i-t_{i-1})^{z^j_i-z^j_{i-1}}X_j^{z^j_i-z^j_{i-1}}}{(z^j_i-z^j_{i-1})!} \\
    &= \rme^{-t\Vert X\Vert_1} P\times  \prod_{j=1}^d\prod_{i=1}^n  \frac{(t_i-t_{i-1})^{z^j_i-z^j_{i-1}}}{(z^j_i-z^j_{i-1})!}\eqsp,
\end{align}
where $P\coloneqq \prod_{j=1}^d X_j^{z^j_n}$. The double product is a constant that we denote $C_\msa$. Then, 
\begin{align}
    \PE[\1_\msa X] &= \PE[X\PP(\msa|X)] \\
    &= \PE[X\rme^{-t\Vert X\Vert_1}P]\times C_A\eqsp.
\end{align}
 Furthermore, as $\{Z_t=z_n\}\supseteq \msa$ and $\PE[X|Z_t]\1_{Z_t=z_n} = \PE[X|Z_t=z_n]\1_{Z_t=z_n}$,
\begin{align}
    \PE[\PE[X|Z_t]\1_\msa] &= \PE[\PE[X|Z_t]\1_{Z_t=z_n}\1_\msa] \\
    &= \PE[X|Z_t=z_n]\PP(\msa) \\
    &= \Denois_t^\pi(z_n)\times \PE[\rme^{-t\Vert X\Vert_1}P]\times C_\msa \\
    &= \PE[X\rme^{-t\Vert X\Vert_1}P] \times C_\msa\eqsp,
\end{align}
where we use \Cref{prop:denoiser} to obtain the last line.
Thus $\PE[\PE[X|Z_t]\1_\msa] = \PE[\1_\msa X]$ which shows that $\PE[X|Z_t] = \PE[X|\Ft^\bfZ_t]$.

Because its stochastic intensity at $t$ is $Z_t$-measurable rather than simply $\Ft^\bfZ_t$-measurable, it is relatively straightforward to see that $\bfZ$ is a Markov process; see \Cref{prop:inhomogeneous_loc}. Alternatively, we can prove it directly as follows. Keep the introduced notation, and let $\msa' = \{Z_{t_1}=z_1,\dots,Z_{t_{n-1}}=z_{n-1}\} \in \Ft_{t_{n-1}}^\bfZ$; we have $\msa\subseteq \msa'$. Notice that 
\begin{equation}C_\msa/C_{\msa'} = \prod_{j=1}^d (t-t_{n-1})^{z^j_n-z^j_{n-1}}/(z^j_n-z^j_{n-1})!\eqsp.
\end{equation}
Then, 
\begin{align}
    \PP(Z_t=z_n\vert \msa') &= \PP(\msa)/\PP(\msa') \\
    &= \left.\PE\left[\rme^{-t\Vert X\Vert_1}\prod_{j=1}^dX_j^{z^j_n}\right]\right/\PE\left[\rme^{-t_{n-1}\Vert X\Vert_1}\prod_{j=1}^dX_j^{z^j_{n-1}}\right] \times \prod_{j=1}^d \frac{(t-t_{n-1})^{z^j_n-z^j_{n-1}}}{(z^j_n-z^j_{n-1})!}\eqsp.
\end{align}
This expression does not depend on $t_i,z_i, i<n-1$: the same computation, but replacing $\msa'$ with $\{Z_{t_{n-1}}=z_{n-1}\}$, yields the same expression: thus $\PP(Z_{t}=z_n\vert \msa')=\PP(Z_t=z_n\vert Z_{t_{n-1}}=z_{n-1})$, hence the Markov property.
\end{proof}

\subsection{Simulation algorithms: proofs and illustrations}
\label{app:mppsimu}

\begin{proposition}
\label{prop:thinning}
Consider a counting process $\bfZ$ whose $\filt^\bfZ$-intensity has the form $t\mapsto \lambda_t = u_t(Z_{t-})\in\rset_+^d$, with $u$ a locally integrable function such that for every $z\in\N^d$, $t\mapsto \Vert u_t(z)\Vert_1$ is positive and non-increasing.
Let $\bfZthin$ denote the counting process recursively defined as follows.

 Set $\bart_0=0$.
 For every $i\geqslant 0$, conditionally on $\bart_i,\zthin_{\bart_i}$, draw the jump times of a univariate Poisson process on $[\bart_i,+\infty)$ with constant intensity $\Vert u_{\bart_i}(\zthin_{\bart_i})\Vert_1$, that is, $\bart_i<S_1<S_2<\dots$ with iid increments distributed as $\Exp(\Vert u_{\bart_i}(\zthin_{\bart_i})\Vert_1)$. 
 
 Starting from $k=1$, draw $U_k\sim\mathcal{U}_{[0,1]}$. If $U_k<\Vert u_{S_k}(\zthin_{\bart_i})\Vert_1/\Vert u_{\bart_i}(\zthin_{\bart_i})\Vert_1$, set $\bart_{i+1} = S_k$, draw $j\sim \mathrm{Cat}(u_{\bart_{i+1}}(\zthin_{\bart_i}))$ and set $\zthin_{\bart_{i+1}} = \zthin_{\bart_{i}}+e_j$, $e_j=(\delta_{j,l})_{l=1}^d$. Else, increment $k$ and repeat.
 
Then, both processes $\bfZthin,\bfZ$ have the same distribution.
\end{proposition}

\begin{proof}

As, conditionally on $\filt_{T_{i}}^\bfZ$, the counting process $\bfZ$ coincides with an inhomogeneous Poisson process with deterministic intensity $t\mapsto u_t(Z_{T_i})\eqqcolon \alpha_t$ until its next jump, \Cref{prop:inhomogeneous_loc} gives
\begin{equation}
    \PP(T_{i+1}>t\vert Z_{T_i},T_i) = \exp\left(-\int_{T_i}^t\Vert \alpha_s\Vert_1\rmd s\right)\eqsp,
\end{equation}
and coordinate $j$ is chosen with probability $\alpha^j_{T_{i+1}}/\Vert \alpha_{T_{i+1}}\Vert_1$, \ie, the same mechanism used to simulate $\bfZthin$.

Now, we must verify that the distribution of $(\bart_i)_{i\geqslant 0}$ coincides with that of $(T_i)_{i\geqslant 0}$.
Fix $i\in\N$ and $t\geqslant \bart_i$.
    Conditionally on $\bart_i, \zthin_{\bart_i}$, let $M$ denote the number of jumps of a univariate Poisson process on $[\bart_i,t]$ with constant intensity $\Vert u_{\bart_i}(\zthin_{\bart_i})\Vert_1 =\Vert \alpha_{\bart_i}\Vert_1$, and $\bart_i<S_1<\dots<S_M$ the associated jump times. Conditionally on $\bart_i, \zthin_{\bart_i}$ and $M$, which follows a Poisson distribution with mean $(t-\bart_i)\Vert \alpha_{\bart_i}\Vert_1$, the unordered jump times are independent and uniformly distributed on $[\bart_i,t]$. Take $U_1,U_2,...$ a sequence of iid random variables, uniform on $[0,1]$, independent of the other random variables. Then, no jump before $t$ means that every proposed jump is rejected, that is,
    \begin{align}
    \PP(\bart_{i+1}>t|\zthin_{\bart_i},\bart_i) &= \sum_{n\geqslant 0}\PP(M=n|\zthin_{\bart_i},\bart_i)\PP\left(\left.\bigcap_{k=1}^n\{U_k>\Vert\alpha_{S_k}\Vert_1/\Vert\alpha_{\bart_i}\Vert_1\}\right| M=n,\zthin_{\bart_i},\bart_i\right) \\
        &= \exp\left(-(t-\bart_i)\Vert\alpha_{\bart_i}\Vert_1\right)\sum_{n\geqslant 0} \frac{((t-\bart_i)\Vert\alpha_{\bart_i}\Vert_1)^n}{n!} \left(\frac{1}{t-\bart_i}\int_{\bart_i}^t1-\frac{\Vert\alpha_s\Vert_1}{\Vert \alpha_{\bart_i}\Vert_1}\rmd s\right)^n \\
        &= \exp\left(-(t-\bart_i)\Vert\alpha_{\bart_i}\Vert_1\right)\exp\left((t-\bart_i)\Vert\alpha_{\bart_i}\Vert_1-\int_{\bart_i}^t\Vert\alpha_s\Vert_1\rmd s\right) \\
        &= \exp\left(-\int_{\bart_i}^t\Vert\alpha_s\Vert_1\rmd s\right)\eqsp,
    \end{align}
    which equals $\PP(T_{i+1}>t|Z_{T_i},T_i)$. Both $\bfZ$ and $\bfZthin$ are Markovian, and their conditional distributions of the next jump are identical; thus, their distributions match. 
\end{proof}

\begin{figure}
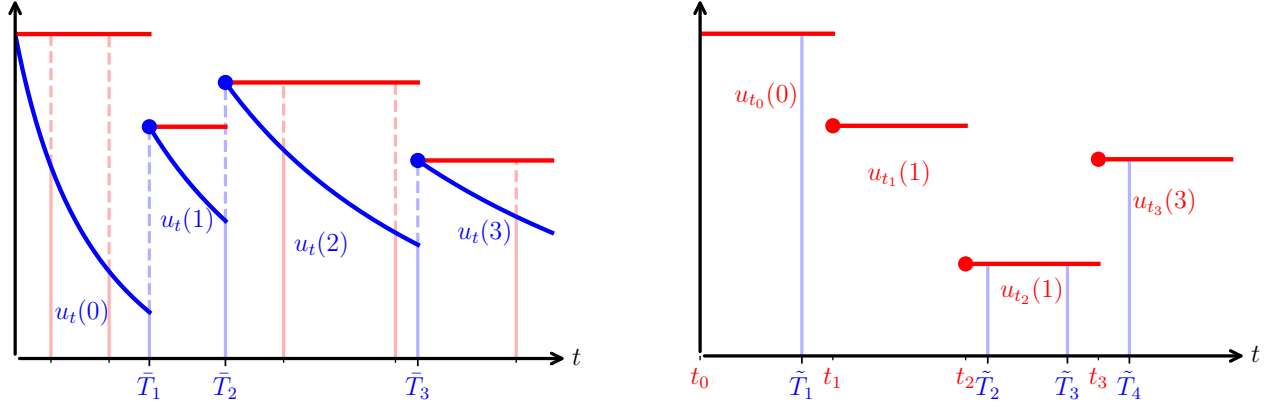

    \centering
    \begin{subfigure}[b]{0.48\linewidth}
      \resizebox{\linewidth}{!}{\input{figures/thinning.pgf}}
    \end{subfigure}
    \hfill
    \begin{subfigure}[b]{0.48\linewidth}
      \resizebox{\linewidth}{!}{\input{figures/discrete.pgf}} 
    \end{subfigure}
    \caption{Visual representation of the simulation procedures when $\pi$ is a Gamma distribution with paramaters $\alpha,\beta=1,2$. \textbf{(Left)} The thinning algorithm of \Cref{prop:thinning}, \Cref{alg:thinning_1d}, in dimension $d=1$. For every $i\geqslant 0, t\geqslant\bart_i$, potential jump times (vertical lines) are drawn according to the ``frozen", piecewise constant intensity
    $u_{\bart_i}(i)$ (\textcolor{red}{red} curve) instead of the true intensity $t\mapsto u_t(i)$ (\textcolor{blue}{blue curve});
    and accepted with probability $u_{t}(i)/u_{\bart_i}(i)=(\textcolor{red}{---})/(\textcolor{red}{\raisebox{0.45ex}{\rule{1cm}{0.4pt}}}+\textcolor{red}{---})$. Accepted jumps times $\bart_i$ are colored \textcolor{blue}{blue}. 
    \textbf{(Right)} The approximate process of \Cref{prop:discretization}, \Cref{alg:discretization_1d}. Given a time-grid $0=t_0<t_1<t_2<\dots$, jump times $\tilde{T}_i$ are drawn according to a discretized version of the intensity (\textcolor{red}{red curve}): on each sub-interval $[t_k,t_{k+1}]$, the number of jumps is distributed as $\poisson((t_{k+1}-t_k)u_{t_k}(\ztild_{t_k}))$.}
    \label{fig:thinning}
\end{figure}

\begin{algorithm}[ht]
\caption{Thinning}
\label{alg:thinning_1d}
\begin{algorithmic}[1]
\Require Intensity $(t,z)\mapsto u_t(z)\in\rset_+^d$ satisfying the assumptions of \Cref{prop:thinning}; $t_{\mathrm{max}} >0$
\State  Initialize $i=0$, $t_0=0,s=0$ and $\zthin_0=0\in\N^d$
\Repeat 
    \State Increment $s=s+\Exp(\Vert u_{t_i}(\zthin_i)\Vert_1)$ \Comment{Potential next jump time}
    \State Compute $u_{s}(\zthin_i)$
    \State Draw $U\sim \mathrm{Unif}_{[0,1]}$
    \If{$U< \Vert u_{s}(\zthin_i)\Vert_1 /\Vert u_{t_i}(\zthin_i)\Vert_1$ and $s<t_{\mathrm{max}}$}   \Comment{Accept or reject jump} 
    \State Draw $j\sim \mathrm{Cat}(u_{s}(\zthin_i))$ \Comment{If accepted, choose which coordinate actually jumps}
    \State Set $\zthin_{i+1}=\zthin_i+e_j, t_{i+1}=s$
    \State Increment $i = i+1$
    \EndIf
    
\Until{$s\geqslant t_{\mathrm{max}}$}
\State \Return $\zthin_i$
\end{algorithmic}
\end{algorithm}

\begin{proposition}
\label{prop:discretization}
    Let $\bfZ$ be a counting process with $\filt^\bfZ$-intensity $t\mapsto u_t(Z_{t-})\in\rset_+^d$, where $u$ is a measurable, locally integrable function.
For any $t>0$, consider a time-grid $0=t_0<t_1<\dots<t_{N}=t$, with $\Delta_k\coloneqq t_{k+1}-t_k$.
    Define the Euler discretization $\bfZtild$  of $\bfZ$ as the  counting process with the following recursion: $\ztild_0=0$ and for every $k\geqslant0, s\in[t_k,t_{k+1}]$, conditionally on $\ztild_{t_k}$,
\begin{equation}
\textstyle
\ztild_{s} = \ztild_{k} + \vphantom{}^k N_{s-t_k}^{u_{t_k}(\ztild_{t_k})}\eqsp,
\end{equation}
where the $\vphantom{}^k\bfN^{u_{t_k}(\ztild_{t_k})}$ is a standard Poisson processes with constant intensity $u_{t_k}(\ztild_{t_k})$ and independent of the past. Said differently, $\bfZtild$ has piecewise constant $\filt^{\bfZtild}$-stochastic intensity $s\mapsto \lambda_s \coloneqq \sum_{k=0}^{N-1}\1_{[t_k,t_{k+1})}(s)u_{t_k}(\ztild_{t_k})$.
Assume the following. 
\begin{itemize}
\item Given any fixed $z\in\N^d$, $s\in[0,t]\mapsto u_s(z)$ is Lipschitz continuous, and there exist constants $a,b,L\geqslant 0$, as well as $0\leqslant L_k\leqslant L, 0\leqslant k\leqslant N-1$, such that the Lipschitz constant of $s\in[t_{k},t_{k+1}]\mapsto u_s(z)$ is upper bounded by $(a\Vert z\Vert_1+b)L_k$.
\item For a fixed $s\in[0,t]$,  $\Lip(u_s)\coloneqq \sup_{z\neq y\in\N^d} \Vert u_s(z)-u_s(y)\Vert_1/\Vert z-y\Vert_1$ if finite, and $\sup_{s\leqslant t} \Lip(u_s)<\infty$. This holds in particular if $s\mapsto \Lip(u_s)$ is continuous.
\end{itemize}
Then, $\PE[\Vert\ztild_t-Z_t\Vert_1] = \mathcal{O}(\max_k \Delta_k)$ as $\max_k\Delta_k \rightarrow 0$.
\end{proposition}

In particular, these assumptions hold when $\bfZ\sim\mathrm{MPP}(\gamma)$ with $\gamma\sim\loigamma(\alpha,\beta)$ a Gamma distribution, in which case $u^i_t(z) = (z_i+\alpha_i)/(\beta_i+t)$; see \Cref{app:gamma_case}.
\begin{proof}

 First, we derive an upper bound on the expected values of $\Vert\bfZtild\Vert_1$ on the grid. Write, for $s\in[t_k,t_{k+1}]$,
\begin{align}
    \PE[\Vert \ztild_{s}\Vert_1] &= \PE[\Vert\ztild_{t_k}\Vert_1] + \PE[\Vert\vphantom{}^kN_{s-t_k}^{u_{t_k}(\ztild_{t_k})}\Vert_1] \\
    &= \PE[\Vert \ztild_{t_k}\Vert_1] + (s-t_k)\PE[\Vert u_{t_k}(\ztild_{t_k})\Vert_1] \\
    &\leqslant \PE[\Vert \ztild_{t_k}\Vert_1] +(s-t_k)\left(\Vert u_{t_k}(0)\Vert_1 + \Lip(u_{t_k})\PE[\Vert \ztild_{t_k}\Vert_1]\right)
    \label{eq:mean_diff_grid}\eqsp,
\end{align}
and in particular $ \PE[\Vert\ztild_{t_{k+1}}\Vert_1] \leqslant \Delta_k \Vert u_{t_k}(0)\Vert_1+ \PE[\Vert\ztild_{t_k}\Vert](1+\Delta_k\Lip(u_{t_k}))$; writing $\mu_k\coloneqq \PE[\Vert \ztild_{t_k}\Vert_1]$, increasing in $k$ due to counting processes being non-decreasing, 
this gives the upper bound
\begin{align}
\label{eq:mean_bound}
\mu_k \leqslant \mu_{N} &\leqslant \sum_{i=0}^{N-1} \Delta_i\Vert u_{t_i}(0)\Vert_1 \prod_{j=i+1}^{N-1} (1+\Delta_j\Lip(u_{t_j})) \\
\label{eq:mean_bound_2}
&\leqslant \sum_{i=0}^{N-1} \Delta_i\Vert u_{t_i}(0)\Vert_1 \times \exp\left(\sum_{j=0}^{N-1}\Delta_j\Lip(u_{t_j})\right)\\
&\leqslant \sum_{i=0}^{N-1} \Delta_i\Vert u_{t_i}(0)\Vert_1 \times \exp\left(t\sup_{s\leqslant t}\Lip(u_s)\right)\\
&\leqslant t \sup_{s\leqslant t}\Vert u_{s}(0)\Vert_1 \times \exp\left(t\sup_{s\leqslant t}\Lip(u_s)\right)\eqsp.
\end{align}
 In particular, there exists a uniform upper bound $K>0$ such that $0=\mu_0\leqslant\mu_1\leqslant\dots\leqslant\mu_N\leqslant K$ for any time-grid.

Next, we use the integral representation \eqref{eq:integral} with $\filt=\filt^{\bfZ}\vee \filt^{\bfZtild}$ and couples both processes through the same $\filt$-Poisson point process $\Pi$. 
For every $1\leqslant i \leqslant d$, 
\begin{align}
    \PE[\vert Z_t^i - \ztild_t^i\vert] &\leqslant \PE\int_0^t\int_{\rset_+} |\1_{[0,u^i_s(Z_{s-})]}(x) - \1_{[0,\lambda^i_{s-}]}(x)| \Pi(\rmd s,\rmd x) \\
    &= \PE\int_0^t\int_{\rset_+} |\1_{[0,u^i_s(Z_s)]}(x) - \1_{[0,\lambda^i_s]}(x)| \rmd x \rmd s \\
    &= \PE\int_0^t |u^i_s(Z_s)-\lambda^i_s| \rmd s \\
    &= \sum_{k=0}^{N-1}\int_{t_k}^{t_{k+1}} \PE[\vert u^i_s(Z_s)-u^i_{t_k}(\ztild_{t_k})\vert] \rmd s \\
    &\leqslant \sum_{k=0}^{N-1}\int_{t_k}^{t_{k+1}} \PE[|u^i_s(\ztild_s)-u^i_{t_k}(\ztild_{t_k})|]\rmd s + \int_0^t \PE[|u^i_s(Z_s)-u^i_s(\ztild_s)|] \rmd s\eqsp,
\end{align}
where the second line is obtained by noting that $s\mapsto |\1_{[0,u^i_s(Z_{s-})]}(x) - \1_{[0,\lambda^i_{s-}]}(x)|$ is $\Ft_s$-predictable by definition of stochastic intensities, and using the compensation formula of the Poisson random measure; see, \eg, \citep[Theorem 4.4]{kyprianou_fluctuations_2014}. Note that we omit the $s-$ in the integral, as, almost surely, this does not change the integral value.
Summing these inequalities, we obtain 
\begin{equation}
    \PE[\Vert Z_t-\ztild_t\Vert_1]\leqslant \sum_{k=0}^{N-1}\int_{t_k}^{t_{k+1}} \PE[\Vert u_s(\ztild_s)-u_{t_k}(\ztild_{t_k})\Vert_1]\rmd s + \int_0^t \Lip(u_s)\PE[\Vert Z_s-\ztild_s\Vert_1] \rmd s\eqsp.
\end{equation}

To bound the first term, write, for $s\in[t_k,t_{k+1}]$,
\begin{align}
    \PE[\Vert u_s(\ztild_s)-u_{t_k}(\ztild_{t_k})\Vert_1] &\leqslant \PE[\Vert u_s(\ztild_s)-u_s(\ztild_{t_k})\Vert_1]+\PE[\Vert u_s(\ztild_{t_k})-u_{t_k}(\ztild_{t_k})\Vert_1] \\
    &\leqslant \Lip(u_s)\PE[\Vert \ztild_s-\ztild_{t_k}\Vert_1] + (s-t_k)L_k(a\PE[\Vert \ztild_{t_k}\Vert_1]+b) \\
    &\leqslant \Lip(u_s)(s-t_k)(\Vert u_{t_k}(0)\Vert_1 +\Lip(u_{t_k})\PE[\Vert \ztild_{t_k}\Vert_1]) + (s-t_k)L_k(a\PE[\Vert \ztild_{t_k}\Vert_1]+b) \\
    &= \Lip(u_s)(s-t_k)(\Vert u_{t_k}(0)\Vert_1 +\Lip(u_{t_k})\mu_k) + (s-t_k)L_k(a\mu_k+b)  \eqsp,
\end{align}
using \eqref{eq:mean_diff_grid} and the assumptions on $u$. This yields
\begin{align}
    &\sum_{k=0}^{N-1}\int_{t_k}^{t_{k+1}} \PE[\Vert u_s(\ztild_s)-u_{t_k}(\ztild_{t_k})\Vert_1]\rmd s \\
    &\leqslant \sum_{k=0}^{N-1} (\Vert u_{t_k}(0)\Vert_1 +\Lip(u_{t_k})\mu_k)\int_{t_k}^{t_{k+1}}\Lip(u_s)(s-t_k)\rmd s 
    + \sum_{k=0}^{N-1} \frac{\Delta_k^2}{2}L_k(a\mu_k+b)\label{eq:constant_gron} \\
    &\leqslant C\coloneqq \sum_{k=0}^{N-1} \left(\sup_{s\leqslant t}\Lip(u_s)(\sup_{s\leqslant t} \Vert u_s(0)\Vert_1+\sup_{s\leqslant t}\Lip(u_s)K) + L(aK+b)\right)\frac{\Delta_k^2}{2} \eqsp.
\end{align}
Using Gronwall's lemma, the discretization error upper bound is 
\begin{equation}
\label{eq:bound_discr}
    \PE[\Vert Z_t-\ztild_t \Vert_1] \leqslant C \exp\left(\int_0^t \Lip(u_s) \rmd s\right)\eqsp.
\end{equation}
Convergence follows, as $C=\mathcal{O}(\Vert\Delta\Vert_2^2)=\mathcal{O}(\max_{k}\Delta_k)$ as $\max_k\Delta_k \rightarrow 0$, concluding the proof.
\end{proof}
\begin{example}
\label{ex:discretized}
We can make the discretization error bound explicit when $\pi=\gamma\sim\loigamma(\alpha,\beta)$. For simplicity, we consider the univariate case. Since $\Denois^\gamma_s(z) = (z+\alpha)/(\beta+s)$, we have $\Lip(u_s) = 1/(\beta+s)$, $u_s(0)=\alpha/(\beta+s)$, $a=1,b=\alpha$ and $L_k = 1/(\beta+t_k)^2$. 

Furthermore, the inequality \eqref{eq:mean_diff_grid} is actually an equality in this specific case, since the denoiser is affine in $z$. Thus $\mu_k$ satisfies the recursion 
\begin{align}
    &\mu_{k+1} = \Delta_k\frac{\alpha}{\beta+t_k} +\mu_k\left(1+\frac{\Delta_k}{\beta+t_k}\right) \\
    &\mu_{k+1}+\alpha = (\mu_k+\alpha)\times\frac{\beta+t_{k+1}}{\beta+t_k}\eqsp,
\end{align}
and as a result, $\mu_k+\alpha = \alpha(\beta+t_k)/\beta$, and finally $\mu_k = t_k\alpha/\beta$.
Using \eqref{eq:constant_gron}, with $\Lip(u_s)\leqslant 1/(\beta +t_k)$ on $[t_k,t_{k+1}]$, we obtain after simplification
\begin{equation}
    \PE[\vert Z_t-\ztild_t\vert] \leqslant \frac{\alpha}{\beta}\sum_{k=0}^{N-1} \frac{\Delta_k^2}{\beta+t_k}\times \left( 1+\frac{t}{\beta}\right)\eqsp.
\end{equation}
We compare this bound with simulations in \Cref{fig:discretized_bound}.

\begin{remark}
    \label{rmk:inference}
    Our sampling procedure based on discretization, \Cref{alg:discretization_1d}, is quite similar to \citep[Algorithm 2]{chen_learning_2023}. Although their construction differs from ours (see \Cref{sec:diffusion_gaussian}), they rigorously derive their algorithm using binomial thinning arguments, while our algorithm leverages the stochastic intensity of counting processes. 
    On the other hand, \citep[Algorithm 2]{bhattacharya_itdpdm_2026} differs substantially from \Cref{alg:discretization_1d} and \citep[Algorithm 2]{chen_learning_2023}: instead of considering an update of the form $\tilde{Z}_{t_{k+1}} = \tilde{Z}_{t_{k}} + n_k, n_k\sim \poisson(\Delta_k u_{t_k}(\tilde{Z}_{t_k}))$, with $u_{t_k}(\tilde{Z}_{t_k})$ approximating the Poisson denoiser, they instead draw $\tilde{Z}_{t_{k+1}} \sim \poisson(t_{k+1}u_{t_{k}}(\tilde{Z}_{t_k}))$. To the best of our understanding, this sampling procedure is not mathematically derived in the paper, and we see no straightforward justification. We assess their sampling procedure with the tractable bivariate Gamma mixture distribution of \Cref{app:settings}, using the same log-uniform time-grid (with $N=64,1024$): we compare the quality of approximate $\pi$ samples obtained with their algorithm and ours, using the analytical expression of the Poisson denoiser. 
    As reported in \Cref{tab:moments}, the samples obtained with \Cref{alg:discretization_1d} accurately reproduce the first two scalar moments of $\pi$, whereas the samples obtained with \citep[Algorithm 2]{bhattacharya_itdpdm_2026} exhibit substantially larger discrepancies.
\end{remark}
\begin{table}[H]
\centering \scriptsize
\begin{tabular}{c |c|c|c|c}
\toprule
\multirow{2}{*}{\textbf{Sampling method}}
    & \multicolumn{2}{c|}{$N=64$}
    & \multicolumn{2}{c}{$N=1024$} \\
    \cmidrule(lr){2-5}
    & $\hat{\mu}$ & $\hat{\sigma}$ & $\hat{\mu}$ & $\hat{\sigma}$\\
\midrule\midrule
\textbf{\Cref{alg:discretization_1d}}
    & 5.03 & 1.95 & 5.00 & 1.98 \\
\citep[Algorithm 2]{bhattacharya_itdpdm_2026}
    & 5.41 & 1.91 & 5.32 & 2.12 \\
\bottomrule
\end{tabular}
\caption{Empirical moments $\hat{\mu},\hat{\sigma}$ obtained with $10^6$ approximate $\pi$ samples, with $\pi$ the bivariate Gamma mixture distribution of \Cref{app:settings}. The exact values of $\mu= \PE[X_1+X_2], \sigma^2 = \PE[(X_1+X_2-\mu)^2]$, are $\mu=5.0,\sigma \approx 1.99$.}
\label{tab:moments}
\end{table}

\end{example}
\begin{figure}
    \centering
\includegraphics[width=\linewidth]{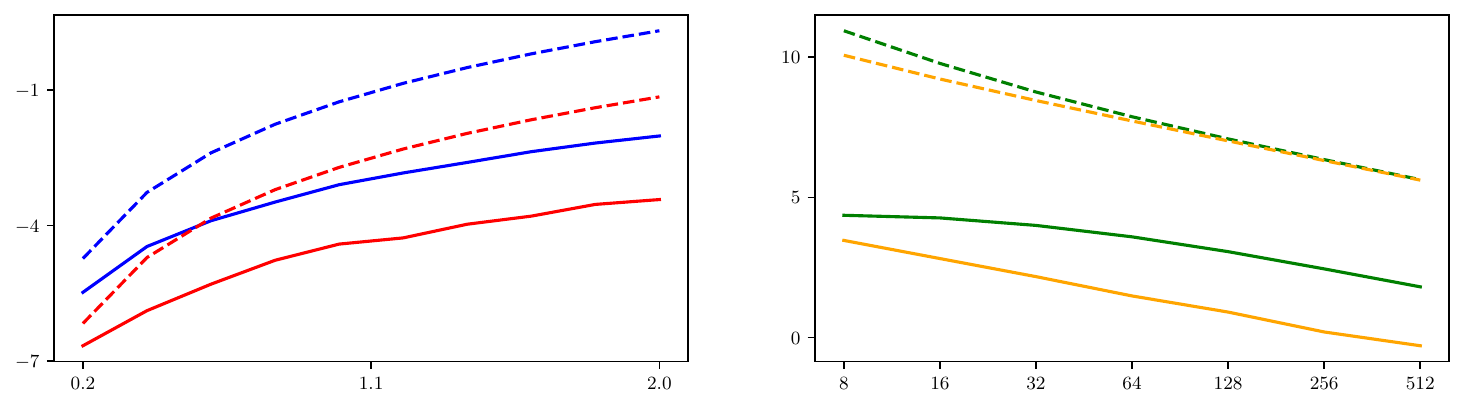}
    \caption{Discretization error for $\bfZ\sim\MPPpi$, $\pi=\Gamma(\alpha=3,\beta=5)$; $\tilde{\bfZ}$ is the discretized process of \Cref{prop:discretization}. The $y$-axis is on a log scale. \textbf{(Left)} Empirical 1-Wassertein distance ($10^7$ samples) between $\law(Z_t)$ and $\law(\tilde{Z}_t)$ (continuous line) and the associated upper theoretical bound of \Cref{ex:discretized} (dashed line) using a uniform time-grid on $[0,t]$, $0.2\leqslant t\leqslant 2$; with either $N=16$ \textcolor{blue}{(blue)} or $N=64$ \textcolor{red}{(red)}. We observe that the theoretical upper bound is relatively sharp for low $t$, with a gap increasing with $t$. \textbf{(Right)} Empirical 1-Wassertein distance between $\law(Z_t)$ and $\law(\tilde{Z}_t)$ (continuous line) and the associated theoretical upper bound of \Cref{ex:discretized} (dashed line) for $t=100$, as a function of $N=8,,16,\dots,512$. We compare the uniform time-grid \textcolor{green}{(green)} and a log-uniform grid \textcolor{orange}{(orange)}. More precisely, we use uniformly spaced $\log(\beta+t_k)$: empirically, this yields a substantially smaller error compared to the uniform time-grid.}
    \label{fig:discretized_bound}
\end{figure}

{\early*}
\begin{proof}
Note that for every $t>0$, conditionally on $X$, $Z_t^i$ is Poisson distributed with mean and variance $tX_i$. Assuming $\PE[\Vert X\Vert^2_2]<\infty,$ we have  
\begin{align}
\PE\left[\Vert Z_t-tX\Vert^2_2\right] &= \PE\left[\PE\left[\Vert Z_t-tX\Vert^2_2\vert X\right]\right] \\
&= \sum_{i=1}^d \PE\left[\PE\left[(Z_t^i-tX_i)^2\vert X\right]\right] \\
&= t\sum_{i=1}^{d}\PE[X_i] = t\PE[\Vert X\Vert_1]\eqsp.
\end{align}
Dividing by $t^2$ gives the result.
\end{proof}

{\estimationerror*}
\begin{proof}
Set $\filt=\filt^\bfZ \vee\filt^{\bfY}$.
    Using the integral representation \eqref{eq:integral} with $\filt$, we may couple $\bfZ$ and $\bfY$ through the same Poisson point process $\Pi$.
Then, for $1\leqslant i\leqslant d$,
\begin{align}
    \PE[|Z^i_t-Y^i_t|] &\leqslant \PE\int_0^t\int_{\rset_+} |\1_{[0,u^i_s(Z_s)]}(x) - \1_{[0,v^i_s(Y_s)]}(x)| \Pi(\rmd s,\rmd x) \\
    &= \PE\int_0^t\int_{\rset_+} |\1_{[0,u_s^i(Z_s)]}(x) - \1_{[0,v_s^i(Y_s)]}(x)| \rmd x \rmd s \\
    &= \PE\int_0^t |u_s^i(Z_s)-v_s^i(Y_s)| \rmd s \\
    &\leqslant \int_0^t \PE[|u_s^i(Y_s)-v_s^i(Y_s)|]\rmd s + \int_0^t \PE[\vert u_s^i(Z_s)-u_s^i(Y_s)\vert] \rmd s\eqsp,
\end{align}
where the second line uses the compensation formula of $\Pi$, see, \eg, \citep[Theorem 4.4]{kyprianou_fluctuations_2014}. Summing for $1\leqslant i\leqslant d$ yields
\begin{align}
    \PE[\Vert Z_t-Y_t\Vert_1] &= \sum_{i=1}^d\PE[\vert Z_t^i-Y_t^i\vert] \\
    &\leqslant \int_0^t \PE[\Vert u_s(Y_s)-v_s(Y_s)\Vert_1]\rmd s + \int_0^t \PE[\Vert u_s(Z_s)-u_s(Y_s)\Vert_1] \rmd s \\
    &\leqslant \delta(t) +\int_0^t \PE[\Vert Z_s-Y_s\Vert_1]\Lip(u_s)\rmd s\eqsp.
\end{align}

Gronwall's inequality leads to the upper bound on $\PE[\Vert Z_t-Y_t\Vert_1]$, and thus on the Wasserstein distance.
\end{proof}

\begin{figure}
\centering
 \includegraphics[width=\linewidth]{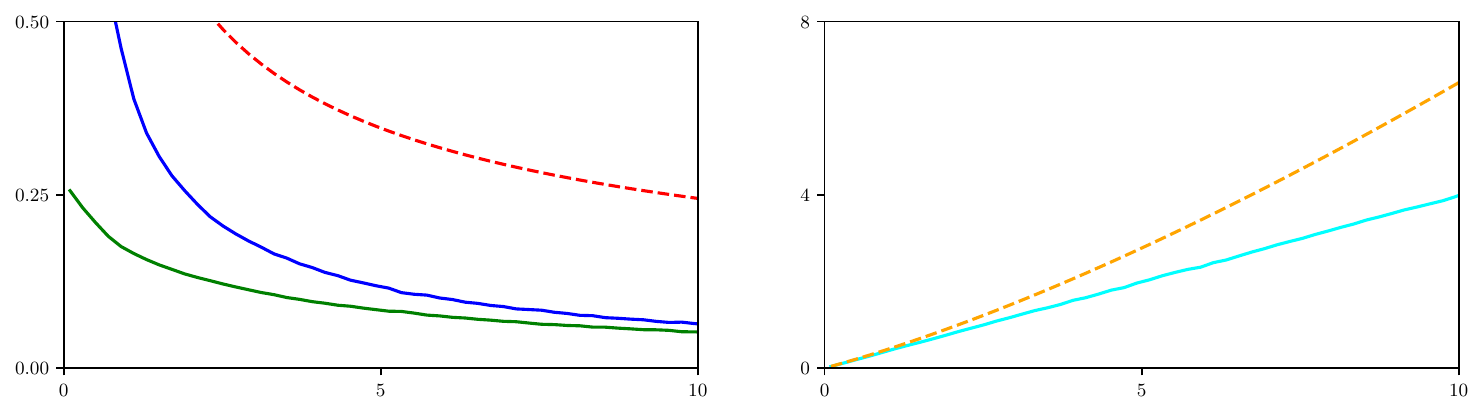}
 \caption{Illustration of \Cref{lemma:early} and \Cref{lemma:estimation_error} with $\bfZ\sim\MPPpi$, $\pi=\Gamma(\alpha=3,\beta=5)$ and $\bfY\sim\text{MPP}(\gamma), \gamma=\Gamma(\tilde{\alpha}=5,\beta=5)$, for $0\leqslant t\leqslant 10$. \textbf{(Left)} Early stopping error. \textcolor{blue}{Blue line} shows the empirical 1-Wasserstein distance between $\law(Z_t/t)$ and $\pi$ as a function of $t$, \textcolor{red}{red dashed line} is the associated (loose) theoretical bound $\sqrt{\PE[X]/t}$,
 and \textcolor{green}{green line} the empirical 1-Wasserstein distance between $\Denois^\pi(Z_t) = (Z_t+\alpha)/(\beta+t)$ and $\pi$; see \Cref{rmk:early}. \textbf{(Right)} Denoiser estimation error. \textcolor{cyan}{Cyan line} is the empirical 1-Wasserstein distance between $\law(Z_t)$ and $\law(Y_t)$, and \textcolor{orange}{orange dashed line} the theoretical upper bound of \Cref{ex:gamma_estimation}.}
\label{fig:bounds_early_denois}
\end{figure}

\begin{example}
    \label{ex:gamma_estimation}
    In dimension 1, if $\bfZ\sim\MPPpi$ with $\pi=\gamma = \loigamma(\alpha,\beta)$, the denoiser is an affine function of $z$: $\Denois_t^\gamma(z) = \frac{z+\alpha}{t+\beta}$; see \Cref{app:gamma_case}. Let $\bfY\sim\mathrm{MPP}(\tilde{\gamma})$, with $\tilde{\gamma}\sim\loigamma(\tilde{\alpha},\beta)$; then $\Denois_t^{\tilde{\gamma}}(z) = \frac{y+\tilde{\alpha}}{t+\beta}$. We have
        $\delta(t) = \int_0^t \PE[|\Denois_s^\gamma(Y_s)-\Denois_s^{\tilde{\gamma}}(Y_s)|]\rmd s = |\alpha-\tilde{\alpha}|\log\left(1+t/\beta\right)$ and $\Lip(\Denois^\gamma_t) = 1/(t+\beta)$, hence the bound  $
        \mathscr{W}_{1,1}(\mathrm{Law}(Z_t),\mathrm{Law}(Y_t)) \leqslant |\alpha-\tilde{\alpha}|\left(1+\frac{t}{\beta}\right)\log\left(1+\frac{t}{\beta}\right)$.
    Here, the exact value of $\mathscr{W}_{1,1}(\mathrm{Law}(Z_t),\mathrm{Law}(Y_t))$ is tractable, and reads $\frac{\vert\alpha-\tilde{\alpha}\vert}{\beta}t$. The general bound we computed differs from the actual distance by only a logarithmic factor; see also \Cref{fig:bounds_early_denois}.
\end{example}

Combining the previous results yields the upper bounds below. Using thinning (\Cref{alg:thinning_1d}) introduces no additional error, leading to \eqref{eq:bound_2}; however, execution time is uncertain. Using time discretization instead (\Cref{alg:discretization_1d}) does not have this drawback, but introduces the additional error derived in \Cref{prop:discretization}, yielding \eqref{eq:bound_3}.
\begin{corollary}
    \label{col:bound}
    Let $\bfZ\sim\MPPpi$ with random rate $X$, and consider an approximation $\bfY$ with stochastic intensity $k^\theta\approx \Denois^\pi$.
    Assume $\PE[\Vert X\Vert_2^2]<\infty$, and that $k^\theta$ satisfies the assumptions of \Cref{lemma:estimation_error}. Then, for any fixed $t>0$, using the notation of \Cref{lemma:estimation_error},
    \begin{equation}
        \label{eq:bound_2}
        \mathscr{W}_{1,1}(\law(Y_t/t),\pi) \leqslant \sqrt{d\frac{\PE[\Vert X\Vert_1]}{t}} + \frac{\delta(t)}{t}\exp\left(\int_0^t \Lip(k_s^\theta)\rmd s\right)\eqsp.
    \end{equation}
    If $\tilde{\bfY}$ is the counting process obtained by discretization of $\bfY$ on a time-grid $0=t_0<t_1<\dots<t_N$, and if $k^\theta$ satisfies the regularity assumptions of \Cref{prop:discretization}, then
    \begin{equation}
        \label{eq:bound_3}
\mathscr{W}_{1,1}(\law(\tilde{Y}_t/t),\pi) \leqslant \sqrt{d\frac{\PE[\Vert X\Vert_1]}{t}} + \frac{1}{t}\left(\delta(t)+C\right)\exp\left(\int_0^t \Lip(k_s^\theta)\rmd s\right) \eqsp,
    \end{equation}
    with $C=\mathcal{O}(\max_{k}\Delta_k)$ being defined in the proof of \Cref{prop:discretization}.
\end{corollary}
\begin{proof}
Using the triangle inequality,
    \begin{align}
        \mathscr{W}_{1,1}(\law(Y_t/t),\pi) &\leqslant \PE[\Vert Y_t/t-X\Vert_1] \\
        &\leqslant \PE[\Vert Z_t/t-X\Vert_1]+\PE[\Vert Z_t-Y_t\Vert_1]/t\eqsp.
    \end{align}
    The first term is upper bounded with \Cref{lemma:early}:
    \begin{align}
    \PE[\Vert Z_t/t-X\Vert_1] &\leqslant\sqrt{d} \PE[\Vert Z_t/t-X\Vert_2] \\
    &\leqslant \sqrt{d}\sqrt{\frac{\PE[\Vert X\Vert_1]}{t}}
    \end{align}
    Plugging in the bounds of \Cref{lemma:early} and \Cref{lemma:estimation_error} yields \eqref{eq:bound_2}. Then,
    \begin{align}
        \mathscr{W}_{1,1}(\law(\tilde{Y}_t/t),\pi) &\leqslant \PE[\Vert \tilde{Y}_t/t-X\Vert_1] \\
        &\leqslant \PE[\Vert Z_t/t-X\Vert_1]+\PE[\Vert Z_t-\tilde{Y}_t\Vert_1]/t \\
        &\leqslant \PE[\Vert Z_t/t-X\Vert_1]+\PE[\Vert Z_t-Y_t\Vert_1]/t + \PE[\Vert \tilde{Y}_t-Y_t\Vert_1]/t
    \end{align}
    yields \eqref{eq:bound_3} using \Cref{prop:discretization}, Equation \eqref{eq:bound_discr}.
\end{proof}

\subsection{The Gamma case}
\label{app:gamma_case}
Diffusion-based methods usually consider Gaussian mixtures as toy examples. For MPDM, the tractable case corresponds to Gamma mixtures. 
\paragraph{MPP with Gamma random rate.}
Consider a mixed Poisson process $\bfZ$ with distribution $\mathrm{MPP}(\gamma)$, where 
\begin{equation}\gamma(\rmd x) = \loigamma(\alpha,\beta) = \prod_{i=1}^d\loigamma(x_i;\alpha_i,\beta_i)\rmd x_i=\prod_{i=1}^d \rme^{-\beta_ix_i}x_i^{\alpha_i-1}\frac{\beta_i^{\alpha_i}} {\Gamma(\alpha_i)}\rmd x_i\eqsp,
\end{equation}
with $\alpha,\beta\in(0,+\infty)^d$. To be clear, $\loigamma(\alpha,\beta)$ is the multivariate Gamma distribution, a probability measure on $\rset_+^d$, $x_i\in\rset_+ \mapsto \loigamma(x_i;\alpha_i,\beta_i)$ is the density of the univariate Gamma distribution, and $\Gamma$ denotes the Gamma function.

Then each component $\bfZ_i,1\leqslant i \leqslant d$ is a \emph{P\'oyla process}, see \citep[Example 4.1]{grandell_mixed_1997}.
Specifically, for every $t\geqslant0$, the components of $Z_t$ are independent and admit a negative binomial distribution:
\begin{align}
\label{eq:poyla}
p_t^\gamma(z)\coloneqq \PP(Z_t=z) &= \int_{\rset_+} \prod_{i=1}^d\rme^{-tx_i}\frac{(tx_i)^{z_i}}{z_i!} \gamma(\rmd x) \\
&= \prod_{i=1}^d\int_0^{\infty}  \rme^{-(t+\beta_i)x_i} \frac{x_i^{z_i+\alpha_i-1}}{z_i!\Gamma(\alpha_i)}t^{z_i}\beta_i^{\alpha_i} \rmd x_i \\
&= \prod_{i=1}^d \frac{\Gamma(z_i+\alpha_i)}{(t+\beta_i)^{z_i+\alpha_i}} \frac{t^{z_i}\beta_i^{\alpha_i}}{z_i!\Gamma(\alpha_i)} \\
&=\prod_{i=1}^d\frac{\Gamma(z_i+\alpha_i)}{z_i!\Gamma(\alpha_i)}\left(1-\frac{\beta_i}{t+\beta_i}\right)^{z_i}\left(\frac{\beta_i}{t+\beta_i}\right)^{\alpha_i}\eqsp,
\end{align}
where we recognize unnormalized Gamma densities to compute the integrals. Given $z\in\N^d$, the denoising density (\ie, the density of $X$ conditionally on $Z_t=z$) is a product of Gamma densities:
\begin{align}
\label{eq:denoising_dist_gamma}
    \pi_t(\rmd x\vert z) = p_t^\gamma(x\vert z)\rmd x &\coloneqq \prod_{i=1}^d \rme^{-tx_i} \frac{(tx_i)^{z_i}}{z_i!} \times \frac{\gamma(\rmd x)}{\PP(Z_t=z)} \\
    &\propto \loigamma(z+\alpha,t\mathsf{1}+\beta)\eqsp.
\end{align}
The denoiser is the mean of this distribution, hence:
\begin{equation}
    \label{eq:denois_gamma}
    [\Denois_t^{\gamma}(z)]_i = \frac{z_i+\alpha_i}{t+\beta_i},\eqsp 1\leqslant i\leqslant d\eqsp.
\end{equation}

\paragraph{Mixture of Gamma.} If $\bfZ\sim\mathrm{MPP}(\mu)$ with $\mu(\rmd x) = \sum_{j=1}^Nw_j\prod_{i=1}^d\loigamma(x_i;\alpha_{ij},\beta_{ij})\rmd x_i$ with $w$ on the simplex, \eqref{eq:poyla} is directly adapted by adding the weighted sum to the equation:
\begin{equation}
    p^\mu_t(z) = \sum_{j=1}^N w_j\prod_{i=1}^d\frac{\Gamma(z_i+\alpha_{ij})}{z_i!\Gamma(\alpha_{ij})}\left(1-\frac{\beta_{ij}}{t+\beta_{ij}}\right)^{z_i}\left(\frac{\beta_{ij}}{t+\beta_{ij}}\right)^{\alpha_{ij}}\eqsp.
\end{equation}
The denoising density is a Gamma mixture, namely:
\begin{equation}
    p_t^\mu(x\vert z)\rmd x = \sum_{j=1}^N \tilde{w}_j \prod_{i=1}^d \loigamma(x_i;z_i+\alpha_{ij},t+\beta_{ij})\rmd x_i\eqsp,
    \tilde{w}_j \coloneqq \frac{w_j}{p_t^\mu(z)} \prod_{i=1}^d\frac{t^{z_i}\beta_{ij}^{\alpha_{ij}}}{\Gamma(\alpha_{ij})z_i!}\frac{\Gamma(z_i+\alpha_{ij})}{(t+\beta_{ij})^{z_i+\alpha_{ij}}}\eqsp,
\end{equation}
hence the denoiser 
\begin{equation}
    \label{eq:denoiser_gamma_mix}
    [\Denois_t^\mu(z)]_i = \sum_{j=1}^N \tilde{w}_j\frac{z_i+\alpha_{ij}}{t+\beta_{ij}}\eqsp, 1\leqslant i\leqslant d\eqsp.
\end{equation}

\section{Overwiew of diffusion-based posterior sampling methods}
\label{sec:inverse_gaussian}
We provide a brief overview of diffusion-based posterior sampling. In particular, we explain the core ideas of the algorithms that are adapted to the MPDM setting in \Cref{sec:inv}. We will refer to the diffusion model introduced in \Cref{sec:diffusion_gaussian} and the associated notations.
\paragraph{Gaussian inverse problems.} Let $\mca:\rset^d\rightarrow \rset^{d_y}$ denote the forward operator, with $1\leqslant d_y\leqslant d$, and consider the posterior distribution $\pi^y$ of $X\sim\pi$ given the observation $Y=y$, where
\begin{equation}
    \label{eq:gaussian_inverse}
    Y = \mca(X)+\sigma\varepsilon\eqsp,
\end{equation}
with $\varepsilon\sim\gauss(0,\Idd)$ and $\sigma>0$. From Bayes' rule, we have  $\pi^y(\rmd x) \propto g(y\vert x)\pi(\rmd x)$, with $g$ the likelihood function: in this case, $g(y\vert x) = \gauss(y;\mca(x),\sigma^2\Idd)$. 

Instead of focusing on a specific forward operator $\mca$, recent training-free methods use a pre-trained diffusion model as a general-purpose prior.
The Gaussian denoiser $D^\pi$ and the observation $y$ are combined to sample from the posterior distribution $\pi^y$ of a given inverse problem; see the review by   \citet{janati_briding_2025}. 

\paragraph*{Guidance-based methods.} Granted access to the \emph{conditional denoiser}, \ie, $D^{\pi^y}$, samples from $\pi^y$ could be obtained by simulating the associated diffusion \eqref{eq:diffusion_sde}.
Diffusion Posterior Sampling (DPS) \citep{chung_diffusion_2024}, approximates the conditional denoiser $D^{\pi^y}$\footnote{This approximation is usually formulated in terms of the diffusion score instead of the denoiser: both formulations are equivalent thanks to Tweedie's formula.} as 
\begin{equation}
\label{eq:approx_dps}
    D_t^{\pi^y}(x_t) = D_t^\pi(x_t) + \nabla \log g_t(y\vert x_t)\approx D_t^\pi(x_t) + \nabla \log g(y\vert D_t^\pi(x_t))\eqsp,
\end{equation}
where the intractable term $g_t(y\vert x_t) \coloneqq \int g(y\vert x)\pi_t(\rmd x\vert x_t)$ is estimated by approximating $\pi_t(\rmd x\vert x_t) \approx \delta_{D_t^\pi(x_t)}(\rmd x)$. 
In practice, evaluating this approximation requires differentiating the learned denoiser with respect to $x_t$.
Alternatively, \pigdm{} \citep {song_pseudoinverse-guided_2022} assumes $\pi_t(\rmd x\vert x_t)$ is Gaussian distributed with mean $D^\pi_t(x_t)$ and covariance matrix $r_t\Idd$, with $r_t>0$.
For Gaussian inverse problems with linear $\mca$, this approximation makes the conditional score, or equivalently the conditional denoiser $D^{\pi^y}$, tractable. However, \pigdm{} does not transpose to Poisson inverse problems.

\paragraph*{Split Gibbs sampling methods.}
Instead of approximating $D^{\pi^y}$ to simulate a diffusion process directly, another class of methods sample from a sequence of intermediary, augmented distributions. For instance, PnP-DM \citep{wu_principled_2024} consider the augmented distributions
\begin{align}
\label{eq:pnpdm_gaussian}
    \pi^y_k(\rmd x, \rmd \tilde{x}) \propto g(y\vert \tilde{x})\rme^{-\Vert x-\tilde{x}\Vert_2^2/(2\rho_k^2)} \pi(\rmd x)\eqsp,
\end{align}
where $(\rho_k)_{k=1}^n\in(0,+\infty)^n$ is typically non-increasing. For each $1\leqslant k\leqslant n$, $x,\tilde{x}$ are updated through a two steps procedure.
The \emph{likelihood step}, \ie, sampling from $\pi^y_k(\rmd \tilde{x}\vert x) \propto g(y\vert \tilde{x})\rme^{-\Vert x-\tilde{x}\Vert_2^2/(2\rho_k^2)}$, is exact for Gaussian inverse problem with linear $\mca$; in general, it can be tackled with standard MCMC methods. The \emph{prior step} leverages the pre-trained diffusion model, as
\begin{equation}
\label{eq:prior_step_gaussian}
    \pi^y_k(\rmd x\vert \tilde{x}) \propto \rme^{-\Vert x-\tilde{x}\Vert_2^2/(2\rho_k^2)} \pi(\rmd x) \propto \pi_{t_k}\left(\rmd x\left\vert t_k\tilde{x}\right)\right.\eqsp,
\end{equation}
setting $t_k\coloneqq 1/\rho_k^2$: the diffusion process is simulated from $t_k$ onward, with initial state $t_k\tilde{x}$. As $\rho_k$ approaches zero, $x\approx\tilde{x}$ is approximately distributed as $\pi^y$.

A related approach is DAPS \citep{zhang_improving_2025}: its structure resembles that of PnP-DM, but includes a third, \emph{noising} step that decouples successive samples to promote a larger exploration of the posterior distribution. We describe DAPS using the diffusion model of \Cref{sec:diffusion_gaussian}, which slightly modifies the formulation of the original paper (which relied on a forward/backward diffusion formalism).
This algorithm revolves around the sequence of distributions 
\begin{equation}
\label{eq:daps_distribution}
    \pi^y_{t_k}(\rmd x\vert x_{t_k}) \propto g(y\vert x)\pi_{t_k}(\rmd x\vert x_{t_k})\eqsp,
\end{equation}
on a time-grid $t_1<\dots t_N$.
DAPS uses the approximation $\pi_{t}(\rmd x\vert x_t) \approx \gauss(\hat{x}(x_{t}),r_{t}^2\Idd)$, see \citep[Equation 5]{zhang_improving_2025}. Here, $\hat{x}(x_t)$ denotes an estimator of $X$ given $X_t=x_t$:
it can be obtained as $D_T^\pi(x_T)$, where $x_T$ is computed by numerical integration of \eqref{eq:diffusion_sde} from $t$ up to a terminal time $T$ using a given time-grid. Using only one denoising step results in the simple estimator $\hat{x}(x_t) = D_t^\pi(x_t)$; this design choice is discussed in \citep[Appendix I]{zhang_improving_2025}. By analogy with PnP-DM, we refer to the estimation of $\hat{x}(x_t)$, for a given $t=t_k$, as the \emph{prior step}. Then, the \emph{likelihood step} consists of sampling from 
\begin{equation}
    \label{eq:daps_likelihood_gaus}
    \pi^y_{t_k}(\rmd x\vert x_{t_k}) \propto g(y\vert x)\gauss(\hat{x}(x_{t_k}),r_{t_k}^2\Idd)
\end{equation}
with an MCMC scheme. Finally, given the output $x$ of the likelihood step, one draws $x_{t_{k+1}} = t_{k+1}x + \sqrt{t_{k+1}}G_{k+1}$, with $G_{k+1}\sim\gauss(0,\Idd)$. This \emph{noising step} is meant to decouple the successive $x_{t_k}$: instead of generating a single trajectory of the denoising process $\bfX^\rmd$ \eqref{eq:diffusion_sde}, we sample a ``fresh" noisy intermediate $x_{t_{k+1}}$ through the noising process $\bfX^\rmn$.

\paragraph{Gaussian denoising.} Diffusion-based posterior sampling methods often exploit the Gaussian nature of the inverse problem, leveraging the matching noise structure between the diffusion and the inverse problem; \eg, \citep{song_pseudoinverse-guided_2022,kawar_denoising_2022}. 
This correspondence is best illustrated
when the problem consists of pure Gaussian noise, \ie, $\mca=a\Idd$ with $a\in\rset$: then,
\begin{equation}
    \label{eq:denoiser_tilt_gaussian}
    D_t^{\pi^y}(x_t) = D^\pi_{t+(a/\sigma)^2}\left(x_t+ay/\sigma^2\right)\eqsp.
\end{equation}
In other words, sampling from the posterior distribution $\pi^y$ is possible by simulating \eqref{eq:diffusion_sde} while shifting the inputs of $D^\pi$, without any further approximation. 

\paragraph{Poisson inverse problems.} 
If $X\sim\pi$ with $\pi$ supported on $\rset_+^d$ and $\mca:\rset_+^d\rightarrow \rset_+^{d_y}$, a Poisson inverse problem consists of sampling from the posterior distribution $\pi^y\propto g(y\vert x)\pi(\rmd x)$ of $X$ given $Y=y$, where  $Y\sim\poisson(\mca(X))$, \ie,
\begin{equation}
\label{eq:posterior_poisson}
    g(y\vert x) \propto \prod_{i=1}^{d_y} \rme^{-[\mca(x)]_i}[\mca(x)]_i^{y_i}\eqsp.
\end{equation}
This class of problems is mostly excluded from the aforementioned posterior sampling methods. Even DPS,  whose approximation \eqref{eq:approx_dps} theoretically holds for any likelihood $g$, eventually replaces the Poisson likelihood by a ``Gaussianized" one due to poor numerical stability; see \citep[Appendix C.4]{chung_diffusion_2024}.
This approximation is only valid when the measured values $y_i$ are large, excluding the regime where some $y_i$ are close to or equal to zero. Gibbs-sampling methods can theoretically handle the Poisson likelihood, but the likelihood step \eqref{eq:pnpdm_gaussian} requires a non-trivial MCMC sampling of a density that is supported on $\rset_+^d$. Moreover, standard gradient-based methods such as Langevin MCMC are problematic near the boundary of the positive orthant, since the Poisson log-likelihood has a singular gradient when $[\mca(x)]_i\rightarrow0$ for $y_i>0$; and unconstrained Gaussian updates can leave the domain.

In stark contrast with the Gaussian case, even when $\mca = a\Idd$, there is no straightforward diffusion-based posterior sampling method. Leveraging \Cref{th:twist}, PDPS precisely fills this gap.

\section{Algorithmics details}
\label{app:algo}
\subsection{Training}
\label{app:training}
\begin{algorithm}[H]
\caption{Training}
\label{alg:train}
\begin{algorithmic}[1]
\Require Dataset $\mathcal{D}$, trainable network $k^\theta$, strictly convex function $\phi:\rset_+^d\rightarrow \rset$, mini-batch size $B$, number of epoch $N$, probability measure $\nu(\rmd t)$ on $\rset_+$ and weight function $w$.
\State Initialize $\theta$
\Repeat
\State Draw mini-batch $x_1,\dots,x_B\in\mathcal{D}$
\State Sample $t_1,\dots,t_B \sim \nu(\rmd t)$
\State Draw $z_i\sim\poisson(t_ix_i), 1\leqslant i\leqslant B$ \Comment{Noising}
\State Compute $k^\theta_{t_i}(z_{i}), 1\leqslant i\leqslant B$ \Comment{Denoising}
\State Perform gradient descent step on $\sum_{i=1}^B \breg_\phi(x_i,k^\theta_{t_i}(z_i))w(t_i)$ \Comment{Minimization of $\mathscr{L}(\theta)$}
\Until{convergence or target number of epochs is reached}
\State \Return learned model $k^\theta$
\end{algorithmic}
\end{algorithm}
For completeness, we provide the training pseudo-code that outputs an approximation of the Poisson denoiser, $k^\theta\approx \Denois^\pi$. This algorithm is similar to the one proposed by \citet{bhattacharya_itdpdm_2026,chen_learning_2023}, albeit with a generic Bregman divergence instead of \eqref{eq:bregman}. For our numerical experiments, we observed that using the usual Euclidean distance instead of \eqref{eq:bregman} had little influence on the learned model performance. The distribution $\nu(\rmd t)$ can be either discrete, if one fixes a time-grid, or continuous, \eg, uniform on some $[t_{\mathrm{min}}, t_{\mathrm{max}}]$. In practice, we set $\nu$ as the log-uniform distribution on $[t_{\mathrm{min}}, t_{\mathrm{max}}]$ and the weight function $t\mapsto w(t)=t$, similar to \citet{bhattacharya_itdpdm_2026}. This increases the number of low-SNR (\ie, low $t$) inputs $z_t$ while preserving the training objective associated with uniform $t$:
\begin{equation} \mathscr{L}(\theta)=\int_{t_\mathrm{min}}^{t_\mathrm{max}}\PE[\breg_\phi(X,k_t^\theta(Z_t))]\rmd t \propto \int_{t_\mathrm{min}}^{t_\mathrm{max}}w(t)\PE[\breg_\phi(X,k_t^\theta(Z_t))]\nu(\rmd t)\eqsp.
\end{equation}
We detail the exact parameters used in our numerical experiments in \Cref{app:expe}.

\subsection{Poisson denoising with PDPS}
\label{app:poisson_denoising}

PDPS (\Cref{alg:poisson_denoising}) is a relatively straightforward implementation of \Cref{th:twist} that solves the inverse problem defined by $Y\sim\poisson(aX), X\sim\pi$ by leveraging a pre-trained MPDM. We provide the proof and pseudo-code below, and show a sample trajectory in \Cref{fig:eagle_pdps}.

{\thpoisson*}
\begin{proof}
    Fix $t>0$ and $z\in\rset_+^d$. Recall that, from
    \Cref{prop:denoiser}, we have
    \begin{equation}
    \label{eq:denois_multid}
        \Denois_t^{\pi}(z) =  \left.\int_{\rset_+^d} x\prod_{j=1}^dx_j^{z_j}\rme^{-tx_j} \pi(\rmd x)\right/\int_{\rset_+^d} \prod_{j=1}^dx_j^{z_j}\rme^{-tx_j} \pi(\rmd x)\eqsp.
    \end{equation}
    Similarly, the denoiser associated to $\mathrm{MPP}(\pi^y)$ writes
    \begin{align}
        \Denois_t^{\pi^y}(z) &= \left.\int_{\rset_+^d} x\prod_{j=1}^dx_j^{z_j}\rme^{-tx_j} \pi^y(\rmd x)\right/\int_{\rset_+^d} \prod_{j=1}^dx_j^{z_j}\rme^{-tx_j} \pi^y(\rmd x) \\
        &= \left.\int_{\rset_+^d} x\prod_{j=1}^dx_j^{z_j}\rme^{-tx_j} \prod_{k=1}^d \rme^{-ax_k}(ax_k)^{y_k}\pi(\rmd x)\right/\int_{\rset_+^d} \prod_{j=1}^dx_j^{z_j}\rme^{-tx_j}  \prod_{k=1}^d \rme^{-ax_k}(ax_k)^{y_k}\pi(\rmd x) \\
        &= \left.\int_{\rset_+^d} x\prod_{j=1}^dx_j^{z_j+y_j}\rme^{-(t+a)x_j} \pi(\rmd x)\right/\int_{\rset_+^d} \prod_{j=1}^dx_j^{z_j+y_j}\rme^{-(t+a)x_j}  \pi(\rmd x)\eqsp,
    \end{align}
    using $\pi^y(\rmd x)\propto g(y\vert x)\pi(\rmd x)\propto\prod_{k=1}^d\rme^{-ax_k}(ax_k)^{y_k}\pi(\rmd x)$ . We recognize the expression of $\Denois_{t+a}^{\pi}(z+y)$, concluding the proof.
\end{proof}

\begin{figure*}[t]
\centering
\begin{algorithm}[H]
\caption{Poisson denoising posterior sampling (PDPS)}
\label{alg:poisson_denoising}
\begin{algorithmic}[1]
\Require Learned denoiser $k^\theta\approx \Denois^\pi$, rate $a>0$, observation $y\in\N^d$, time-grid $0=t_0<t_1<\dots<t_N$, with $\Delta_k = t_{k+1}-t_k$
\State  Initialize $z_0 = 0$
\For{$k=0,\dots,N-1$}
    \State Estimate $\Denois^{\pi^y}_{t_k}(z_k)\approx \tilde{\Denois}_{t_k}^{\pi^y}(z_k)=k^\theta_{t_k+a}(z_k+y)$ \Comment{\Cref{th:twist}}
    \State Draw $n_k\sim\poisson(\Delta_k\tilde{\Denois}_{t_k}^{\pi^y}(z_k))$ 
    \State Set $z_{k+1}=z_k+n_k$ \Comment{Discretization scheme (\Cref{alg:discretization_1d})}
\EndFor
\State \Return $z_N/t_N$ or $k^\theta_{t_{N-1}+a} (z_{N-1}+y)$ \Comment{See \Cref{rmk:early}}
\end{algorithmic}
\end{algorithm}
\end{figure*}

\begin{figure}
    \centering
    \includegraphics[width=\linewidth]{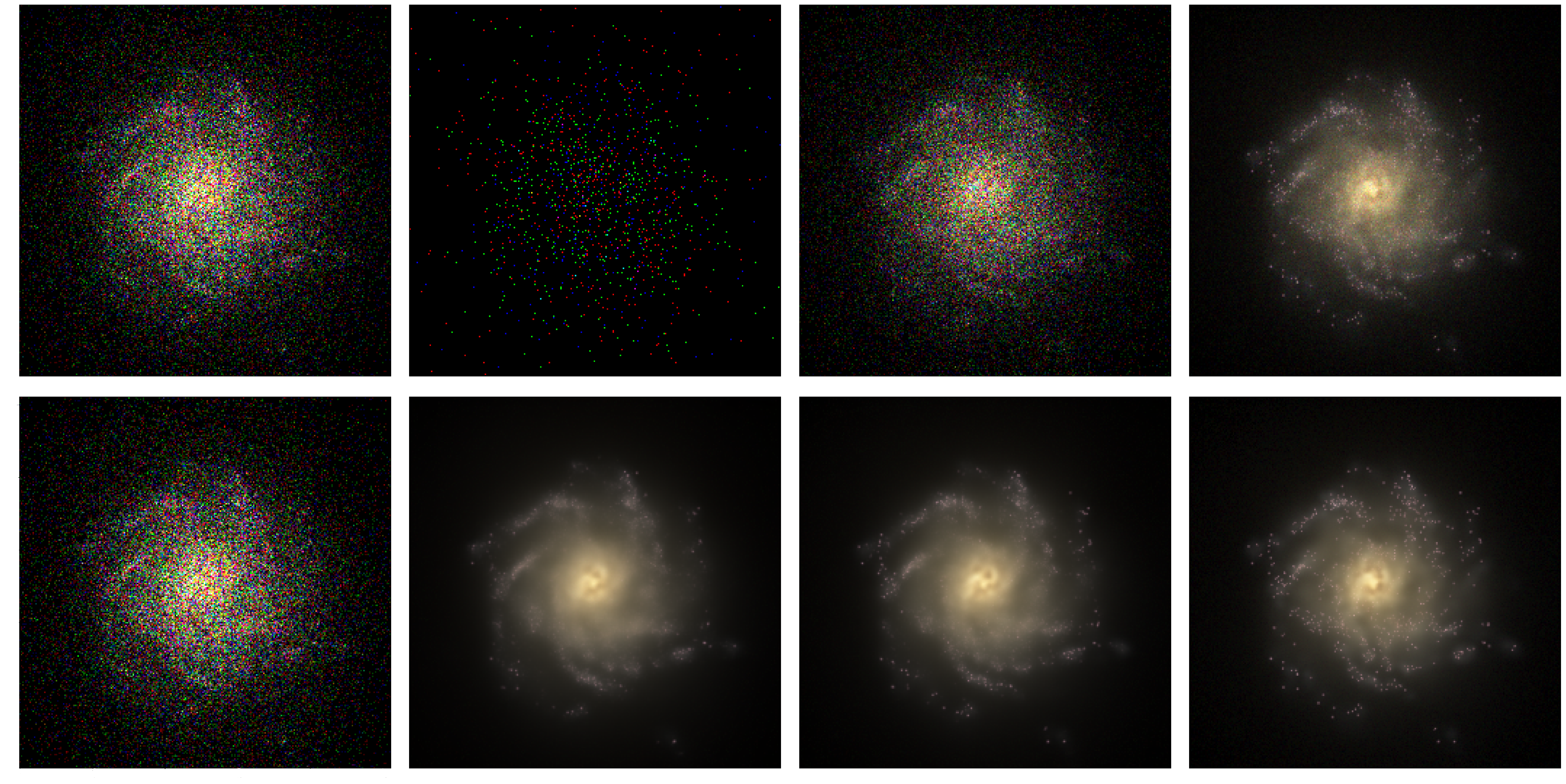}
    \caption{Intermediate states of PDPS for the Poisson denoising problem on the EAGLE dataset. \textbf{(Top-left, bottom-left)} Observation $y\sim\poisson(ax)$, with $x$ an image of the validation dataset. \textbf{(Top row)} $Z^y_t/(1+t)$, with $\bfZ^y\sim \mathrm{MPP}(\pi^y)$, for increasing values of $t$. This MPP is simulated using the Poisson denoiser approximation $\Denois^{\pi^y}(Z^y_t)\approx k^\theta(Z_t^y+y)$. \textbf{(Bottom-row)} The corresponding Poisson denoiser outputs $\Denois^{\pi^y}(Z^y_t)\approx k^\theta(Z_t^y+y)$.}
    \label{fig:eagle_pdps}
\end{figure}

\subsection{MPDM-based DPS}
\label{app:dps_poisson}

We provide additional details regarding the MPDM adaptation of DPS.

\begin{figure*}[t]
\centering
\begin{algorithm}[H]
\caption{MPDM adaptation of DPS}
\label{alg:poisson_dps}
\begin{algorithmic}[1]
\Require Learned denoiser $k^\theta\approx \Denois^\pi$, likelihood $g(y\vert \cdot)$, measurement $y$, time-schedule $0=t_0<t_1<\dots<t_N$ with $\Delta_k=t_{k+1}-t_k$.
\State  Initialize $z_0 = 0$
\For{$k=0,\dots,N-1$}
    \State Compute $k_{t_k}^\theta(z_k)$
    \State Estimate $\Denois^{\pi^y}_{t_k}(z_k)\approx \tilde{\Denois}_{t_k}^{\pi^y}(z_k)$ using $g(y\vert\cdot)$ and \eqref{eq:dps_poisson_approx} (either the continuous approximation \newline
    \hspace*{5em} or the independent coordinates approximation)
    \State Draw $n_k\sim\poisson(\Delta_k\tilde{\Denois}_{t_k}^{\pi^y}(z_k))$
    \State Set $z_{k+1}=z_k+n_k$ \Comment{Discretization scheme (\Cref{alg:discretization_1d})}
\EndFor
\State \Return $z_N/t_N$ or $\tilde{\Denois}_{t_{N-1}}^{\pi^y}(z_{N-1})$ \Comment{See \Cref{rmk:early}}
\end{algorithmic}
\end{algorithm}
\end{figure*}

{\dpspoisson*}
\begin{proof}
Let $\bfZ\sim\MPPpi$.
    As seen in the proof of \Cref{prop:denoiser}, we have for every $z_t\in\N^d,t>0$,
    \begin{equation}
    \label{eq:prop_ratio}
        [\Denois_t^\pi(z_t)]_i = \frac{z_t^i+1}{t}\frac{\PP(Z_t=z_t+e_i)}{\PP(Z_t=z_t)}\eqsp,
    \end{equation}
    where $z_t^i$ denotes the $i$-th coordinate of $z_t$.
    If $\bfZ^y\sim\mathrm{MPP}(\pi^y)$, we have 
    \begin{align}
        \PP(Z_t^y=z_t) &= \int_{\rset_+^d} \prod_{j=1}^d\rme^{-tx_j}\frac{(tx_j)^{z_t^j}}{z_t^j!} \pi^y(\rmd x) \\ &
        \propto \int_{\rset_+^d} \prod_{j=1}^d\rme^{-tx_j}\frac{(tx_j)^{z_t^j}}{z_t^j!} g(y\vert x)\pi(\rmd x) \\
        &\propto \PP(Z_t=z)\int_{\rset_+^d} g(y\vert x)\pi_t(\rmd x\vert z_t)\eqsp,
    \end{align}
    using $\pi^y(\rmd)\propto g(y\vert x)\pi(\rmd x)$ and $p_t(z_t)\pi_t(\rmd x\vert z_t) = p_t(z_t\vert x)\pi(\rmd x)$.
    Combining \eqref{eq:prop_ratio} and $g_t(y\vert z_t)=\int_{\rset_+^d} g(y\vert x)\pi_t(\rmd x\vert z_t)$, this yields 
    \begin{equation}
        [\Denois_t^{\pi^y}(z_t)]_i =  [\Denois_t^{\pi}(z_t)]_i\frac{g_t(y\vert z_t+e_i)}{g_t(y\vert z_t)} = [\Denois_t^{\pi}(z_t)]_i\left(1+ \frac{g_t(y\vert z_t+e_i)-g_t(y\vert z_t)}{g_t(y\vert z_t)}\right) \eqsp.
    \end{equation}
\end{proof}

\begin{figure}
    \centering
    \includegraphics[width=\linewidth]{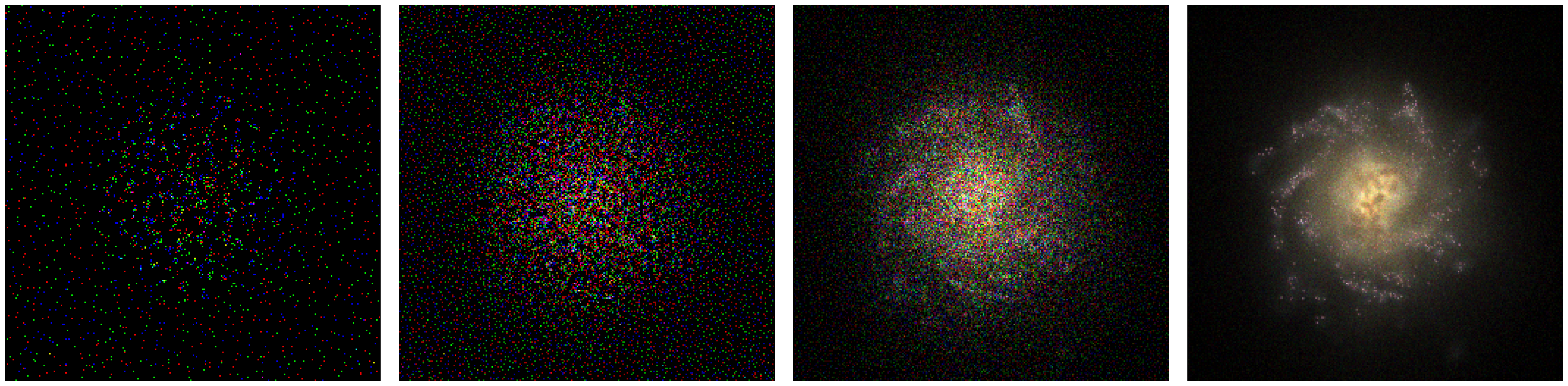}
    \caption{Trajectory generated by MDPM-based DPS, for the Poisson denoising problem. Images show $Z^y_t/(1+t)$, with $\bfZ^y\sim \mathrm{MPP}(\pi^y)$, for increasing values of $t$. This MPP is simulated using the Poisson denoiser approximation of \Cref{th:approx_dps}, with the discrete derivative being approximated by the continuous one.}
    \label{fig:eagle_poisson_dps}
\end{figure}

\paragraph{Estimation of the discrete derivative.}
Computing the discrete derivative resulting from the DPS approximation, that is, equation \eqref{eq:dps_poisson_approx}, would require $d$ denoiser evaluations, which is prohibitive in most applications. Thus, we suggest two approximations. 

The first replaces the discrete derivative with a continuous one, using \Cref{prop:denoiser} to extend the domain of definition of $\Denois^\pi_t$ to $\rset_+^d$.
We motivate this approximation by noting that, by construction of the MPP, for a relatively large $t>0$, $\Denois^\pi_t$ is typically evaluated with an input $z_t\approx tx$, with $x = \mathcal{O}(1)$. Then, $\Denois^\pi_t(z_t+e_i)\approx \Denois^\pi_t(t(x+e_i/t)) \approx \Denois^\pi(tx)+\partial_i \Denois_t^\pi(tx)$, hence $\Denois^\pi_t(z_t+e_i)-\Denois^\pi_t(z_t) \approx \partial_i\Denois^\pi_t(z_t)$.

If for any $1\leqslant i\leqslant d$, $z_t\mapsto [\Denois^\pi_t(z_t)]_i$ only depends on $z_t^i$, we may evaluate the discrete derivative with only two denoiser evaluations: indeed, $[\Denois^\pi_t(z_t +e_j)]_i-[\Denois^\pi_t(z_t)]_i = ([\Denois^\pi_t(z_t+\mathsf{1})]_i-[\Denois^\pi_t(z_t)]_i)\times\delta_{i,j}$ in that case. This condition holds notably if $\pi=\pi_1\otimes\dots\otimes\pi_d$. However, this assumption is generally too restrictive in practice and therefore introduces an approximation error. 

As noted in \Cref{ex:exact_derivative}, both approximations are in fact
exact when $\pi=\gamma\sim\loigamma(\alpha,\beta)$, as $[\Denois^\gamma_t(z_t)]_i = (z_t^i+\alpha_i)/(\beta_i+t)$.
Since the continuous approximation can be justified for large $t$, a reasonable strategy would be to use it for large $t$, and default to the independent coordinates approximation when $t$ is close to $0$. We illustrate this compromise in \Cref{fig:discrete_derivative}, where we compare the quality of both approximations on tractable examples. For the numerical experiments of \Cref{sec:expe}, we observed that the continuous approximation generally led to better empirical results, but that starting with the independent coordinates approximation for small $t$ usually slightly improved results; see \Cref{app:settings}. These observations suggest that the continuous derivative should be used by default, while the independent coordinate approximation can be relevant for small $t$, depending on the application.

\begin{figure}
    \centering
    \includegraphics[width=\linewidth]{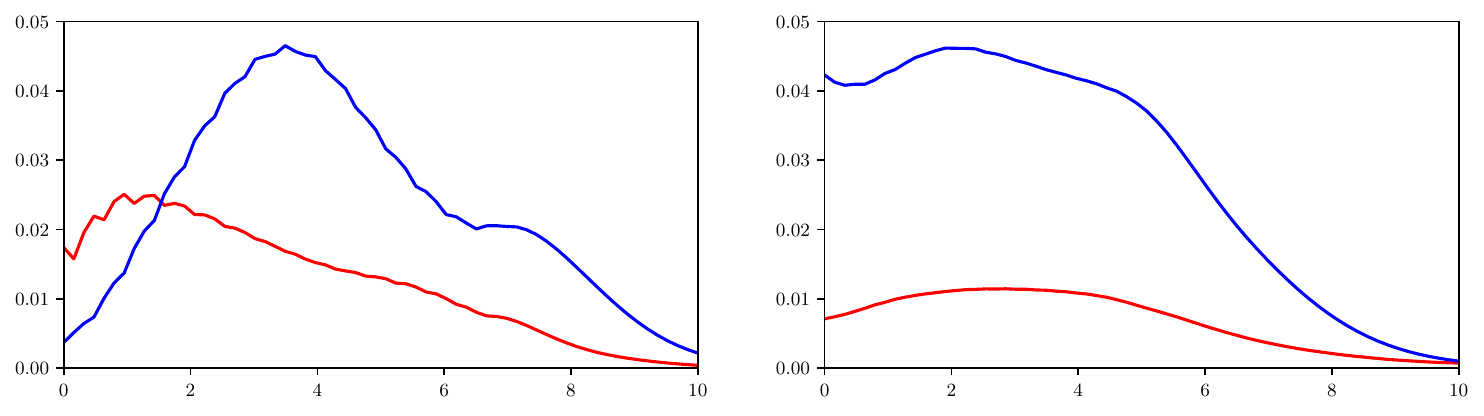}
    \caption{Errors associated with the approximation of the Poisson denoiser $\Denois^\pi_t$  discrete derivative, for $0\leqslant t\leqslant 10$, when $\pi$ is a mixture of bivariate Gamma distributions. \textcolor{red}{Red lines} show the total error associated with the continuous approximation, \ie, $1/4\sum_{i,j=1}^2 \vert [\Denois^\pi_t(z_t+e_j)]_i - [\Denois^\pi_t(z_t)]_i-\partial_j [\Denois^\pi_t(z_t)]_i\vert$. \textcolor{blue}{Blue lines} show the error associated with the independent coordinate approximations, that is, $1/4\sum_{i,j=1}^2 \vert  [\Denois^\pi_t(z_t+e_j)]_i - [\Denois^\pi_t(z_t)]_i - ([\Denois^\pi_t(z_t+\mathbf{1})]_i- [\Denois^\pi_t(z_t)]_i)\times\delta_{ij}\vert$. The errors are averaged over $z_t = (k,l), 1\leqslant k,l\leqslant 15$. 
    \textbf{(Left)} $\pi$ is the distribution used in the toy example; see \Cref{app:settings}. We observe that the independent coordinates approximation is preferable in the low $t$ regime, whereas the continuous approximation gets significantly better as $t$ increases. 
    \textbf{(Right)} we modify $\pi$ to be less ``symmetric",  setting $\alpha_1 = (16,16), \alpha_2 = (64,64), \alpha_3=(32,64),\alpha_4=(32,16), \alpha_5=(20,20)$. In this case, the independent coordinates approximation exhibits a substantially larger error.} 
    \label{fig:discrete_derivative}
\end{figure}

\subsection{MPDM-based PnP-DM}
\label{app:pnpdm}
Consider the following family of distributions. For clarity, we keep the notation $\rho_k$ of the original PnP-DM paper; however, in practice we set $t_k=1/\rho_k^2$, for a time-grid $0\approx t_0<t_1<\dots<t_N$: 
\begin{equation}
    \pi^y_k(\rmd x, \rmd \tilde{x})
    \propto
    g(y\vert \tilde{x})
    \exp\left(-\mathscr{D}_{\mathrm{e}}(\tilde{x},x)/\rho_k^2\right)
    \pi(\rmd x)\rmd \tilde{x}\eqsp.
\end{equation}
The marginal in $x$ is heuristically expected to approach the posterior $\pi^y$ as $t_k\rightarrow \infty$. We do not claim any exact convergence: this expression is motivated by the analogy existing with the Gaussian diffusion case, as \Cref{prop:split_gibbs} shows that the denoising distribution of $\MPPpi$ naturally appears within the prior step. However, a properly normalized kernel of this form generally involves an additional normalization factor, potentially depending on $x$, which is omitted here.
\paragraph{Prior step.}
 \Cref{prop:split_gibbs} shows that $\pi^y_k(\rmd x \vert \tilde{x}) 
\propto\pi_{t_k}\left(\rmd x
\left\vert t_k\tilde{x}\right.\right)$, setting $t_k\coloneqq 1/\rho_k^2$. Recall that $\pi_{t}(\rmd x\vert z_t)$ denotes the conditional distribution of $X$ given $Z_t=z_t$, with $\bfZ$ the MPP with rate $X$. A priori, it is defined for $z_t\in\N^d$; however, equation \eqref{eq:denoising_dist} allows us to extend its domain of definition to $\rset_+^d$.
A sample from $\pi^y_k(\rmd x\vert \tilde{x})$ can be obtained by simulating this MPP, using a slightly modified version of \Cref{alg:discretization_1d}. 
Instead of starting with $t_0=0, \tilde{Z}_0=0$, consider a time-grid $t_k = t_{k,0} <\dots< t_{k,N}$, and an initial value $\tilde{Z}_{t_k} = t_k\tilde{x}$. Again, $t_k\tilde{x}$ does not necessarily belong to $\N^d$: one solution is to round each coordinate to the nearest integer. However, note that the discretization scheme of \Cref{alg:discretization_1d} does not explicitly require the initial state to be integer-valued. Indeed, as noted already, \Cref{prop:denoiser} shows that the Poisson denoiser $\Denois_t^\pi$, for a fixed $t$, can be extended to $\rset_+^d$-valued inputs. Furthermore, as explained in \Cref{app:training_details}, the learned denoiser $k^\theta_t$ is typically trained using rescaled inputs of the form $z_t/(1+t)$ to compensate for the $Z_t=\mathcal{O}(t)$ scaling of the MPP. Thus, the learned denoiser is naturally compatible with non-integer inputs: in our numerical experiments, we observed no substantial difference when keeping $t_k\tilde{x}$ instead of rounding it.
{\pnpdm*}
\begin{proof}
    As by definition $\breg_\phi(x,y) = \sum_{i=1}^dx_i\log(x_i/y_i)-x_i+y_i$, we have 
    \begin{align}
        \exp\left(-t\breg_\phi(tz_t,x)\right) &= \exp\left(-t\sum_{i=1}^d \frac{z_t^i}{t}\log\left(\frac{z_t^i/t}{x_i}\right)-\frac{z_t^i}{t}+x_i\right) \\
        &= \exp\left(\sum_{i=1}^d z_t^i\log(x_i)-tx_i\right)\times C(z_t)\\
        &\propto \prod_{i=1}^d\rme^{-tx_i}x_i^{z_t^i}\eqsp,
    \end{align}
    where $C(z_t)$ does not depend on $x$. We recognize the expression of $\pi_t(\rmd x\vert z_t)$ \eqref{eq:denoising_dist}.
\end{proof}

\paragraph{Likelihood step.}
The likelihood step, \ie, sampling from $\pi^y_k(\rmd \tilde{x}\vert x)$, requires an MCMC procedure. To be precise, from \eqref{eq:pnpdm_poisson}, \eqref{eq:bregman}, we obtain
\begin{equation}
    \pi^y_k(\rmd \tilde{x}\vert x) \propto g(y\vert \tilde{x})\exp\left(-t_k\sum_{i=1}^d \tilde{x}_i\log(\tilde{x}_i/x_i)-\tilde{x}_i\right)\rmd \tilde{x} \eqqcolon q_y^k(\tilde{x}\vert x)\rmd \tilde{x}\eqsp,
\end{equation}
and the associated score (gradient log with respect to $\tilde{x}$) is
\begin{equation}
    \nabla \log q_y^k(\tilde{x}\vert x) = \nabla\log g(y\vert \tilde{x}) -t_k\sum_{i=1}^d(\log(\tilde{x}_i)-\log(x_i))e_i\eqsp,
\end{equation}
with $e_i = (\delta_{i,j})_{1\leqslant j\leqslant d}$.
However, this score is not well behaved when $\tilde{x}$ has coordinates close to zero. Furthermore, we need to enforce the constraints $\tilde{x}\in\rset_+^d$.
We address these points by relying on a  Riemannian Langevin MCMC method; see \citep[Section 5]{girolami_riemann_2011}. More precisely, we consider the manifold $(0,+\infty)^d$ equipped with the Riemannian $\metric$ defined as $\metric_x = \mathrm{diag}(1/x_1,\dots,1/x_d)$. This metric is actually the \emph{Fisher information matrix} associated with the Poisson distribution; see \citep[Section 4]{girolami_riemann_2011}. With this metric, the Langevin SDE has the form
\begin{equation}
    \label{eq:langevin_riemannian}
    \rmd \tilde{X}_t = \left(\tilde{X}_t\odot\nabla \log q_y^k(\tilde{X}_t\vert x)+ \mathsf{1}\right)\rmd t + \sqrt{2\tilde{X}_t}\odot\rmd B_t\eqsp,
\end{equation}
with $\odot$ the coordinate-wise product and $\mathsf{1}=(1,\dots,1)\in\rset^d$. The Euler discretization of this SDE yields \Cref{alg:rla}. In particular, the $j$-th Langevin MCMC update has the form
\begin{equation}
    \tilde{x}^{j+1} = \tilde{x}^j + \delta \times \tilde{x}^j\odot \nabla\log g(y\vert \tilde{x}^j) -t_k \delta\times \sum_{i=1}^d\tilde{x}_i^j(\log(\tilde{x}_i^j) -\log(x_i))e_i + \delta\times \mathsf{1} + \sqrt{2\delta\tilde{x}^j}\odot G_j\eqsp,
\end{equation}
with $\delta>0$ denoting the step-size and $G_j\sim\gauss(0,\Idd)$. Note that $\tilde{x}_i^j\log(\tilde{x}_i^j)$ is well behaved near zero.

While the Riemannian Langevin SDE is intrinsically defined on $(0,+\infty)^d$, its discretization may lead to negative outputs, especially if a large step-size $\delta$ is used; in practice, we clamp each update to guarantee positivity. Finally, \Cref{alg:pnp_dm} includes an optional hyperparameter $\eta$, which is not present in the original PnP-DM algorithm. This addition is justified by a similar hyperparameter within the original DAPS algorithm; see below.

\begin{figure}
    \centering
    \includegraphics[width=\linewidth]{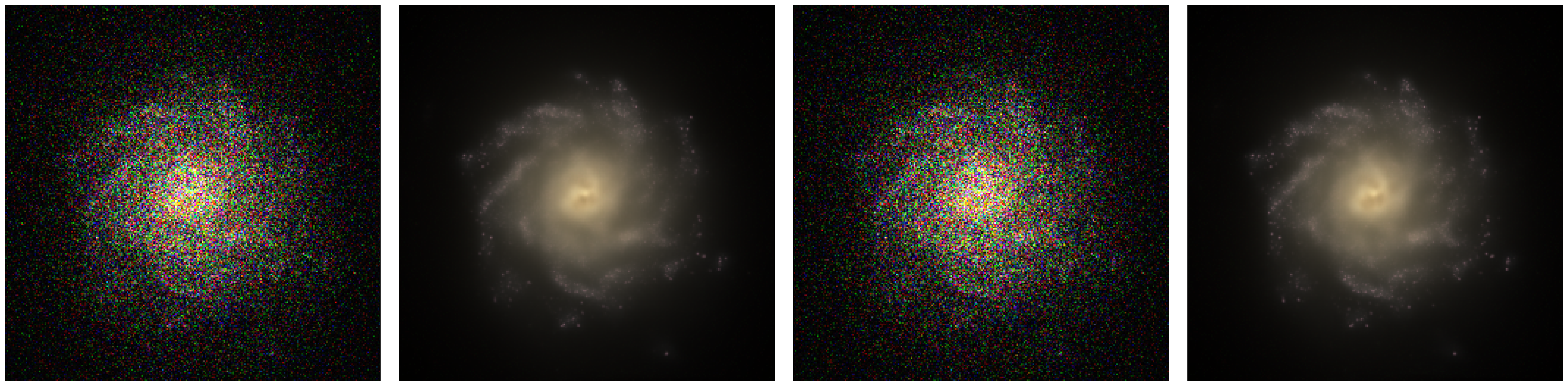}
    \caption{Same setting as \Cref{fig:eagle_pdps,fig:eagle_poisson_dps}. Successive likelihood and prior steps (images 1,3 and 2,4, respectively) of MPDM-based PnP-DM.}
    \label{fig:eagle_poisson_ppnp}
\end{figure}

\begin{figure*}[t]
\centering
\begin{algorithm}[H]
\caption{MPDM adaptation of PnP-DM}
\label{alg:pnp_dm}
\begin{algorithmic}[1]
\Require Learned denoiser $k^\theta\approx \Denois^\pi$, likelihood $g(y\vert \cdot)$, measurement $y$, schedule $t_0<t_1<\dots<t_N$; optional: $\eta>0$, default $\eta=1$.
\State  Initialize $x$
\For{$k=0,\dots,N-1$}
    \State Sample $\tilde{x}$ from $
        \pi^y_k(\rmd \tilde{x} \vert x) \propto g(y\vert \tilde{x})^\eta\exp\left(-t_k\breg_{\rme}(\tilde{x},x)\right)\rmd \tilde{x}$ with \Cref{alg:rla} \newline \hspace*{2em} where $\breg_{\rme}(\tilde{x},x) = \sum_{i=1}^d\tilde{x}_i\log(\tilde{x}_i/x_i)-\tilde{x}_i+x_i$ \Comment{Likelihood step}

    \State Sample $x$ from $
        \pi^y_k(\rmd x \vert \tilde{x})$ by simulating the counting process with intensity $k^\theta$ \newline \hspace*{2em} with initial state $t_k\tilde{x}$ at time $t_k$; see \Cref{alg:discretization_1d} \Comment{Prior step}
\EndFor
\State \Return $x$
\end{algorithmic}
\end{algorithm}
\end{figure*}

\begin{figure*}
\begin{algorithm}[H]
\caption{Riemannian Langevin MCMC algorithm on $\rset_+^d$}
\label{alg:rla}
\begin{algorithmic}[1]
\Require Score function $s:\rset_+^d \rightarrow\rset^d$, time-step $\delta>0$, number of iterations $N$, initial state $x_0\in\rset_+^d$.
\For{$k=0,\dots,N-1$}
    \State Sample $G_k\sim\gauss(0,\Idd)$
    \State Set $x_{k+1} = x_k + \delta x_k \odot s(x_k)+\delta\mathbf{1} + \sqrt{2\delta x_k}\odot G_k$, with $\odot$  the component-wise product
    \State Set $x_{k+1} = (x_{k+1})_+$ \Comment{Enforce positivity}
\EndFor
\State \Return $x_N$
\end{algorithmic}
\end{algorithm}
\end{figure*}

\begin{figure*}
\begin{algorithm}[H]
\caption{MPDM adaptation of DAPS}
\label{alg:poisson_daps}
\begin{algorithmic}[1]
\Require Learned denoiser $k^\theta\approx \Denois^\pi$, likelihood $g(y\vert \cdot)$, measurement $y$, schedule $t_0<t_1<\dots<t_N$; optional: $\eta>0$, default $\eta=1$.
\State  Initialize $x$
\For{$k=0,\dots,N-1$}
\State Sample $z_k\sim\poisson(t_k x)$ \Comment{Noising step}
    \State Compute $\hat{x} = k^\theta_{t_k}(z_k)$ (or simulate the counting process with intensity $k^\theta$ \newline \hspace*{2em} with initial state $z_k$ at time $t_k$ with \Cref{alg:discretization_1d}) \Comment{Prior step}
    \State Sample $x$ from the density $g(y\vert x)^\eta\loigamma(x;/r_{t_k}^{-2}\hat{x}, r_{t_k}^{-2}\mathsf{1})\rmd x$ \newline \hspace*{2em} with \Cref{alg:rla} \Comment{Likelihood step}
\EndFor
\State \Return $x$
\end{algorithmic}
\end{algorithm}
\end{figure*}

\subsection{MPDM-based DAPS}
\label{app:daps}
Our MPDM version of DAPS \citep{zhang_improving_2025} is closely related to the original Gaussian diffusion version, described in \Cref{sec:inverse_gaussian}.
We introduce, given 
a time-grid $t_0<t_1<\dots t_N$,
\begin{equation}
\label{eq:daps_distribution_poisson}
    \pi^y_{t_k}(\rmd x\vert z_{t_k}) \propto g(y\vert x)\pi_{t_k}(\rmd x\vert z_{t_k})\eqsp.
\end{equation}
This is mostly the same equation as \eqref{eq:daps_distribution}, but here $\pi_t(\rmd x\vert z_t)$ is the conditional distribution associated with the MPP, defined in \eqref{eq:denoising_dist}.

\begin{figure}
    \centering
    \includegraphics[width=\linewidth]{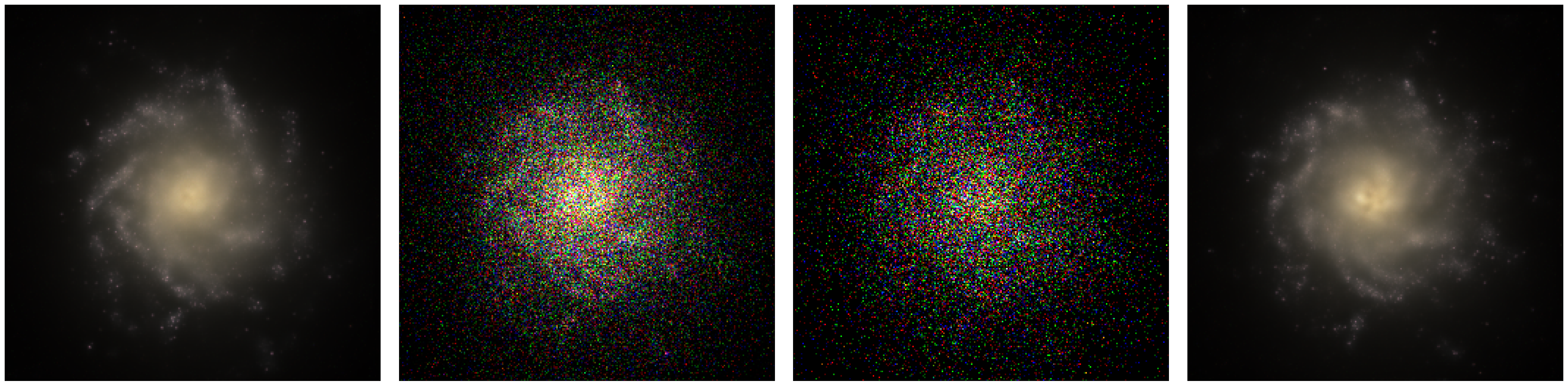}
    \caption{Same setting as \Cref{fig:eagle_pdps,fig:eagle_poisson_dps,fig:eagle_poisson_ppnp}. Successive prior steps (images 1,4), likelihood step (image 2), and noising step (image 3) for MPDM-based DAPS. For the noising step, we rescale by a factor of $1/(1+t_k)$.}
    \label{fig:placeholder}
\end{figure}

\paragraph{Prior step.}
Instead of a Gaussian approximation, we use a Gamma one: 
\begin{equation}
    \pi_{t}(\rmd x\vert z_t) \approx \loigamma(r_t^{-2}\hat{x}(z_{t}),r_{t}^{-2}\mathsf{1})\eqsp,
\end{equation}
where $\hat{x}(z_t)$ is either obtained by simulation of the MPP from $z_t$ onward, or directly as $\Denois^\pi_t(z_t)$. This latter choice is actually exact when $\pi$ is Gamma distributed.
\begin{example}
Note that when $\pi=\gamma\sim\loigamma(\alpha,\beta\mathsf{1}),\alpha \in(0,+\infty)^d, \beta>0$, from \Cref{app:gamma_case}, \eqref{eq:denoising_dist_gamma}, 
\begin{equation}
    \pi_t(\rmd x\vert z_t) \sim \loigamma(z+\alpha, (t+\beta)\mathsf{1})\eqsp.
\end{equation}
Since $\Denois^\gamma_t(z_t)=(z_t+\alpha)/(t+\beta)$, we have $ \pi_t(\rmd x\vert z_t)\sim\loigamma((t+\beta)\Denois^\pi_t(z_t), (t+\beta)\mathsf{1})$, matching our ``DAPS-like" approximation with $r_t^2 = 1/(t+\beta)$.
\end{example}
\paragraph{Likelihood step.}
Similar to our PnP-DM adaptation, we use \Cref{alg:rla} to sample from
\begin{equation}
    \label{eq:daps_likelihood_poisson}
    \pi^y_{t_k}(\rmd x\vert z_{t_k}) \propto g(y\vert x)\loigamma(r_t^{-2}\hat{x}(z_{t}),r_{t}^{-2}\mathsf{1})\eqsp.
\end{equation}
The Riemannian metric discussed in the previous section combines well with the Gamma density: the $j$-th Riemannian Langevin MCMC update is,
after simplification,
\begin{align}
    x^{j+1} &= x^j+\delta \times x^j\odot \nabla \log g(y\vert x^j) + \delta\sum_{i=1}^d x_i^j\times \left((r_{t_k}^{-2}\hat{x}_i^j -1)/x_i^j  -r_{t_k}^{-2}\right)e_i + \delta\mathsf{1}+\sqrt{2\delta x^j}\odot G_j \\
    &=x^j+\delta \times x^j\odot \nabla \log g(y\vert x^j) + \delta\sum_{i=1}^d r_{t_k}^{-2}\left(\hat{x}_i^j - x_i^j\right)e_i + \sqrt{2\delta x^j}\odot G_j\eqsp.
\end{align}
The otherwise problematic $1/x_i^j$ factors, which follow from the score of the Gamma distribution, are cancelled out. 

\paragraph{Noising step.}
In analogy with the Gaussian diffusion case, we then sample $z_{t_{k+1}}\sim \poisson(t_{k+1}x)$. Unlike our adaptation of PnP-DM, this extra step ensures that the successive $z_{t_k}$ lie in $\N^d$.

To improve empirical performance, the original DAPS includes a hyperparameter to control the weight of the likelihood $g$ in the MCMC sampling phases, see \citep[Appendix F.2]{zhang_improving_2025}. We adopt the same strategy: in the Langevin update of \Cref{alg:poisson_daps}, we replace $\nabla \log g(y\vert \cdot)$ with $\eta \nabla \log g(y\vert \cdot)$, $\eta>0$ (equivalently, replace $g(y\vert \cdot)$ with $g(y\vert\cdot)^\eta$). We apply the same modification to the Langevin update of PnP-DM, as we also observe better empirical performance in our inverse problem experiments when $\eta >1$.  While it tends to improve the quality of the reconstructed images, note that setting $\eta> 1$ alters the posterior distribution that we aim to sample from. In particular, it can lead to more concentrated samples, hence decreasing the posterior diversity. The values of $\eta$ used in our experiments are empirically tuned and reported in \Cref{tab:eta}.

\section{Numerical experiments: details}
\label{app:expe}

\subsection{Datasets}
\label{app:datasets}

\paragraph{PET.} The public \emph{BrainWeb} dataset \citep{aubert-broche_twenty_2006,schramm_simulated_2021} contains 21 PET-like brain volumes; we use 20 for training and the remaining one as validation for the inverse problems. From each volume, we extract 60 slices (the middle-most ones), resulting in 1200 $256\times256$ training images (and 60 for validation) with four different pixel values (lowest is 0, highest is 1). During training, we add random lesions to slices \footnote{ \url{https://pet-mr.github.io/brainweb/}}. We augment the dataset during training by applying random translation ($\pm 10$ pixels), horizontal flip, and small rotation ($\pm$ 2°). 

\paragraph{Astrophysics.} We downloaded the public mock galaxy images from the \emph{EAGLE} simulation \citep{the_eagle_team_eagle_2017} after requesting access to the database\footnote{\url{https://astro.dur.ac.uk/ICC/EAGLE/database.php}}. We extracted mock images from the \texttt{RefL0100N1504\_Subhalo} with \texttt{GalaxyID} $<2.10^7$, giving 41372 RGB $256\times256$ images. To enforce a relative homogeneity of the training dataset, we used PCA with 100 components followed by a $k$-means clustering with 3 clusters, and selected one cluster. We kept $100$ images for validation and $7787$ for training. We augment the dataset during training by applying random translation ($\pm 10$ pixels), horizontal and vertical flips, and a random rotation. As for the PET dataset, pixel values span $[0,1]$.

\subsection{Training details}
\label{app:training_details}
While most of the literature relies on pre-trained models to evaluate Gaussian diffusion-based posterior sampling algorithms, our approach requires training models from scratch. The input $z$ of the Poisson denoiser, $(z,t)\mapsto \Denois_t^\pi(z)$, scales with $t$ due to the nature of MPP. Thus, we use rescaled inputs for the network approximating $\Denois^\pi$, denoted by $k^\theta$, namely $z/(1+t)$ and $1/(1+t)$ instead of $z,t$. These design choices are similar to those of \citet{bhattacharya_itdpdm_2026}. The network directly predicts the denoised image, rather than a residual correction to the input. We used the same strategy to train Gaussian denoisers approximating $D^\pi$. The Gaussian diffusion formulation we considered, detailed in \Cref{sec:diffusion_gaussian}, is admittedly not the most common; however, it shares several similarities with MPDM: SNR increasing with $t$, process $\bfX$ scaling as $X_t \approx t$, and such that $X_t/t\rightarrow X\sim\pi$ almost surely. 

\paragraph{Architectures.}
For the toy example, we used a time-conditioned residual multilayer perceptron with a hidden dimension of 128 and three residual blocks. The two-dimensional input was first projected to the hidden dimension, and the diffusion time variable was encoded using a 64-dimensional time embedding followed by a two-layer MLP. Each residual block consisted of two linear layers with SiLU activations and was conditioned on the time embedding through feature-wise linear modulation (FiLM), implemented as a learned scale and shift. A final linear layer mapped the hidden representation back to the two-dimensional output space, followed by a Softplus activation to enforce non-negativity.

For the BrainWeb and EAGLE datasets, we used a U-Net architecture \citep{RFB15a} with 32 input channels, base width 256, and 5 resolution levels. The number of channels doubles at every level, with spatial downsampling by a factor of two. Each encoder and decoder level consists of two residual blocks. The residual blocks use Group Normalization and SiLU activations and are conditioned on the diffusion time through the same feature-wise linear modulation (FiLM) mechanism. The time variable is first mapped to a 128-dimensional sinusoidal embedding, followed by a two-layer MLP. The bottleneck consists of two residual blocks and contains no attention. The decoder uses nearest-neighbor upsampling followed by convolution and concatenates the corresponding encoder features through standard U-Net skip connections. A final $1\times1$ convolution produces the output channels, followed by a Softplus activation to enforce non-negativity.

\paragraph{Training.} 
We trained our models with the Adam optimizer and a $10^{-4}$ learning rate. During training, each image is noised at a level $t$ with a log-uniform distribution: $\log t$ is uniform on $[-7,7]$. This non-uniform distribution for $t$ is compensated with a Jacobian factor; see \Cref{app:training}. We use the usual $l^2$ loss instead of the Bregman divergence \eqref{eq:bregman}: after comparing both choices, we observed no significant difference in our experiments.

For the image datasets, we used a batch size of 32 and trained for 3500 epochs for the PET dataset and 1000 epochs for the astrophysics dataset. For the toy example, we used approximately $10^5$ optimization steps, with batches of 1024 fresh samples from the Gamma mixture distribution defined in \Cref{app:settings}.

\subsection{Inverse problems: settings and detailed results}
\label{app:settings}

\paragraph{Computational budget and hyperparameters.} 
PDPS, DPS, PnP-DM and DAPS all require a time grid $t_0<t_1<\dots< t_N$ to approximately simulate the underlying Gaussian diffusion/MPDM. PDPS and DPS use one NFE (denoiser call) per iteration, hence $N$ NFEs in total; to be precise, DPS requires vector-Jacobian computations. For DAPS and PnP-DM, there is an additional inner loop at each prior step: if we fix the total number of NFEs, we must allocate them between the outer loop on $t_0<t_1<\dots t_N$ and the inner loop that simulates the diffusion from $t$ onward. For fairness and simplicity, we use only one NFE for each prior step, and keep the same grid $t_0<t_1<\dots< t_N$ for every algorithm. This design choice is present in DAPS: the authors observe that more NFEs in the inner loop lead to slightly better reconstruction metrics for linear inverse problems; see \citep[Appendix I]{zhang_improving_2025}. On the other hand, using only one inner-loop NFE is admittedly a slight modification of the original PnP-DM design.  

In the likelihood step of DAPS and PnP-DM, we set the number of Langevin steps to 128 and the step-size to $\delta=0.001$. While computationally cheap, this likelihood step is repeated at each iteration, thus $N$ times. For simplicity, we choose to only match the number of NFEs between these methods and DPS/PDPS, ignoring these Langevin steps. Note that for Gaussian diffusion-based PnP-DM and DAPS, we use a standard Langevin MCMC scheme and clamp updates that have negative values, as the authors do not discuss the case where the density is supported on $\rset_+^d$. Meanwhile, our MPDM-based adaptations rely on Riemannian Langevin MCMC (\Cref{alg:rla}). 

The hyperparameter $\eta$ mimics the empirical tuning of $\beta_y$ present in DAPS \citep[Appendix F.2]{zhang_improving_2025}: it effectively increases the weight of the likelihood term in the Langevin scheme. As in the original DAPS paper, we observe better reconstruction metrics when $\eta>1$, so we keep this design choice in our implementations.
This additional parameter is not present in PnP-DM: we include it to ensure fairness, as we also observe better empirical performance with $\eta>1$. In each experiment, we test several values of $\eta$ for both DAPS and PnP-DM, ranging from $1$ to $20$, and select the one that yields the best reconstruction metrics; see \Cref{tab:eta}. For DAPS algorithms (either MPDM or diffusion-based), we used $r_t^2 = 1/(1+t)$. We mostly rely on continuous derivative approximation for MPDM-based DPS (see \Cref{app:dps_poisson}), except for the EAGLE dataset (both Poisson denoising and blur problems) and the toy dataset, where we use the independent coordinates approximation until $t=0.05$, before switching to the continuous approximation. We observed that it slightly improved the reconstruction performance in these cases.

Overall, comparison of these different methods is not straightforward, especially because DAPS and PnP-DM prior steps and likelihood steps require specific computational budget and several hyperparameters: we choose to tune $\eta$, but to keep the Langevin parameters identical for every experiment. Note that, while limited to the $\mca=a\Idd$ case, PDPS requires few hyperparameters (only the time-grid design), only one NFE per iteration, and no vector-Jacobian computation.

\paragraph{Toy example.} We consider the bidimensional Gamma mixture distribution $\pi=\frac{1}{5}\sum_{i=1}^5\loigamma(\alpha_i,\beta_i)$ with $\alpha_1 = (16,16), \alpha_2=(64,64), \alpha_3=(16,64), \alpha_4=(64,16), \alpha_5=(20,20),
\beta_1=\beta_2 =\beta_3=\beta_4 = (16,16), \beta_5 = (8,8)$. The MPP is completely tractable in this case; see \Cref
{app:gamma_case}.

We sample 10 points $x_1,\dots x_{10}$ from $\pi$ and draw $y_i\sim \poisson(ax_i)$ in both the low-count regime ($a=0.6$) and high-count regime ($a=3$). All algorithms use the same time grid: $s_1,\dots s_N$ are log-uniform on $[0.1,100]$ and we set $t_i = 0.0001 +s_i -0.1$; we choose $N=1024$. For PnP-DM and DAPS methods, we use only one denoising step at each time $t_i$ to ensure that all methods have identical NFE; furthermore, we keep $\eta=1$ in order to preserve the correct target distribution. For $i=1,\dots,10$, every method generates 10000 approximate $\pi^{y_i}$ samples: we compute the 2-Wasserstein distance against a set of 10000 exact $\pi^{y_i}$ samples, and average the results over $i$. We repeat this procedure 10 times to obtain confidence intervals. We also compute the 2-Wasserstein distance between two distinct sets of exact samples to obtain a baseline. The full results are given in \Cref{tab:toy_inv_large}.
\begin{table}[ht]
\scriptsize
\centering
\setlength{\tabcolsep}{5pt}
\begin{tabular}{c l | c | c}
\toprule
\multicolumn{2}{c}{\multirow{2}{*}{\textbf{Methods}}} & \multicolumn{1}{c}{\textbf{Toy (LC)}} & \multicolumn{1}{c}{\textbf{Toy (HC)}} \\
\cmidrule(lr){3-3}
\cmidrule(lr){4-4}
& & 2-Wasserstein {\scriptsize$\downarrow$} & 2-Wasserstein {\scriptsize$\downarrow$} \\
\midrule
\midrule
\multirow{4}{*}{{\shortstack{\textbf{MPDM-based}\\Poisson}}} & DAPS & $0.3320\,{\scriptstyle\pm\,0.0055}$ & $0.3173\,{\scriptstyle\pm\,0.0042}$ \\
 & DPS & $0.2332\,{\scriptstyle\pm\,0.0091}$ & $0.3106\,{\scriptstyle\pm\,0.0039}$ \\
 & PnP-DM & $1.0761\,{\scriptstyle\pm\,0.0053}$ & $0.3749\,{\scriptstyle\pm\,0.0037}$ \\
 & PDPS & \colorbox{red!25}{$0.0945\,{\scriptstyle\pm\,0.0053}$} & \colorbox{red!25}{$0.0732\,{\scriptstyle\pm\,0.0023}$} \\
\midrule
\multirow{3}{*}{{\shortstack{Diffusion-based\\Poisson}}} & DAPS & $0.2898\,{\scriptstyle\pm\,0.0043}$ & \colorbox{red!10}{$0.2576\,{\scriptstyle\pm\,0.0044}$} \\
 & DPS & \colorbox{red!10}{$0.2134\,{\scriptstyle\pm\,0.0046}$} & $0.2999\,{\scriptstyle\pm\,0.0038}$ \\
 & PnP-DM & $1.0783\,{\scriptstyle\pm\,0.0030}$ & $0.3553\,{\scriptstyle\pm\,0.0030}$ \\
\midrule
\multirow{3}{*}{\shortstack{Diffusion-based\\Gaussian-approximated}} & DAPS & $0.4292\,{\scriptstyle\pm\,0.0081}$ & $0.3241\,{\scriptstyle\pm\,0.0039}$ \\
 & DPS & $0.6285\,{\scriptstyle\pm\,0.0052}$ & $0.2849\,{\scriptstyle\pm\,0.0040}$ \\
 & PnP-DM & $1.0792\,{\scriptstyle\pm\,0.0072}$ & $0.3551\,{\scriptstyle\pm\,0.0041}$ \\
 \midrule 
 \midrule
 \multicolumn{2}{c|}{PDPS (exact denoiser)} & $0.0837\,{\scriptstyle\pm\,0.0047}$ & $0.0625\,{\scriptstyle\pm\,0.0016}$ \\
 \multicolumn{2}{c|}{Ground truth} & $0.0721\,{\scriptstyle\pm\,0.0052}$ & $0.0469\,{\scriptstyle\pm\,0.0022}$ \\
\bottomrule
\end{tabular}
\caption{Average posterior sampling performance for the toy example. PDPS (exact denoiser) uses the analytical expression of the Poisson denoiser (see \Cref{app:gamma_case}) instead of a learned model. As a baseline, ``Ground truth" indicates the 2-Wasserstein distance between two sets of exact samples.}
\label{tab:toy_inv_large}
\end{table}
We also generate approximate $\pi$ samples with both the MPDM and the diffusion model, using the same time-grid (with either $N=64$ or $N=1024$ time steps). \Cref{tab:wass_gen} shows the 2-Wassertein distance against exact $\pi$ samples. Note that both models have comparable generative performance, suggesting that the posterior sampling discrepancy observed in \Cref{tab:toy_inv_large} is due to posterior sampling algorithms rather than the quality of the learned prior.
\begin{table}[H]
\scriptsize
    \centering
    
    \begin{tabular}{c|c|c}
    \toprule
        \textbf{Generative model}  & $N=64$ & $N=1024$ \\
        \midrule\midrule
        \textbf{MPDM} &$0.1108\pm\text{\scriptsize{0.0205}}$  &$0.1011\pm\text{\scriptsize{0.0128}}$  \\
        \textbf{Diffusion} &$0.1153\pm\text{\scriptsize{0.0125}}$ &$0.1078\pm\text{\scriptsize{0.0167}}$\\
        \midrule
        \midrule
        MPDM (exact denoiser) &$0.1176\pm\text{\scriptsize{0.0137}}$
        &$0.1013\pm\text{\scriptsize{0.0227}}$\\
        Ground truth &\multicolumn{2}{c}{$0.0861\pm\text{\scriptsize{0.0164}}$}  \\
        \bottomrule
    \end{tabular}
    \caption{Evaluation of generative models for the toy example. 2-Wasserstein distance between exact and approximate samples.}
    \label{tab:wass_gen}
\end{table}

For each imaging inverse problem, we draw an observation $y$ for each image $x$ in the validation dataset, compute the PSNR, SSIM, and LPIPS of the reconstructed image obtained with each posterior sampling algorithm, and average the metrics over the validation dataset. We repeat this procedure 4 times to obtain confidence intervals.  

\begin{table}[ht]
\centering
\scriptsize
\setlength{\tabcolsep}{2pt}
\begin{tabular}{c l c c c c c c}
\toprule
\multicolumn{2}{c}{\multirow{2}{*}{\textbf{Methods}}} & \multicolumn{3}{c}{\textbf{BrainWeb denoising (LC)}} & \multicolumn{3}{c}{\textbf{BrainWeb denoising (HC)}} \\
\cmidrule(lr){3-5}
\cmidrule(lr){6-8}
 & & PSNR {\scriptsize$\uparrow$} & SSIM {\scriptsize$\uparrow$} & LPIPS {\scriptsize$\downarrow$} & PSNR {\scriptsize$\uparrow$} & SSIM {\scriptsize$\uparrow$} & LPIPS {\scriptsize$\downarrow$} \\
\midrule
\midrule
\multirow{4}{*}{\shortstack{\textbf{MPDM-based} \\Poisson}} & DAPS & \colorbox{red!10}{$16.72\,{\scriptstyle\pm\,0.07}$} & $0.7570\,{\scriptstyle\pm\,0.0025}$ & $0.1654\,{\scriptstyle\pm\,0.0067}$ & $20.62\,{\scriptstyle\pm\,0.08}$ & $0.8559\,{\scriptstyle\pm\,0.0019}$ & $0.1159\,{\scriptstyle\pm\,0.0049}$ \\
 & DPS & \colorbox{red!10}{$16.91\,{\scriptstyle\pm\,0.15}$} & $0.7574\,{\scriptstyle\pm\,0.0092}$ & $0.1429\,{\scriptstyle\pm\,0.0110}$ & $20.49\,{\scriptstyle\pm\,0.13}$ & $0.8477\,{\scriptstyle\pm\,0.0086}$ & $0.1207\,{\scriptstyle\pm\,0.0110}$ \\
 & PnP & $16.65\,{\scriptstyle\pm\,0.10}$ & $0.6278\,{\scriptstyle\pm\,0.0051}$ & $0.2258\,{\scriptstyle\pm\,0.0070}$ & \colorbox{red!25}{$21.72\,{\scriptstyle\pm\,0.07}$} & $0.8171\,{\scriptstyle\pm\,0.0038}$ & $0.1408\,{\scriptstyle\pm\,0.0071}$ \\
 & PDPS & \colorbox{red!10}{$16.90\,{\scriptstyle\pm\,0.09}$} & \colorbox{red!10}{$0.7812\,{\scriptstyle\pm\,0.0033}$} & \colorbox{red!10}{$0.1124\,{\scriptstyle\pm\,0.0063}$} & \colorbox{red!10}{$21.31\,{\scriptstyle\pm\,0.10}$} & \colorbox{red!25}{$0.9019\,{\scriptstyle\pm\,0.0020}$} & \colorbox{red!25}{$0.0585\,{\scriptstyle\pm\,0.0038}$} \\
\midrule
\multirow{3}{*}{\shortstack{Diffusion-based \\ Poisson}} & DAPS & \colorbox{red!25}{$17.12\,{\scriptstyle\pm\,0.07}$} & \colorbox{red!10}{$0.7794\,{\scriptstyle\pm\,0.0028}$} & $0.1141\,{\scriptstyle\pm\,0.0051}$ & \colorbox{red!10}{$21.41\,{\scriptstyle\pm\,0.08}$} & $0.8796\,{\scriptstyle\pm\,0.0019}$ & \colorbox{red!10}{$0.0668\,{\scriptstyle\pm\,0.0033}$} \\
 & DPS & \colorbox{red!25}{$17.22\,{\scriptstyle\pm\,0.11}$} & \colorbox{red!25}{$0.7915\,{\scriptstyle\pm\,0.0032}$} & \colorbox{red!25}{$0.0890\,{\scriptstyle\pm\,0.0043}$} & $20.62\,{\scriptstyle\pm\,0.14}$ & \colorbox{red!10}{$0.8814\,{\scriptstyle\pm\,0.0028}$} & \colorbox{red!25}{$0.0591\,{\scriptstyle\pm\,0.0052}$} \\
 & PnP & \colorbox{red!10}{$16.77\,{\scriptstyle\pm\,0.06}$} & \colorbox{red!10}{$0.7815\,{\scriptstyle\pm\,0.0022}$} & $0.1113\,{\scriptstyle\pm\,0.0042}$ & $20.91\,{\scriptstyle\pm\,0.09}$ & \colorbox{red!10}{$0.8842\,{\scriptstyle\pm\,0.0017}$} & \colorbox{red!10}{$0.0687\,{\scriptstyle\pm\,0.0030}$} \\
\midrule
\multirow{3}{*}{{\shortstack{Diffusion-based\\Gaussian-approximated}}} & DAPS & $15.63\,{\scriptstyle\pm\,0.07}$ & $0.7333\,{\scriptstyle\pm\,0.0036}$ & $0.1821\,{\scriptstyle\pm\,0.0073}$ & $21.03\,{\scriptstyle\pm\,0.09}$ & $0.8753\,{\scriptstyle\pm\,0.0020}$ & \colorbox{red!25}{$0.0606\,{\scriptstyle\pm\,0.0039}$} \\
 & DPS & $15.99\,{\scriptstyle\pm\,0.12}$ & $0.7516\,{\scriptstyle\pm\,0.0046}$ & \colorbox{red!10}{$0.1006\,{\scriptstyle\pm\,0.0063}$} & $17.35\,{\scriptstyle\pm\,0.11}$ & $0.7959\,{\scriptstyle\pm\,0.0031}$ & $0.0859\,{\scriptstyle\pm\,0.0050}$ \\
 & PnP & $15.10\,{\scriptstyle\pm\,0.01}$ & $0.7262\,{\scriptstyle\pm\,0.0007}$ & $0.3029\,{\scriptstyle\pm\,0.0038}$ & $20.36\,{\scriptstyle\pm\,0.07}$ & $0.8750\,{\scriptstyle\pm\,0.0012}$ & $0.0937\,{\scriptstyle\pm\,0.0033}$ \\
\bottomrule
\end{tabular}
\caption{Poisson denoising problem ($\mca = a\Idd$) with the BrainWeb dataset, in the low-count ($a=1$) and high-count ($a=5$) regime. ``Poisson" and ``Gaussian-approximated" refer to which likelihood is used, similar to the suffix -p, -g of the smaller tables: either the exact Poisson likelihood or $g(y\vert x) \approx \gauss(\mca(x),\mathrm{diag}(\mca(x))$.}
\label{tab:brain_poisson}
\end{table}

\begin{table}[ht]
\centering
\scriptsize
\setlength{\tabcolsep}{2pt}
\begin{tabular}{c l c c c c c c}
\toprule
\multicolumn{2}{c}{\multirow{2}{*}{\textbf{Methods}}} & \multicolumn{3}{c}{\textbf{BrainWeb PET (LC)}} & \multicolumn{3}{c}{\textbf{BrainWeb PET (HC)}} \\
\cmidrule(lr){3-5}
\cmidrule(lr){6-8}
 & & PSNR {\scriptsize$\uparrow$} & SSIM {\scriptsize$\uparrow$} & LPIPS {\scriptsize$\downarrow$} & PSNR {\scriptsize$\uparrow$} & SSIM {\scriptsize$\uparrow$} & LPIPS {\scriptsize$\downarrow$} \\
\midrule
\midrule
\multirow{3}{*}{\shortstack{\textbf{MPDM-based}\\Poisson}} & DAPS & $15.82\,{\scriptstyle\pm\,0.07}$ & $0.7296\,{\scriptstyle\pm\,0.0028}$ & $0.1619\,{\scriptstyle\pm\,0.0065}$ & $17.25\,{\scriptstyle\pm\,0.06}$ & $0.7783\,{\scriptstyle\pm\,0.0024}$ & $0.1541\,{\scriptstyle\pm\,0.0060}$ \\
 & DPS & $15.56\,{\scriptstyle\pm\,0.13}$ & $0.7084\,{\scriptstyle\pm\,0.0085}$ & $0.1395\,{\scriptstyle\pm\,0.0096}$ & $16.84\,{\scriptstyle\pm\,0.16}$ & $0.7503\,{\scriptstyle\pm\,0.0100}$ & $0.1318\,{\scriptstyle\pm\,0.0092}$ \\
 & PnP & \colorbox{red!25}{$16.41\,{\scriptstyle\pm\,0.05}$} & $0.7329\,{\scriptstyle\pm\,0.0019}$ & \colorbox{red!25}{$0.1051\,{\scriptstyle\pm\,0.0040}$} & \colorbox{red!25}{$18.11\,{\scriptstyle\pm\,0.05}$} & \colorbox{red!25}{$0.7940\,{\scriptstyle\pm\,0.0018}$} & \colorbox{red!25}{$0.0828\,{\scriptstyle\pm\,0.0031}$} \\
\midrule
\multirow{3}{*}{\shortstack{Diffusion-based\\ Poisson}} & DAPS & \colorbox{red!10}{$16.07\,{\scriptstyle\pm\,0.06}$} & \colorbox{red!10}{$0.7327\,{\scriptstyle\pm\,0.0032}$} & $0.1330\,{\scriptstyle\pm\,0.0066}$ & \colorbox{red!10}{$17.63\,{\scriptstyle\pm\,0.06}$} & \colorbox{red!25}{$0.7923\,{\scriptstyle\pm\,0.0022}$} & \colorbox{red!10}{$0.1027\,{\scriptstyle\pm\,0.0045}$} \\
 & DPS & $15.82\,{\scriptstyle\pm\,0.13}$ & \colorbox{red!10}{$0.7401\,{\scriptstyle\pm\,0.0042}$} & \colorbox{red!25}{$0.0968\,{\scriptstyle\pm\,0.0051}$} & $16.66\,{\scriptstyle\pm\,0.20}$ & $0.7758\,{\scriptstyle\pm\,0.0063}$ & \colorbox{red!10}{$0.0947\,{\scriptstyle\pm\,0.0064}$} \\
 & PnP & $15.74\,{\scriptstyle\pm\,0.05}$ & $0.7194\,{\scriptstyle\pm\,0.0022}$ & $0.1755\,{\scriptstyle\pm\,0.0056}$ & $17.33\,{\scriptstyle\pm\,0.06}$ & \colorbox{red!10}{$0.7819\,{\scriptstyle\pm\,0.0023}$} & $0.1365\,{\scriptstyle\pm\,0.0049}$ \\
\midrule
\multirow{3}{*}{\shortstack{Diffusion-base \\Gaussian-approximated}} & DAPS & \colorbox{red!10}{$16.13\,{\scriptstyle\pm\,0.06}$} & \colorbox{red!25}{$0.7478\,{\scriptstyle\pm\,0.0025}$} & \colorbox{red!10}{$0.1111\,{\scriptstyle\pm\,0.0058}$} & \colorbox{red!10}{$17.60\,{\scriptstyle\pm\,0.08}$} & \colorbox{red!25}{$0.7930\,{\scriptstyle\pm\,0.0046}$} & \colorbox{red!10}{$0.1115\,{\scriptstyle\pm\,0.0123}$} \\
 & DPS & $14.75\,{\scriptstyle\pm\,0.15}$ & $0.6891\,{\scriptstyle\pm\,0.0057}$ & $0.1678\,{\scriptstyle\pm\,0.0106}$ & $15.73\,{\scriptstyle\pm\,0.13}$ & $0.7372\,{\scriptstyle\pm\,0.0047}$ & \colorbox{red!10}{$0.0962\,{\scriptstyle\pm\,0.0059}$} \\
 & PnP & $15.86\,{\scriptstyle\pm\,0.06}$ & \colorbox{red!10}{$0.7392\,{\scriptstyle\pm\,0.0028}$} & $0.1388\,{\scriptstyle\pm\,0.0090}$ & $17.29\,{\scriptstyle\pm\,0.07}$ & \colorbox{red!10}{$0.7869\,{\scriptstyle\pm\,0.0039}$} & $0.1390\,{\scriptstyle\pm\,0.0106}$ \\
\bottomrule
\end{tabular}
\caption{PET  (with the BrainWeb dataset), in the low-count ($a=20$) and high-count ($a=100$) regimes.}
\label{tab:pet}
\end{table}

\begin{table}[ht]
\centering
\scriptsize
\setlength{\tabcolsep}{2pt}
\begin{tabular}{c l c c c c c c}
\toprule
\multicolumn{2}{c}{\multirow{2}{*}{\textbf{Methods}}} & \multicolumn{3}{c}{\textbf{EAGLE denoising (LC)}} & \multicolumn{3}{c}{\textbf{EAGLE denoising (HC)}} \\
\cmidrule(lr){3-5}
\cmidrule(lr){6-8}
 & & PSNR {\scriptsize$\uparrow$} & SSIM {\scriptsize$\uparrow$} & LPIPS {\scriptsize$\downarrow$} & PSNR {\scriptsize$\uparrow$} & SSIM {\scriptsize$\uparrow$} & LPIPS {\scriptsize$\downarrow$} \\
\midrule
\midrule
\multirow{4}{*}{\shortstack{\textbf{MPDM-based}\\Poisson}} & DAPS & $29.97\,{\scriptstyle\pm\,0.09}$ & $0.5871\,{\scriptstyle\pm\,0.0036}$ & $0.3065\,{\scriptstyle\pm\,0.0105}$ & $32.94\,{\scriptstyle\pm\,0.06}$ & $0.6964\,{\scriptstyle\pm\,0.0023}$ & $0.1283\,{\scriptstyle\pm\,0.0066}$ \\
 & DPS & \colorbox{red!10}{$31.55\,{\scriptstyle\pm\,0.19}$} & $0.7472\,{\scriptstyle\pm\,0.0082}$ & \colorbox{red!10}{$0.1783\,{\scriptstyle\pm\,0.0152}$} & $33.16\,{\scriptstyle\pm\,0.11}$ & $0.7748\,{\scriptstyle\pm\,0.0044}$ & $0.1369\,{\scriptstyle\pm\,0.0128}$ \\
 & PnP & $26.25\,{\scriptstyle\pm\,0.14}$ & $0.5933\,{\scriptstyle\pm\,0.0040}$ & $0.4149\,{\scriptstyle\pm\,0.0106}$ & \colorbox{red!10}{$34.80\,{\scriptstyle\pm\,0.06}$} & \colorbox{red!10}{$0.8401\,{\scriptstyle\pm\,0.0012}$} & \colorbox{red!10}{$0.0991\,{\scriptstyle\pm\,0.0040}$} \\
 & PDPS & \colorbox{red!25}{$32.34\,{\scriptstyle\pm\,0.15}$} & \colorbox{red!10}{$0.7916\,{\scriptstyle\pm\,0.0034}$} & \colorbox{red!25}{$0.1543\,{\scriptstyle\pm\,0.0107}$} & \colorbox{red!25}{$34.94\,{\scriptstyle\pm\,0.08}$} & \colorbox{red!25}{$0.8435\,{\scriptstyle\pm\,0.0013}$} & \colorbox{red!25}{$0.0863\,{\scriptstyle\pm\,0.0044}$} \\
\midrule
\multirow{3}{*}{\shortstack{Diffusion-based\\Poisson}} & DAPS & $29.92\,{\scriptstyle\pm\,0.09}$ & $0.7268\,{\scriptstyle\pm\,0.0028}$ & $0.2547\,{\scriptstyle\pm\,0.0108}$ & $34.60\,{\scriptstyle\pm\,0.09}$ & \colorbox{red!10}{$0.8388\,{\scriptstyle\pm\,0.0016}$} & $0.1546\,{\scriptstyle\pm\,0.0053}$ \\
 & DPS & \colorbox{red!25}{$32.46\,{\scriptstyle\pm\,0.14}$} & \colorbox{red!25}{$0.8060\,{\scriptstyle\pm\,0.0026}$} & $0.2213\,{\scriptstyle\pm\,0.0127}$ & \colorbox{red!10}{$34.69\,{\scriptstyle\pm\,0.09}$} & \colorbox{red!25}{$0.8409\,{\scriptstyle\pm\,0.0014}$} & $0.1612\,{\scriptstyle\pm\,0.0065}$ \\
 & PnP & $26.95\,{\scriptstyle\pm\,0.03}$ & $0.6884\,{\scriptstyle\pm\,0.0014}$ & $0.4085\,{\scriptstyle\pm\,0.0053}$ & $32.99\,{\scriptstyle\pm\,0.04}$ & $0.8203\,{\scriptstyle\pm\,0.0009}$ & $0.1902\,{\scriptstyle\pm\,0.0033}$ \\
\midrule
\multirow{3}{*}{\shortstack{Diffusion-based \\Gaussian-approximated}} & DAPS & $26.48\,{\scriptstyle\pm\,0.06}$ & $0.6530\,{\scriptstyle\pm\,0.0055}$ & $0.2322\,{\scriptstyle\pm\,0.0086}$ & $34.66\,{\scriptstyle\pm\,0.08}$ & $0.8372\,{\scriptstyle\pm\,0.0016}$ & $0.1630\,{\scriptstyle\pm\,0.0072}$ \\
 & DPS & \colorbox{red!10}{$31.23\,{\scriptstyle\pm\,0.19}$} & $0.7503\,{\scriptstyle\pm\,0.0044}$ & $0.3001\,{\scriptstyle\pm\,0.0160}$ & $32.57\,{\scriptstyle\pm\,0.15}$ & $0.8075\,{\scriptstyle\pm\,0.0027}$ & $0.2122\,{\scriptstyle\pm\,0.0120}$ \\
 & PnP & $24.07\,{\scriptstyle\pm\,0.02}$ & $0.6610\,{\scriptstyle\pm\,0.0010}$ & $0.3753\,{\scriptstyle\pm\,0.0047}$ & $33.42\,{\scriptstyle\pm\,0.07}$ & $0.8286\,{\scriptstyle\pm\,0.0013}$ & $0.1680\,{\scriptstyle\pm\,0.0043}$ \\
\bottomrule
\end{tabular}
\caption{Poisson denoising problem ($\mca=a\Idd$) for the EAGLE dataset, in the low-count ($a=4$) and high-count ($a=20$) regime.}
\label{tab:eagle_poisson}
\end{table}

\begin{table}[ht]
\centering
\scriptsize
\setlength{\tabcolsep}{2pt}
\begin{tabular}{c l c c c c c c}
\toprule
\multicolumn{2}{c}{\multirow{2}{*}{\textbf{Methods}}} & \multicolumn{3}{c}{\textbf{EAGLE blur (LC)}} & \multicolumn{3}{c}{\textbf{EAGLE blur (HC)}} \\
\cmidrule(lr){3-5}
\cmidrule(lr){6-8}
 & & PSNR {\scriptsize$\uparrow$} & SSIM {\scriptsize$\uparrow$} & LPIPS {\scriptsize$\downarrow$} & PSNR {\scriptsize$\uparrow$} & SSIM {\scriptsize$\uparrow$} & LPIPS {\scriptsize$\downarrow$} \\
\midrule
\midrule
\multirow{3}{*}{\shortstack{\textbf{MPDM-based}\\Poisson}} & DAPS & $27.34\,{\scriptstyle\pm\,0.22}$ & $0.6174\,{\scriptstyle\pm\,0.0044}$ & $0.4528\,{\scriptstyle\pm\,0.0100}$ & $30.63\,{\scriptstyle\pm\,0.12}$ & $0.6832\,{\scriptstyle\pm\,0.0029}$ & $0.2970\,{\scriptstyle\pm\,0.0105}$ \\
 & DPS & \colorbox{red!10}{$31.13\,{\scriptstyle\pm\,0.22}$} & $0.7638\,{\scriptstyle\pm\,0.0044}$ & \colorbox{red!25}{$0.2049\,{\scriptstyle\pm\,0.0171}$} & $31.64\,{\scriptstyle\pm\,0.19}$ & $0.7764\,{\scriptstyle\pm\,0.0033}$ & \colorbox{red!25}{$0.1858\,{\scriptstyle\pm\,0.0138}$} \\
 & PnP & $29.83\,{\scriptstyle\pm\,0.10}$ & $0.7341\,{\scriptstyle\pm\,0.0023}$ & $0.3039\,{\scriptstyle\pm\,0.0101}$ & \colorbox{red!25}{$33.22\,{\scriptstyle\pm\,0.05}$} & \colorbox{red!10}{$0.8000\,{\scriptstyle\pm\,0.0009}$} & $0.2511\,{\scriptstyle\pm\,0.0084}$ \\
\midrule
\multirow{3}{*}{\shortstack{Diffusion-based\\Poisson}} & DAPS &  $30.68\,{\scriptstyle\pm\,0.11}$ & $0.7419\,{\scriptstyle\pm\,0.0024}$ & $0.2724\,{\scriptstyle\pm\,0.0098}$ & \colorbox{red!10}{$33.01\,{\scriptstyle\pm\,0.09}$} & \colorbox{red!10}{$0.8000\,{\scriptstyle\pm\,0.0018}$} & $0.2370\,{\scriptstyle\pm\,0.0087}$ \\
 & DPS & \colorbox{red!25}{$31.72\,{\scriptstyle\pm\,0.20}$} & \colorbox{red!25}{$0.7935\,{\scriptstyle\pm\,0.0030}$} & \colorbox{red!10}{$0.2432\,{\scriptstyle\pm\,0.0155}$} & $32.56\,{\scriptstyle\pm\,0.18}$ & \colorbox{red!25}{$0.8017\,{\scriptstyle\pm\,0.0026}$} & \colorbox{red!10}{$0.2129\,{\scriptstyle\pm\,0.0104}$} \\
 & PnP & $27.90\,{\scriptstyle\pm\,0.09}$ & $0.6530\,{\scriptstyle\pm\,0.0024}$ & $0.3573\,{\scriptstyle\pm\,0.0084}$ & $29.58\,{\scriptstyle\pm\,0.08}$ & $0.7041\,{\scriptstyle\pm\,0.0023}$ & $0.3103\,{\scriptstyle\pm\,0.0085}$ \\
\midrule
\multirow{3}{*}{\shortstack{Diffusion-based\\Gaussian-approximated}} & DAPS & \colorbox{red!25}{$31.99\,{\scriptstyle\pm\,0.10}$} & \colorbox{red!10}{$0.7881\,{\scriptstyle\pm\,0.0019}$} & \colorbox{red!10}{$0.2652\,{\scriptstyle\pm\,0.0088}$}  & \colorbox{red!25}{$33.14\,{\scriptstyle\pm\,0.09}$} & \colorbox{red!25}{$0.8041\,{\scriptstyle\pm\,0.0018}$} & $0.2365\,{\scriptstyle\pm\,0.0087}$ \\
 & DPS & \colorbox{red!10}{$30.90\,{\scriptstyle\pm\,0.22}$} & $0.7691\,{\scriptstyle\pm\,0.0041}$ & $0.3263\,{\scriptstyle\pm\,0.0145}$ & $32.52\,{\scriptstyle\pm\,0.21}$ & \colorbox{red!10}{$0.7985\,{\scriptstyle\pm\,0.0027}$} & $0.2352\,{\scriptstyle\pm\,0.0116}$ \\
 & PnP & $28.64\,{\scriptstyle\pm\,0.10}$ & $0.7136\,{\scriptstyle\pm\,0.0023}$ & $0.3632\,{\scriptstyle\pm\,0.0088}$ & $29.74\,{\scriptstyle\pm\,0.08}$ & $0.7153\,{\scriptstyle\pm\,0.0024}$ & $0.3072\,{\scriptstyle\pm\,0.0081}$ \\
\bottomrule
\end{tabular}
\caption{Deblurring problem (with the EAGLE dataset), in the low-count ($a=4$) and high-count ($a=20$) regimes.}
\label{tab:blur}
\end{table}

\paragraph{PET.} Our PET setting was constructed with  \texttt{parallelproj}\footnote{\url{https://parallelproj.readthedocs.io/en/stable/}} \citep{schramm_parallelproj_2024}. The forward operator comprised $180$ equally spaced projection angles over $0\leqslant \theta\leqslant \pi$. With $256\times 256$ grayscale images, the discretized Radon operator satisfies $\mca:\rset_+^d\rightarrow \rset_+^{d_y}$, with $d=256\times 256$ and $d_y = 180\times 180$.
For each angle, projection bins were uniformly distributed radially, spanning 3.5 times the image width. The projection geometry used a source-to-detector radius of 400 pixels. The image origin was centered such that the reconstructed field of view was symmetric about the origin. 
See \Cref{fig:tep_fbp} for an illustration. 
For both the PET and Poisson denoising problems, each posterior sampling algorithm relies on the same log-uniform time-grid $t_0<\dots<t_N$, with $t_N = 100$ and $N=128$,  
\begin{figure}
    \centering
    \includegraphics[width=\linewidth]{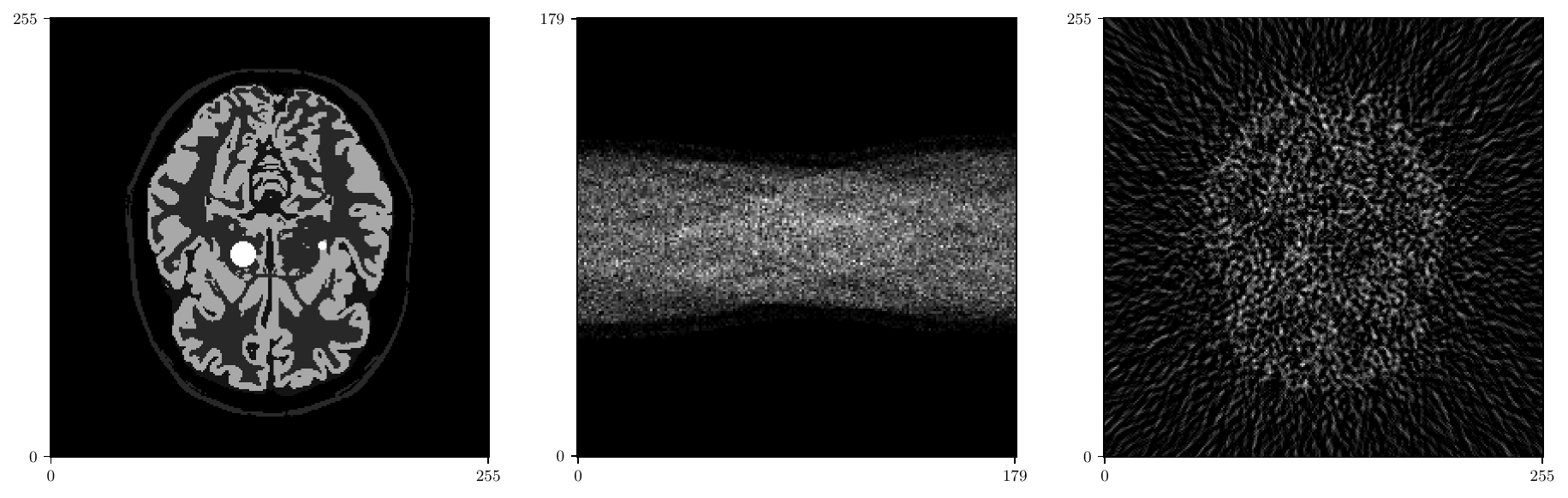}
    \caption{PET illustration. \textbf{(Left)} Ground-truth image $x$ from the BrainWeb dataset, with lesions (bright dots). \textbf{(Center)} the \emph{sinogram}, \ie, observation $y\sim \poisson(\mca(x))$. Each pixel represents the number of detections for a given pair of detectors.
    \textbf{(Right)} Filtered back-projection of $y$.}
    \label{fig:tep_fbp}
\end{figure}

\paragraph{Astrophysics.} We considered a two-dimensional Gaussian blurring forward operator. The input images were convolved with a normalized $65\times65$ Gaussian kernel with standard deviation $\sigma=3$. pixels. Zero-padding was used at the image boundaries to preserve the image dimensions. Here, $\mca : \rset_+^d \rightarrow \rset_+^d$ with $d=d_y = 3\times 256\times 256$.
See \Cref{fig:blur_fbp} for a visual example. For both the blur and Poisson denoising problems, each posterior sampling algorithm relies on the same log-uniform time-grid $t_0<\dots<t_N$, with $t_N = 100$ and $N=256$. 
\begin{figure}
    \centering
    \includegraphics[width=0.8\linewidth]{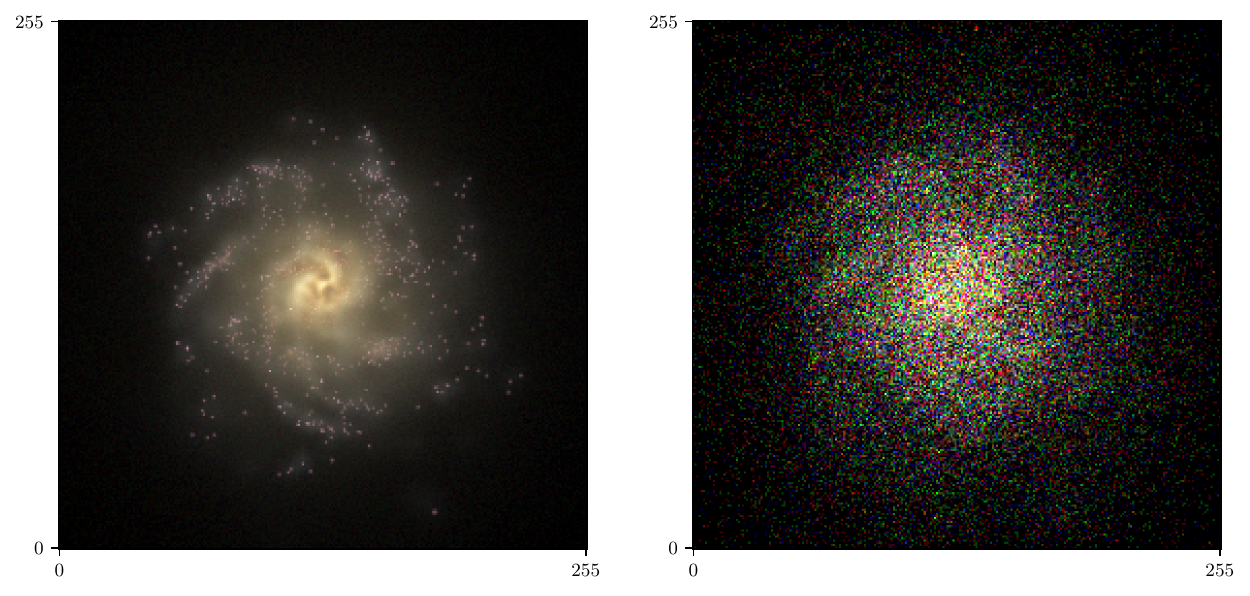}
    \caption{Gaussian blurring illustration. \textbf{(Left)} Ground-truth image $x$ from the EAGLE dataset. \textbf{(Right)} Observation $y\sim\poisson(\mca(x))$.}
    \label{fig:blur_fbp}
\end{figure}

\begin{table}[ht]
\scriptsize
\centering
\newcolumntype{N}{>{\centering\arraybackslash}m{1cm}}
\setlength{\tabcolsep}{2pt}
\begin{tabular}{c*{10}{N}}
\toprule
 & \multicolumn{2}{c}{\textbf{Toy}} &
\multicolumn{2}{c}{\textbf{BrainWeb denoising}} &
\multicolumn{2}{c}{\textbf{BrainWeb PET}} &
\multicolumn{2}{c}{\textbf{Eagle denoising}} &
\multicolumn{2}{c}{\textbf{Eagle blur}} \\
\cmidrule(lr){2-3}
\cmidrule(lr){4-5}
\cmidrule(lr){6-7}
\cmidrule(lr){8-9}
\cmidrule(lr){10-11}
  & \textbf{LC} & \textbf{HC} &
    \textbf{LC} & \textbf{HC} &
    \textbf{LC} & \textbf{HC} &
    \textbf{LC} & \textbf{HC} &
    \textbf{LC} & \textbf{HC} \\
\midrule
\midrule
PnP-DM (MPDM) $\eta$ & 1 & 1 & 8 & 5 & 14 & 8 & 5 & 4 & 7 & 5 \\
DAPS (MPDM) $\eta$ & 1 & 1 & 5 & 4 & 8 & 5 & 2 & 3 & 6 & 3 \\
PnP-DM-p (diffusion) $\eta$ & 1 & 1 & 4 & 2 & 8 & 4 & 10 &  10 & 7 & 3 \\
DAPS-p (diffusion) $\eta$ & 1 & 1 & 6 & 5 & 12 & 8 & 5 & 7 & 8 & 6 \\
PnP-DM-g (diffusion) $\eta$ & 1 & 1 & 9 & 2 & 6 & 4 & 9 & 2 & 4 & 2 \\
DAPS-g (diffusion) $\eta$ & 1 & 1 & 2 & 2 & 10 & 8 & 2 & 2 & 4 & 5 \\
\bottomrule
\end{tabular}
\caption{We specify the parameter $\eta$ chosen for the algorithms relying on a Langevin MCMC scheme: our MPDM-based PnP-DM (\Cref{alg:pnp_dm}) and DAPS (\Cref{alg:poisson_daps}), as well as their diffusion counterparts (either with true Poisson likelihood, indicated by -p, or its Gaussian approximation, indicated by -g). Values of $\eta$ ranging from $1$ to $20$ were tested in each setting, and we chose the ones leading to the overall best reconstruction metrics.}
\label{tab:eta}
\end{table}

\subsection{PET: additional measurements}
\label{app:more_pet}
The classical metrics we use to evaluate posterior sampling algorithms measure the quality of a given reconstructed image. They therefore provide only a partial assessment of sampling quality, as they do not directly account for the dispersion or diversity of the generated samples. Typically, in low-count regimes, posterior sampling methods should capture the dispersion of the true posterior. Moreover, the relevant notion of reconstruction quality may depend on the downstream task and the specific inverse problem at hand. In PET, for example, detecting lesions and accurately quantifying the associated uncertainty can be more relevant than the image quality measured by the aforementioned metrics. \Cref{fig:pet_lesion_poisson,fig:pet_lesion} illustrate the detection of lesions with the BrainWeb dataset, in low-count regimes.  

To quantify this,
we also evaluate reconstruction errors specifically within the lesions; see \Cref{tab:brain_lesion,tab:pet_lesion}. To do so, we compute the mean absolute error between the ground truth images and the approximate posterior samples, restricted to lesions (pixels with value  $>0.9$ on either image). If $x$ denotes the ground truth and $\hat{x}$ the reconstruction, we define it as 
\begin{equation}
l(x,\hat{x})=\frac{1}{\#\mathcal{L}(x)\cup\mathcal{L}(\hat{x})}\sum_{i=1}^d \vert \1_{\mathcal{L}(x)} -\1_{\mathcal{L}(\hat{x})}\vert \eqsp,
\end{equation}
with $\mathcal{L}(x) = \{1\leqslant i\leqslant d: x_i >0.9\}, \mathcal{L}(\hat{x}) = \{1\leqslant i\leqslant d: \hat{x}_i >0.9\}$. 
In particular, we show that, while diffusion-based DPS is on par with, or even outperforms, MPDM-based methods on several metrics reported in \Cref{tab:brain_poisson,tab:pet}, it nevertheless exhibits larger reconstruction errors in lesion regions. This is still an image metric, however, and not a proper posterior calibration evaluation.
\begin{figure}
    \centering
    \includegraphics[width=0.7\linewidth]{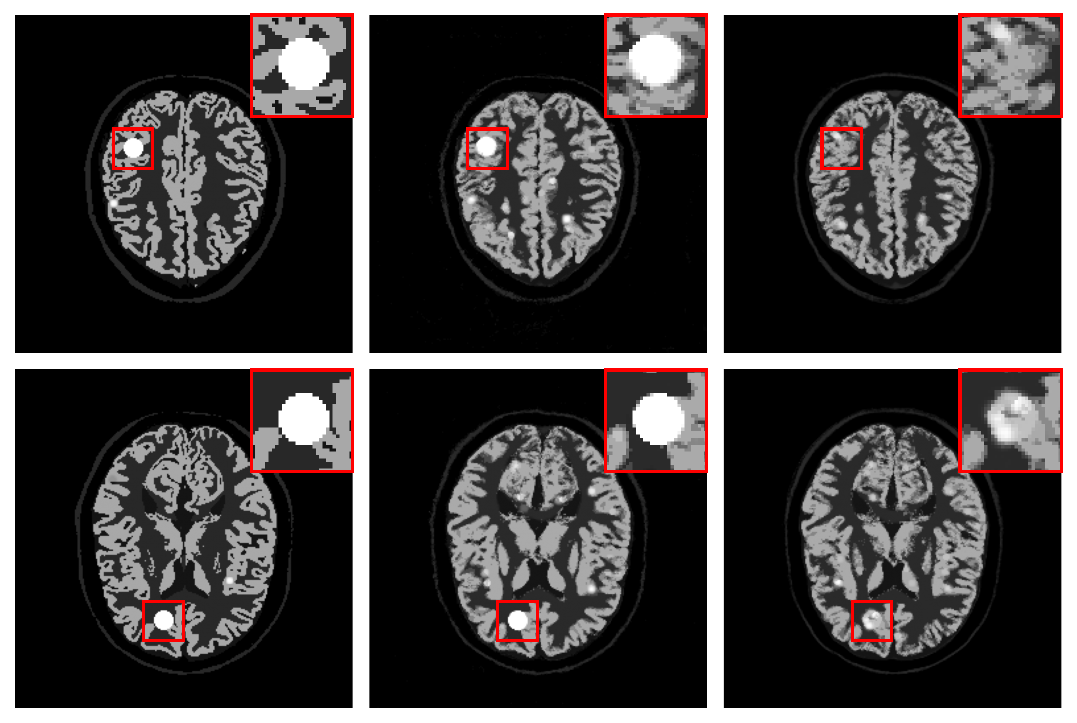}
    \caption{Poisson denoising problem with the BrainWeb dataset, in the low-count regime. \textbf{(Left)} Ground truth image from the BrainWeb validation dataset. \textbf{(Center)} Average of 4 samples obtained with PDPS. \textbf{(Right)} Average of 4 samples obtained with diffusion-based DPS. }
    \label{fig:pet_lesion_poisson}
\end{figure}

\begin{figure}
    \centering
\includegraphics[width=0.7\linewidth]{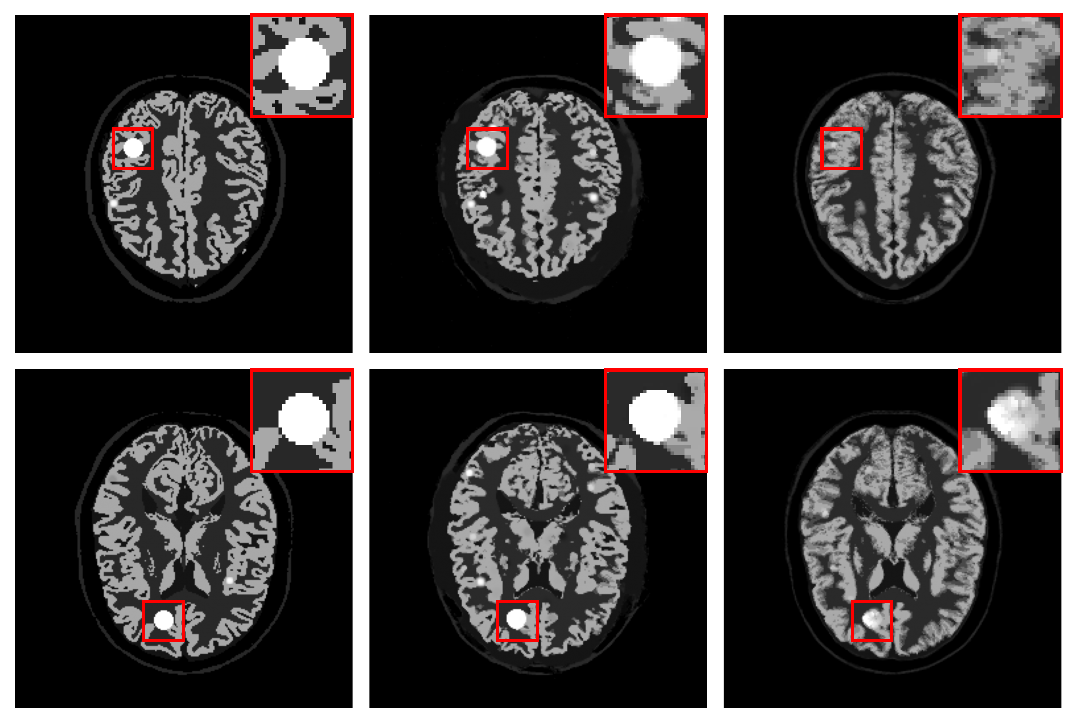}
    \caption{PET problem with the BrainWeb dataset, in the low-count regime. \textbf{(Left)} Ground truth image from the BrainWeb validation dataset. \textbf{(Center)} Average of 4 samples obtained with MPDM-based PnP-DM. \textbf{(Right)} Average of 4 samples obtained with diffusion-based DPS.}
    \label{fig:pet_lesion}
\end{figure}

\begin{table}[ht]
\centering

\begin{minipage}[t]{0.48\textwidth}
\centering
\scriptsize
\setlength{\tabcolsep}{5pt}
\begin{tabular}{c | c}
\toprule
\textbf{Methods} & \textbf{BrainWeb denoising (LC)} \\
\midrule
\midrule
\textbf{PDPS (MPDM)} & $0.7264\,{\scriptstyle\pm\,0.0051}$ \\
DPS-p (diffusion) & $0.7832\,{\scriptstyle\pm\,0.0094}$ \\
\bottomrule
\end{tabular}
\caption{Average $l(x,\hat{x})$, where $x$ is a validation sample of the BrainWeb dataset, and $\hat{x}$ is an approximate posterior sample obtained with either diffusion-based DPS or PDPS, for the Poisson denoising problem ($\mca=a\Idd$) in the low-count regime.}
\label{tab:brain_lesion}
\end{minipage}
\hfill
\begin{minipage}[t]{0.48\textwidth}
\centering
\scriptsize
\setlength{\tabcolsep}{5pt}
\begin{tabular}{c | c}
\toprule
\textbf{Methods} & \textbf{BrainWeb PET (LC)} \\
\midrule
\midrule
\textbf{PnP-DM (MPDM)} &  $0.7573\,{\scriptstyle\pm\,0.0080}$\\
DPS-p (diffusion) &  $0.8458\,{\scriptstyle\pm\,0.0228}$\\
\bottomrule
\end{tabular}
\caption{Average $l(x,\hat{x})$, where $x$ is a validation sample of the BrainWeb dataset, and $\hat{x}$ is an approximate posterior sample obtained with either diffusion-based DPS or PDPS, for the PET problem in the low-count regime.}
\label{tab:pet_lesion}
\end{minipage}

\end{table}

\subsection{Generative modeling: gallery}
\label{app:gallery}
\begin{figure}
    \centering
    \includegraphics[width=0.9\linewidth]{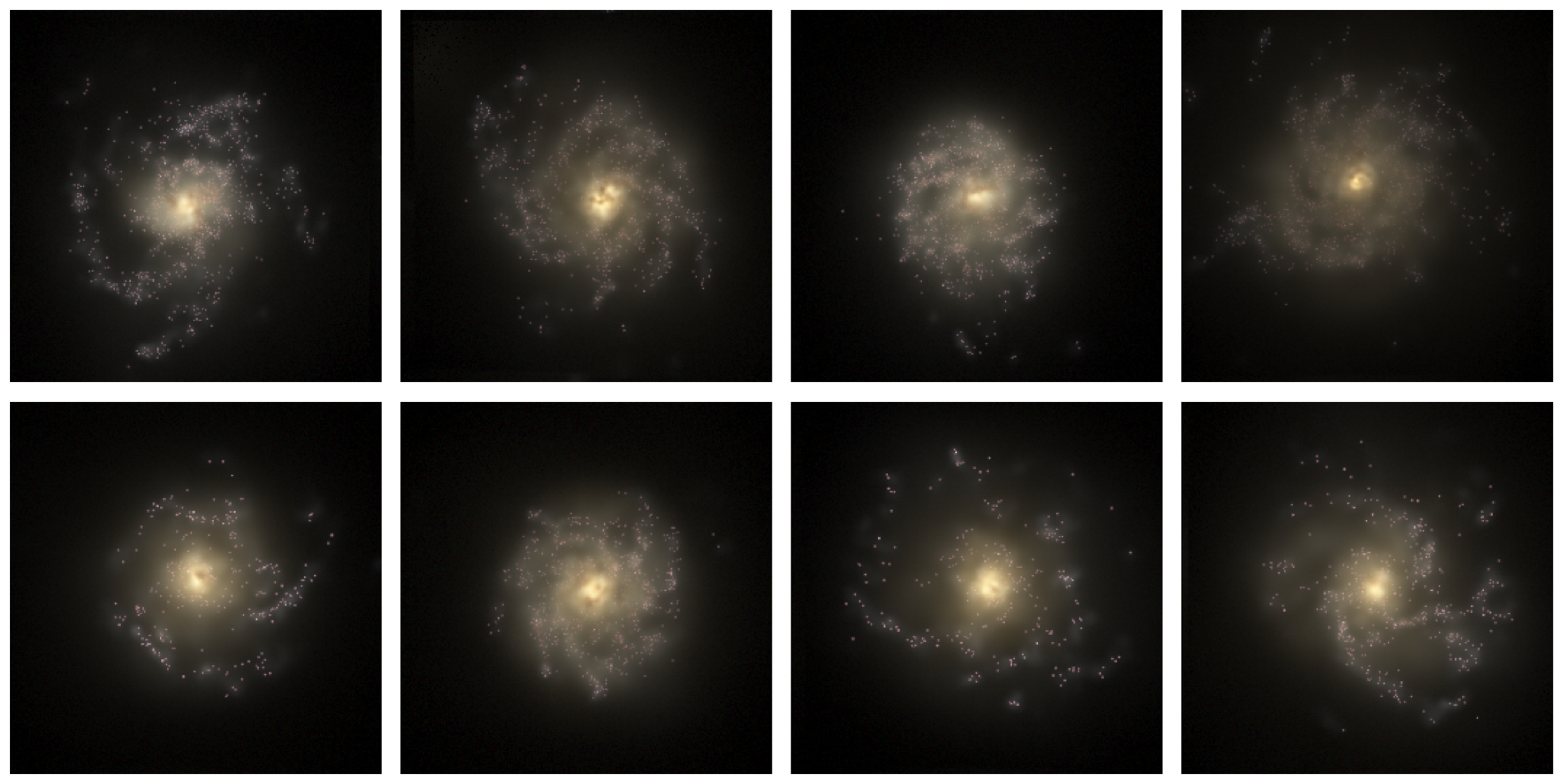}
    \caption{Images generated by our MPDM trained on the EAGLE dataset.}
    \label{fig:eagle_gen}
\end{figure}
\begin{figure}
    \centering
    \includegraphics[width=0.9\linewidth]{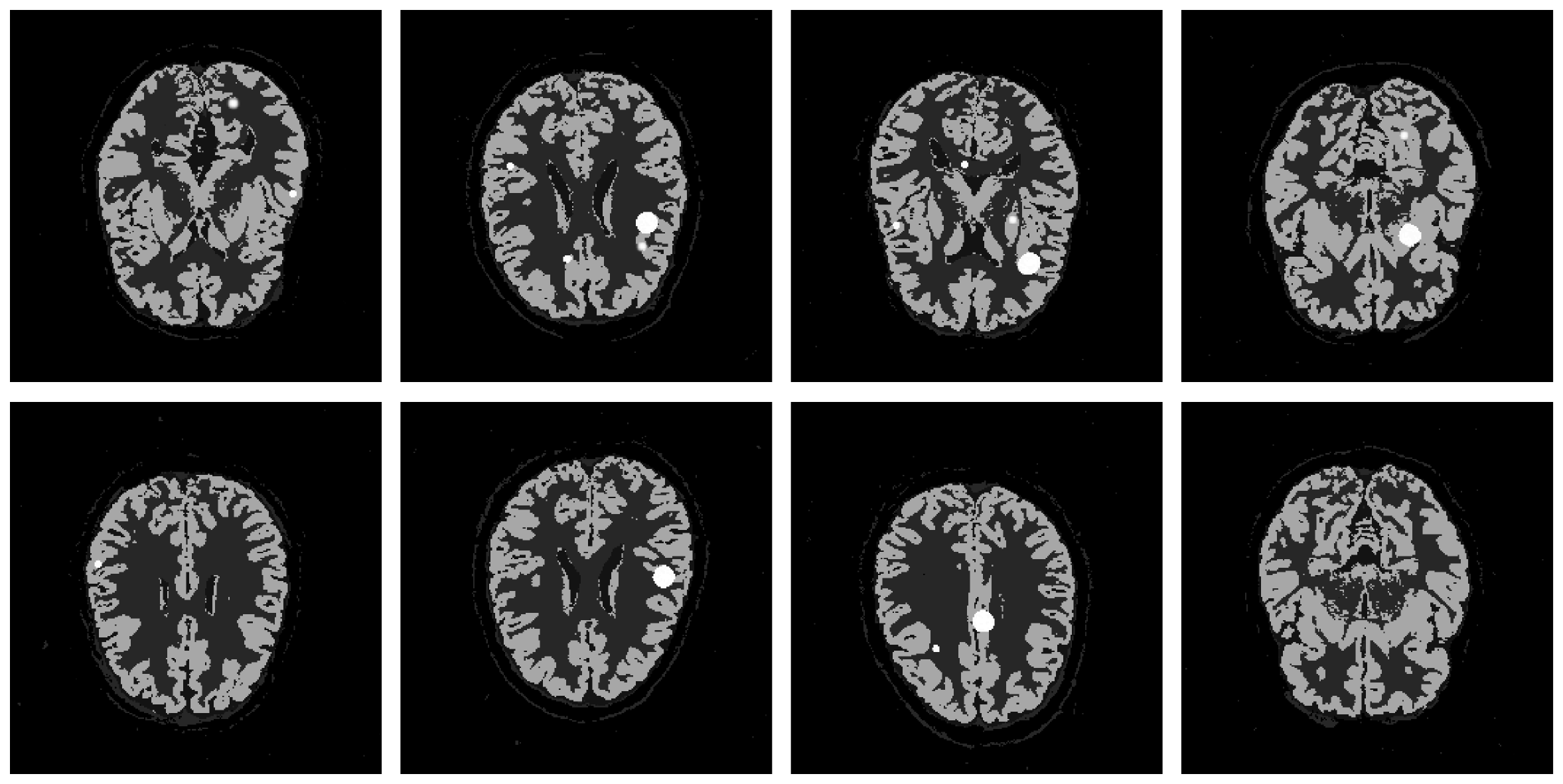}
    \caption{Images generated by our MPDM trained on the BrainWeb dataset.}
    \label{fig:brain_gen}
\end{figure}

\end{document}